\documentclass[twoside]{article}

\usepackage[accepted]{aistats2026}
\usepackage{amsmath,amsfonts,bm}

\def\eqref#1{equation~\ref{#1}}

\def\1{\bm{1}}

\DeclareMathAlphabet{\mathsfit}{\encodingdefault}{\sfdefault}{m}{sl}
\SetMathAlphabet{\mathsfit}{bold}{\encodingdefault}{\sfdefault}{bx}{n}

\newcommand{\R}{\mathbb{R}}

\usepackage{hyperref}
\usepackage{url}
\usepackage{tikz}
\usepackage{amsmath}
\usepackage{amssymb}
\usepackage{mathtools}
\usepackage{amsthm}
\usepackage{dsfont}
\usepackage{wrapfig}
\usepackage{microtype}
\usepackage{multirow}
\usepackage{graphicx}
\usepackage[small]{caption}
\usepackage{subcaption}
\usepackage[nice]{nicefrac}
\usepackage[capitalize,noabbrev]{cleveref}
\usepackage{algorithm}
\usepackage{algorithmic}
\usepackage{enumerate}
\usepackage{xcolor,colortbl}
\usepackage{nicefrac}
\usepackage[most]{tcolorbox}

\definecolor{mydarkblue}{rgb}{0,0.08,0.45}
\hypersetup{
  colorlinks=true,
  linkcolor=mydarkblue,
  citecolor=mydarkblue,
  urlcolor=mydarkblue
}

\theoremstyle{plain}
\newtheorem{theorem}{Theorem}
\newtheorem{proposition}{Proposition}
\newtheorem{lemma}{Lemma}
\newtheorem{corollary}{Corollary}
\theoremstyle{definition}
\newtheorem{definition}{Definition}
\newtheorem{assumption}{Assumption}
\theoremstyle{remark}
\newtheorem{remark}{Remark}

\newcommand{\bw}{\mathbf{w}}
\newcommand{\bP}{\mathbf{P}}
\newcommand{\bx}{\mathbf{x}}
\newcommand{\bX}{\mathbf{X}}
\newcommand{\bZ}{\mathbf{Z}}

\newcommand{\ba}{\mathbf{a}}
\newcommand{\bv}{\mathbf{v}}
\newcommand{\by}{\mathbf{y}}
\newcommand{\bW}{\mathbf{W}}
\newcommand{\bJ}{\mathbf{J}}
\newcommand{\bH}{\mathbf{H}}
\newcommand{\bI}{\mathbf{I}}

\newcommand{\bb}{\mathbf{b}}
\newcommand{\bR}{\mathbf{R}}

\DeclareMathOperator{\img}{img}

\definecolor{snsblue}{RGB}{52, 152, 219}
\definecolor{snsorange}{RGB}{230, 126, 34}

\usepackage{enumitem}

\newcommand{\ourtitle}{On the Interplay of Priors and Overparametrization\\ in Bayesian Neural Network Posteriors}
\runningtitle{On the Interplay of Priors and Overparametrization in Bayesian Neural Network Posteriors}

\usepackage[round]{natbib}
\renewcommand{\bibname}{References}
\renewcommand{\bibsection}{\subsubsection*{\bibname}}

\begin{document}

\runningauthor{Julius Kobialka, Emanuel Sommer, Chris Kolb, Juntae Kwon, Daniel Dold, David R\"ugamer}

\twocolumn[

\aistatstitle{\ourtitle}

\aistatsauthor{Julius Kobialka$^{*}$\\LMU Munich, MCML \And Emanuel Sommer$^{*}$\\LMU Munich, MCML \AND Chris Kolb\\LMU Munich, MCML \And Juntae Kwon\\LMU Munich \And Daniel Dold\\HTWG Konstanz \And David R\"ugamer\\LMU Munich, MCML}
\aistatsaddress{ }
]

\begin{abstract}
Bayesian neural network (BNN) posteriors are often considered impractical for inference, as symmetries fragment them, non-identifiabilities inflate dimensionality, and weight-space priors are seen as meaningless. In this work, we study how overparametrization and priors together reshape BNN posteriors and derive implications allowing us to better understand their interplay. We show that redundancy introduces three key phenomena that fundamentally reshape the posterior geometry: balancedness, weight reallocation on equal-probability manifolds, and prior conformity. We validate our findings through extensive experiments with posterior sampling budgets that far exceed those of earlier works, and demonstrate how overparametrization induces structured, prior-aligned weight posterior distributions.
\end{abstract}

\section{INTRODUCTION}

Bayesian neural networks (BNNs) have been a central tool for facilitating uncertainty quantification in deep learning by placing distributions over weights and approximating their posterior. As the posterior of such networks is usually intractable, one has to rely on simplifying variational assumptions \citep[e.g.,][]{blundell2015weight} or other approximation techniques \citep[e.g.,][]{daxberger_2021_LaplaceReduxa} to estimate these distributions. However, these approximations are rather crude and often insufficient, raising doubts about whether BNNs are truly appropriate for uncertainty quantification.

    \begin{figure}[t]
    \centering
    \vspace{0.5em}
    \setlength{\tabcolsep}{8pt}
    \renewcommand{\arraystretch}{1}
    \begin{tabular}{c}
        \hspace{-1em} \includegraphics[width=0.44\linewidth]{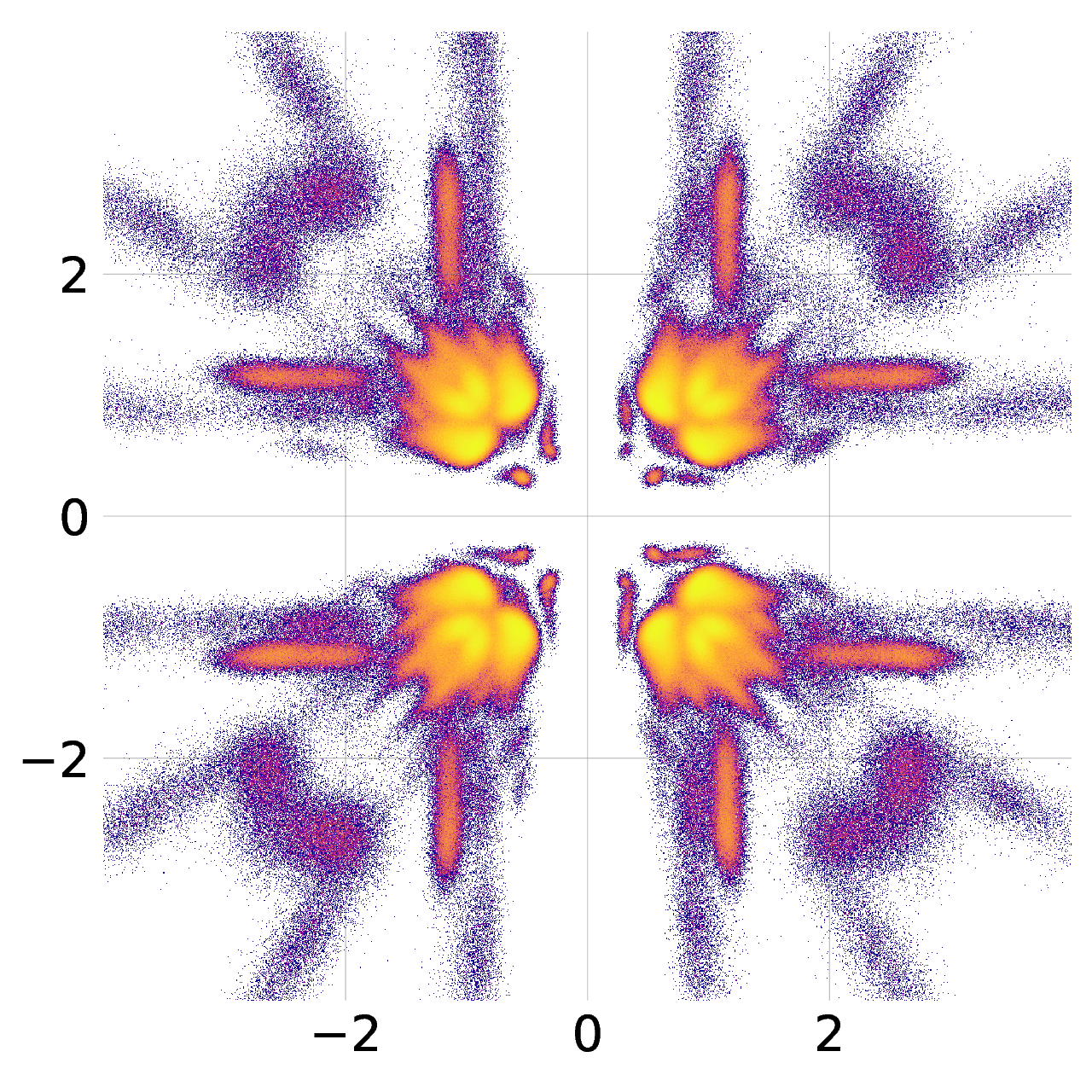} 
        \hspace{-0.3em} \includegraphics[width=0.44\linewidth]{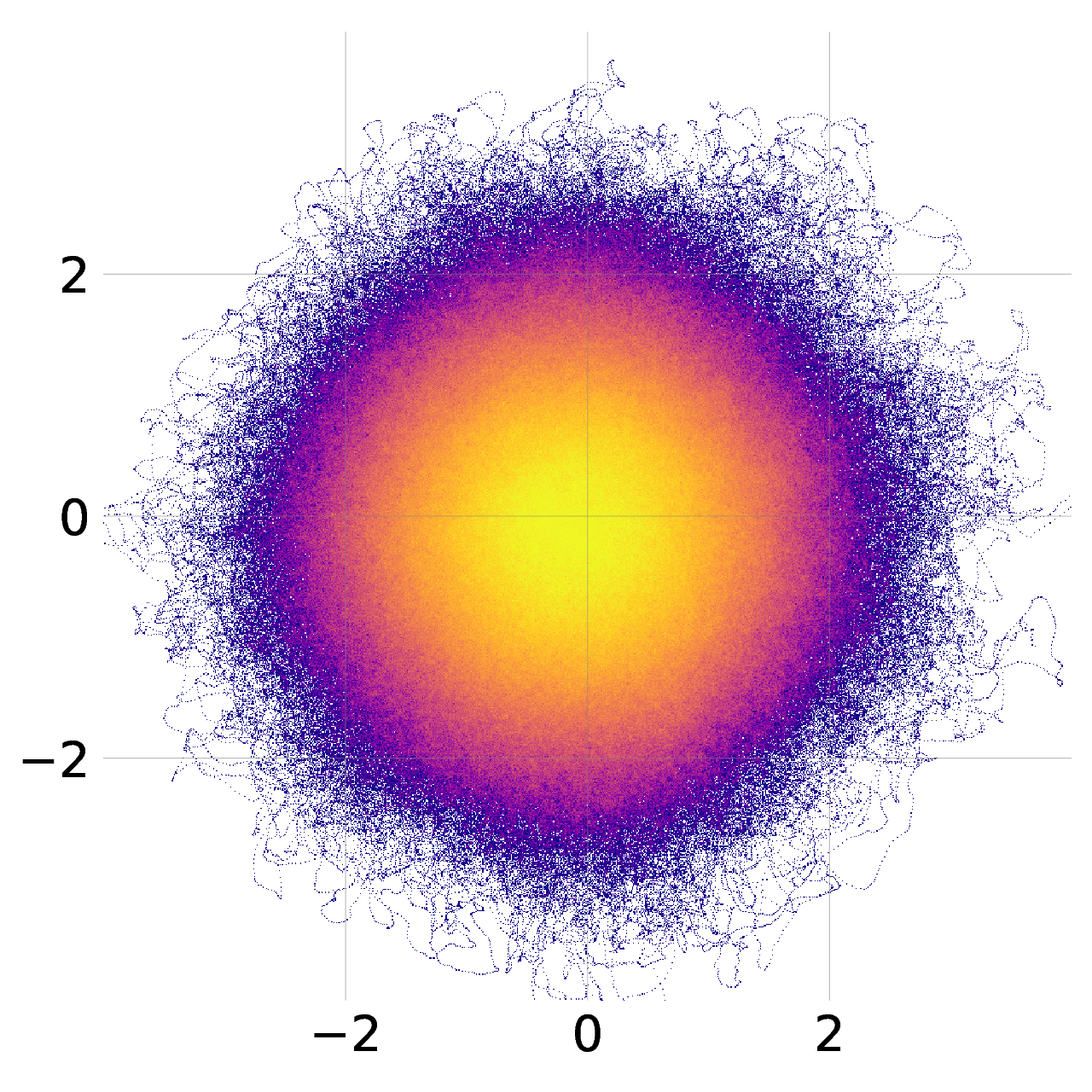}  \\
        \scriptsize Underparametrized \hspace{4.5em} Overparametrized   
    \end{tabular}
    \caption{
    Sampled marginal bivariate posterior densities of two weights (axes, lighter means higher density) for an underparametrized (left) and an overparametrized model (right) on the same task using 10 million posterior samples.
    }
    \label{fig:rohrschach}
\end{figure}
Sampling-based inference for BNNs, an alternative approach that estimates the posterior without restrictive distributional assumptions, has gained traction thanks to recent advances. These include methods to scale to large datasets \citep{adasghmc} and large models \citep{MILE}, as well as approaches improving acceptance rates \citep[e.g.,][]{papamarkou2023}, prior understanding \citep[e.g.,][]{vladimirova2019understanding}, initialization \citep{sommer2024connecting}, and scalable software implementations \citep{duffield2025scalable}. Several of these works report superior performance compared to variational methods and other approximations.

Nevertheless, criticisms in the community remain \citep[see, e.g.,][]{pmlr-v235-papamarkou24b}, questioning the adequacy of (sampling-based) inference for BNNs: \textit{Why does the posterior landscape appear less and less fragmented (cf.~\cref{fig:rohrschach}) in overparametrized models?} \textit{How does the non-identifiability affect posterior inference?} \textit{Are there meaningful priors in weight space?} \textit{What can be said about the individual weight distributions?}

In this paper, we make important advances toward answering these questions. In particular, we identify \emph{overparametrization} as the common property that shapes BNN posteriors and helps to explain the recent success of sampling.

\paragraph{Our Contributions} In this work, we theoretically and empirically investigate how overparametrization and priors shape the posterior and how this affects sampling-based inference. 
Specifically, we
\begin{itemize}[leftmargin=1.5em]
    \item translate properties from the optimization literature into statements about prior choices and posterior shape; 
    \item show how overparametrization affects the posterior by creating equal-probability manifolds;
    \item show that overparametrization induces what we call \emph{prior conformity}, yielding an easy-to-encode a priori understanding of weight distributions;
    \item conduct extensive experimental evaluations to analyze the effect of changing posterior shapes on sampling-based inference, with sampling budgets largely exceeding existing work.
\end{itemize}

\section{RELATED LITERATURE} \label{sec:relwork}

Sampling-based inference for BNNs typically relies on Markov Chain Monte Carlo (MCMC) methods. It is often considered a gold standard in the Bayesian deep learning (BDL) community \citep[see, e.g.,][]{farquhar_2020_LibertyDeptha}, theoretically allowing to sample from the true posterior (in the limit). While devising MCMC methods that can handle the high dimensionality and intricate structure of parameter spaces in modern neural networks has remained a challenge \citep{papamarkou_2022_ChallengesMarkova, wenzel_2020_HowGooda}, recent advances have made progress toward more efficient mixing of Markov chains \citep{sen2024bayesian}, scaling samplers for larger datasets \citep{paulin2025sampling,chen2014stochastic,zhang_2020_CyclicalStochastic}, and parameter spaces \citep{sommer2024connecting,MILE,SMILE}. 

\paragraph{Non-Identifiabilities and Symmetries} A particular challenge for sampling-based BNN inference lies in the fact that contemporary neural networks contain a vast number of neurons and connections that induce a high degree of redundancy in model parameters \citep{papamarkou_2022_ChallengesMarkova}. As such, the function learned by a neural network is invariant to a range of parameter transformations, which subsequently induces loss invariance \citep{ziyin2024symmetry}. In the literature, these invariant transformations are also viewed as symmetries, with the most common examples introduced in \cref{sec:background}. In the BNN literature, symmetries have been repeatedly discussed with the general premise of being harmful in posterior inference as they introduce redundancy and, thus, cause inefficiencies in posterior approximation \citep{papamarkou_2022_ChallengesMarkova, wiese2023towards, laurent2024a}. 

Proposals to deal with symmetries in neural networks include bias sorting \citep{pourzanjani_2017_ImprovingIdentifiabilityb}, assigning individual parameter offsets \citep{ziyin2024removesymmetriescontrolmodel}, or skip connections \citep{kurle2021symmetries} to remove permutation invariances; using invariant networks \citep[e.g.,][]{maron2018invariant,navon_2023_EquivariantArchitecturesa}, removing scaling symmetries via regularization \citep{laurent2024a}, or computing a model average over the elements of a symmetry orbit \citep{gelberg2024variational}. While symmetries are also known to slow down sampling \citep{nalisnick_priors_2018, papamarkou_2022_ChallengesMarkova, wiese2023towards}, only a few papers have studied symmetries in sampling-based inference systematically.

\paragraph{Bayesian Neural Networks} Likewise, there are only a few works that elucidate the connection between posterior structure, symmetries, and overparametrization in BNNs. Earlier large-scale empirical studies of sampling-based posterior approximations were able to uncover posterior mode connectivity that can be exploited by samplers \citep{izmailov_2021_WhatArea}, in line with the mode connectivity phenomenon in deep NNs \citep{garipov_2018_LossSurfacesb, draxler_2018_EssentiallyNo} and BDL methods exploiting this phenomenon, such as subspace inference \citep{izmailov_2020_SubspaceInferencea, dold2024, dold2025}. In contrast to the optimization community, where the influence of overparametrization is hypothesized to make the loss landscape more benign, work such as \citet{trippe2017overpruning} expressed concern that a potential pathology for variational approximations of overparametrized BNN posteriors can paradoxically lead to worse performance and induce conditional independence between parameters and data.

We refer the reader to \cref{app:rellit} for further discussion of related literature.

\section{THE INTERPLAY BETWEEN OVERPARAMETRIZATION AND SYMMETRIES} \label{sec:background}

We start with an instructive example to explain the interplay between overparametrization, priors, and symmetries in a simplified setup and define our notation. 

\paragraph{Setup and Notation} First, let $f(\bx;\bw)$ denote a neural network with input $\bx\in\mathcal{X}$ and parameters $\bw \in \mathcal{W} \equiv \mathbb{R}^d$, where $\bw$ stacks all weights across all layers. Throughout the paper, we will work with various architectures, but focus on homogeneous networks to better convey our findings and demonstrate the properties induced by network components. We will also use $f(\bw)$ if discussing a property that holds for all $x\in\mathcal{X}$. When considering fully-connected architectures, we will use ``$M_1$-$M_2$-$\ldots$-$M_L$'' as a short form for a network with $M_l$ neurons in layers $l\in[L]:=\{1,\ldots,L\}$ Given training data $\mathcal{D}=\{(\bx_i,\by_i)\}_{i=1}^n$, we define the empirical (unregularized) loss
   $\mathcal{L}(\bw) \;:=\; -\textstyle\sum_{i=1}^n \ell(f(\bx_i;\bw),\by_i)$,    
where $\ell$ is a pointwise loss function. In this paper, we will study BNNs that have a corresponding regularized risk learning problem
\begin{equation*}
    \arg\min_{\bw} \mathcal{L}_\lambda := \arg\min_{\bw} \Big( \mathcal{L}(\bw) + \lambda \mathcal{R}(\bw) \Big),    
\end{equation*}
with regularizer $\mathcal R(\bw)$ and tuning parameter $\lambda>0$.

From a Bayesian perspective, we interpret $\mathcal{L}(\bw)$ as the negative log-likelihood of the data given the weights $\mathcal{L}(\bw) \;=\; -\sum_{i=1}^n \log p(\by_i \mid f(\bx_i;\bw))$ with $f$ typically encoding the mean of the distribution, and assume a prior $p(\bw) \propto \exp(-\lambda\mathcal{R}(\bw))$. Then the posterior $p(\bw|\mathcal{D}) \propto \exp(-\mathcal{L}_\lambda)$. A standard choice we will not only use throughout this paper but also advocate for is the isotropic Gaussian prior $p(\bw) \propto \exp\!\left(-\tfrac{1}{2\tau^2}\|\bw\|_2^2\right)
$ with $\lambda = 1/(2\tau^{2})$. If prior variances differ across layers $l\in[L]$ of an $L$-layer network, we will make this explicit by writing $\mathcal{L}_{{\bm{\tau}}}$ with $\bm{\tau} = \{\tau_l\}_{l\in[L]}$. We will use the Bayesian and regularized risk point of view often interchangeably when talking about properties of the weight space.

\paragraph{Overparametrization} Given the recent success of warmstarted sampling-based inference \citep[see, e.g.,][]{MILE}, we will investigate BNNs that have more hidden neurons than necessary to represent the underlying data-generating process. This is in line with other works studying overparametrization, such as \citet{nguyen2019connected} or \citet{kim2025exploring}. An alternative layerwise property to define overparametrization is to assume that the activation matrix of the $l$-th layer has rank strictly smaller than its number of neurons. Formal definitions are given in \cref{app:defs}.

\subsection{Instructive Overparametrization Example}

In the following instructive example, we will provide some intuition on how overparametrization interacts with symmetries and introduce concepts required for our more general exposition in subsequent sections. 
As an extension of \textit{single-neuron linear networks} \citep{kunin2024get, azulay2021implicit}, we consider a linear neural network for scalar inputs given by
\begin{equation} \label{eq:oneMone}
f(x;\bw) = \textstyle\sum_{m=1}^M w_{1,m} w_{2,m} x, \quad x\in\mathbb{R},
\end{equation}
which can be interpreted as a shallow network with $M$ hidden units (a \emph{1}-$M$-\emph{1} network). It is intuitively clear why $f$ is \emph{overparametrized}: the space of functions that can be represented by $f$ is the hypothesis space $\{h: h(x) = x\beta,  \beta\in\mathbb{R}, x\in\mathcal{X}\}$, i.e., a univariate linear model. Yet, $f$ uses $2M$ instead of just one parameter. As a result, $f$ admits various symmetries: 
\begin{enumerate}[leftmargin=1.5em]
    \item a (positive) rescaling symmetry as $\forall c>0$ and $\forall m\in[M]$, we have that $f(\bw) = f(\tilde{\bw})$ for $\tilde{\bw} = (w_{1,1}, w_{2,1},\ldots, c \cdot w_{1,m}, c^{-1} \cdot w_{2,m},\ldots,w_{1,M},w_{2,M})$; 
    \item as a special case of the above, a sign-flip symmetry as $f(\bw) = f(\tilde{\bw})$ for $\tilde{\bw}$ containing at least one pair of flipped signs, i.e., $\tilde{\bw} = (\ldots, - w_{1,m}, -w_{2,m},\ldots)$; \item a permutation symmetry, since $f(\bP\bw) = f(\bw)$ for any block-permutation matrix $\bP$ that permutes hidden units as pairs; 
    \item a rotation symmetry, since $f$ only depends on inner products. In particular, for any orthogonal matrix $\bR$, $f(\bw) = \bw_1^\top \bw_2 x = (\bR \bw_1)^\top (\bR \bw_2) x = f(\tilde{\bw})$ with $\bw_i = (w_{i,1},\ldots,w_{i,M})$ and $\tilde{\bw}$ being the vector $\bw$ containing the rotated components.
\end{enumerate}

\paragraph{The Bayesian Perspective} From a Bayesian point of view, the existence of symmetries in $f$ implies invariance to certain transformations. If the inferential target is merely $f$, symmetry-respecting priors are the principled baseline encoding knowledge about $f$.
More specifically, the matching prior for rescaling symmetries must encode a Cartesian product of scale-invariant improper prior distributions over the product of weights for every weight combination that admits rescaling symmetries, e.g., $p(w_j,w_{j'}) \propto (w_j w_{j'})^{-1}$. This is a rather atypical choice in BNNs, and as a consequence, most posterior densities will not be scale invariant. 
In contrast, every symmetric prior admits sign-flip symmetries. If $f$ is sign-flip invariant, we have $p(w_{j}|\mathcal{D}) = p(-w_{j}|\mathcal{D})$, and thus $\mathbb{E}(w_{j}|\mathcal{D}) = 0$ for every marginal posterior of $w_{j}$. Permutation symmetries imply that the (co-)variance structure of the permuted weights is identical.
In particular, heterogeneous weight priors imply a permutation variant posterior density. Finally, rotations, which can be viewed as a continuous generalization of permutations, are challenging to encode in general. For specific examples as shown in the next section, however, an independent Gaussian distribution $\bw \sim \mathcal{N}(0,\tau^2 \bI)$ preserves such symmetries. 

\subsection{The Influence of Regularization} 

The previously described conformity between symmetries and weights can also be regarded from a regularization perspective. While the \emph{1}-$M$-\emph{1} network might admit previously discussed symmetries, this is not necessarily the case for the loss function $\mathcal{L}_\lambda$, in particular for $\lambda>0$ and rescaling symmetries. In contrast to $f$, the regularizer $\mathcal{R}$ in $\mathcal{L}_\lambda$ is typically not invariant under rescaling. For $L_2$-regularization, we obtain for the example above that for any $\bw \in \mathbb{R}^{2M}$ there exists a rescaled version 
$\tilde\bw\in\mathbb{R}^{2M}$ for which $\mathcal{L}_\lambda(\tilde\bw) \;\leq\; \mathcal{L}_\lambda(\bw)$.
That is, the penalty breaks the rescaling invariance and prefers \emph{balancedness}. For $M=1$, this means that it must hold $w_{1,1}^2 = w_{2,1}^2 = |{\beta}|$ with $\beta = w_{1,1}w_{2,1}$, as the regularizer would otherwise be larger while $f$ stays constant. This can be seen by applying the AM-GM inequality: $\mathcal{R}(\bw) = \frac{1}{2} (w_{1,1}^2 + w_{2,1}^2) \geq \sqrt{w_{1,1}^2 w_{2,1}^2} = |\beta|$, where equality holds iff $|w_{1,1}| = |w_{2,1}| = \sqrt{|\beta|}$. In general, for any homogeneous unit, we obtain this balancedness at optimality \citep{parhi2023deep}. 

For $M>1$ and $\beta = \bw_1^\top \bw_2$, however, exact balancedness is only enforced between product factors, i.e., $|w_{1,m}| = |w_{2,m}| \,\forall m\in[M]$ and not across neurons. This can be understood as a result of the function's rotation symmetry, yielding the following global minimizers of $\mathcal{R}(\bw)$ for a fixed $\beta$: $$\bw_1=\pm\sqrt{|\beta|} \cdot \bm{v},\, \bw_2=\pm\text{sign}(\beta)\sqrt{|\beta|} \cdot \bm{v}$$ with $\bm{v} \in \mathbb{S}^{M-1} := \{\bm{v} \in \mathbb{R}^M: ||\bm{v}||_2 = 1\},$ using $\text{sign}(\beta)$ for $\bw_2$ without loss of generality. 

\paragraph{The Bayesian Perspective} The zero-centered isotropic Gaussian prior as the equivalent of an $L_2$ regularization does not preserve scaling symmetries. However, it preserves sign-flip symmetries (and rotation symmetry in this example). As we will see in the following sections, this trait and the balancedness among weight norms imply important properties.

\subsection{The Influence of Overparametrization} \label{sec:infloverparam}

In \cref{fig:oneMone:a}, we plot the resulting posterior from a bivariate marginal point of view for two components $w_{1,m},w_{2,m}$ for different $M$. This reveals an interesting effect: While the scaling and permutation symmetries are clearly visible for $M=1$, producing a butterfly-shaped posterior when the penalty (prior) is not too strong, increasing $M$ causes the rotation symmetry to dominate, pushes the posterior mode closer to zero, and decreases the anisotropy of the distribution. 
\begin{figure}[!htb]
    \centering
        \includegraphics[width=0.95\linewidth]{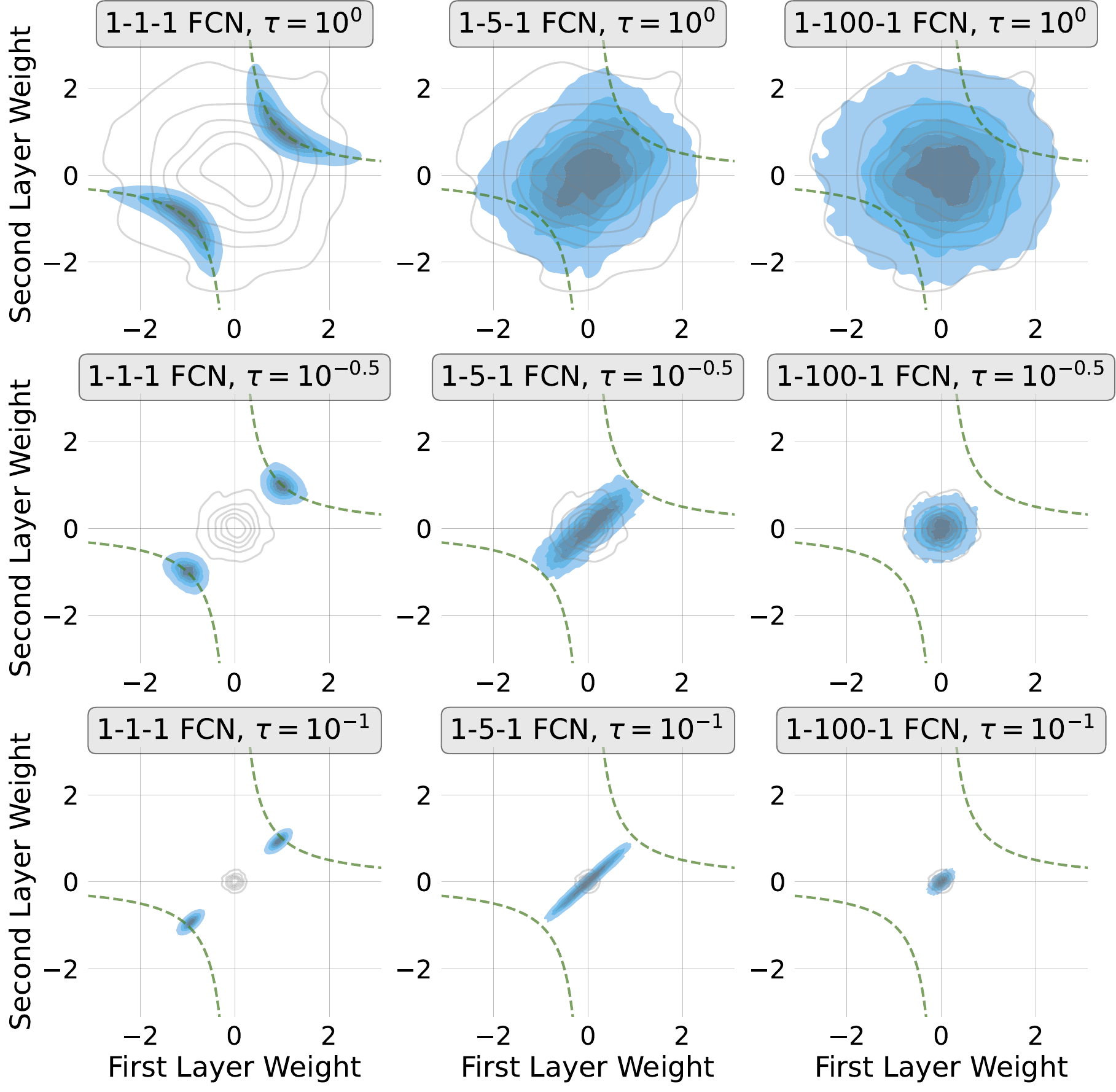}
    \caption{Posterior marginals of \emph{1}-$M$-\emph{1} fully-connected network (FCN) with zero-mean Gaussian priors and different variances $\tau^2$ and data $y = x +\epsilon$, $\epsilon \sim N(0, 0.1)$. }
    \label{fig:oneMone:a}
\end{figure}
More specifically, under the constraint $\sum_{m} u_{m} := \sum_{m=1}^M w_{1,m} w_{2,m}  = \beta$ , i.e., an optimal likelihood, and using the prior $w_{j,m}\sim\mathcal{N}(0,\tau^2)$, we have $\mathbb{E}(w_{j,m} | \sum_m u_m = \beta) = 0$, $\operatorname{Var}(w_{j,m} | \sum_m u_m = \beta) = \tau^2 + O(M^{-2})$, and 
\begin{equation*} \label{eq:covu}
\text{Cov}(w_{j,m},w_{j',m}|\textstyle\sum_m w_{1,m} w_{2,m} = \beta) = \beta/M.
\end{equation*}
We will formalize this result in the next section.

\paragraph{The Bayesian Perspective} From a Bayesian perspective, increasing the overparametrization in this linear network case implies that when in high likelihood regions, posterior variances approach the prior variances, and the covariance between two weights vanishes, as can be nicely seen in \cref{fig:oneMone:a}.

\section{GENERAL RELU NETWORKS} \label{sec:global}

Moving from the previous instructive example to more general ReLU networks, we investigate the existence of these properties in more realistic architectures. As before, we analyze BNNs that obtain constant high likelihood values, which allows us to reason about their local posterior geometry in regions that ultimately matter the most,  which is again in line with recently proposed warmstarting approaches. In contrast to the previous section, where this was formalized by assuming an underlying linear model, we here generalize the idea by contrasting an overparametrized model $f$ with excessive hidden neurons with an interpolating network $f^\ast$ that achieves minimum norm (i.e., the smallest model that can still fit the data and adheres to the prior). Due to non-identifiabilities, many models $f$ can achieve the same loss as $f^\ast$, but not all are minimum norm solutions. We therefore define the following:

\begin{definition}[Manifold of minimum norm interpolators] \label{def:mnman}
    Given an $L$-layer network $f$ with weights $\bw$ and interpolating $L$-layer network $f^\ast$ with weights $\bw^\ast$ attaining the minimum norm among all interpolants of $L$-layer networks, the \emph{minimum norm manifold} is defined by 
    \begin{equation}
        \mathcal{M} \!=\! \{\bw\in\mathbb{R}^d: \mathcal{L}_\lambda(\bw) = \mathcal{L}_\lambda(\bw^\ast)\},
    \end{equation}
\end{definition}
where we assume $\mathcal{R}$ to be a (weighted) $L_2$-regularization (i.e., a diagonal Gaussian prior). 

Overparametrization then means that each neuron $\varpi \in [M^\ast]$ of $f^\ast$ can be replicated by a group $G_\varpi \subseteq [M]$ of $k_\varpi = |G_\varpi|$ neurons in $f$, as described by a surjective assignment map $\varsigma: [M] \to [M^\ast]$. The key quantity of interest becomes the vector of reallocation coefficients $\bm{\rho}^{(\varpi)}$, which describes how the norm of neuron $\varpi$ in $f^\ast$ is distributed among its copies in $f$, which is the higher-dimensional analogue of the equal spreading of weight norm observed in \cref{sec:infloverparam}. That is, each neuron $m \in G_\varpi$ in $f$ takes the form $\bm\omega_m = \sqrt{\rho_m}\,\bm\omega_{\varpi}^\ast$, where $\rho_m \geq 0$ denotes the fraction of the squared norm of neuron $\varpi$ allocated to copy $m$, with $\sum_{m \in G_\varpi} \rho_m = 1$. Since ReLU is positive homogeneous, this rescaling preserves the function (see \cref{lem:simplex}).

As we cannot condition the posterior on $\mathcal{M}$ in the usual way, we define an induced (conditional) law via shrinking tubes around $\mathcal{M}$ (an $\varepsilon$-neighborhood $\mathcal{M}^\varepsilon$). This yields a way to study the ambient space around the minimum-norm manifold $\mathcal{M}$.

\begin{definition} \label{def:mnman_tube}
For any $\varepsilon > 0$, define the $\varepsilon$-tube around a manifold
$\mathcal{M} \subseteq \mathcal{W}$ as
\[
\mathcal{M}^{\varepsilon} := \bigl\{\mathbf{w} \in \mathbb{R}^d :
\operatorname{dist}(\mathbf{w}, \mathcal{M}) \leq \varepsilon \bigr\},
\]
where the distance is measured in the Euclidean norm. 
\end{definition}

For each $\varepsilon > 0$, the tube has positive Lebesgue measure, admitting:
\[
\mathbb{P}_n^{\varepsilon}(A) := \mathbb{P}_n\bigl(\mathbf{w} \in A
\mid \mathbf{w} \in \mathcal{M}^{\varepsilon}\bigr), \qquad A \subseteq
\mathcal{M},
\]
where $\mathbb{P}_n(d\mathbf{w}) = p(\mathbf{w} \mid \mathcal{D}_n)\, d\mathbf{w}$
is the ordinary posterior given $n$ observations.

\subsection{Two-Layer ReLU Networks} \label{sec:twolayer}

The general shallow $p$-$M$-\emph{1} ReLU network $f(\mathbf x) = \sum_{m=1}^M  w_{2,m}^\top \phi (\mathbf w_{1,m}^\top \mathbf{x})$ with ReLU activation functions $\phi$ and multivariate input $\bx \in \mathbb{R}^p$
has been intensively studied in the literature and admits useful properties such as convexifiability \citep[see, e.g.,][]{mishkin2022fast} and connectedness of solution sets in the overparametrized regime \citep{kim2025exploring}. To generalize previously made observations of balanced weight behavior as $M$ increases, we obtain the following:
\begin{theorem}[informal] \label{th:simplex}
    Under \cref{ass:interpol} (an overparametrized $p$-$M$-\emph{1} model $f$ and a minimal norm cost interpolant $f^\ast$ with  $M^\ast < M$ neurons), the coefficients $\bm{\rho}^{(\varpi)}\in\Delta^{k_{\varpi}-1}$ of any norm-preserving reallocation $\varsigma:[M] \to [M^\ast]$ of the weight strength of neuron $\varpi\in[M^\ast]$ of $f^\ast$ to a group of weights $m \in G_{\varpi}$ with $k_\varpi = |G_\varpi|\geq 1$ neurons from a model $f$ on $\mathcal{M}$ follows a symmetric Dirichlet distribution.
\end{theorem}

In order to move towards conclusions about the posterior of $\bw$, we use \cref{def:mnman_tube} to regard the ambient space of $\mathcal{M}$. For this, we first define $\bm{\omega}_m := (\bw_{1,m}, w_{2,m}) \in\mathbb{R}^{p+1}$ to be the block of weights (in- and outgoing weights) corresponding to the $m$-th neuron in the hidden layer and $\{\bm{\omega}_m\}_{m \in G_{\varpi}}$ the block of neurons in $G_\varpi$ that overparametrize neuron $\bm{\omega}_\varpi$.

\begin{corollary}\label{cor:alpha_discrepancy}
Consider a $k_\varpi$-fold split of a neuron $\bm{\omega}^\ast_{\varpi}$, $\varpi\in[M^\ast]$ into $k_\varpi$ neurons
$\bm\omega_m=\sqrt{\rho_m}\,\bm\omega_{\varpi}^\ast$ with $\bm\rho^{(\varpi)} := (\rho_m)_{m \in G_{\varpi}}\in\Delta^{k_{\varpi}-1}$.
Then the induced distribution of $\bm{\rho}^{(\varpi)}$ on $\mathcal{M}^\epsilon$ is symmetric Dirichlet,
\[
\bm{\rho}^{(\varpi)} \sim \mathrm{Dirichlet}\Big(\tfrac{p+1}{2},\ldots,\tfrac{p+1}{2}\Big),
\]
for each $\varpi=1,\dots,M^\ast$ independently under tube-conditioning as in \cref{def:mnman_tube}.
\end{corollary}

Although the $\varepsilon$-tube $\mathcal{M}^\epsilon \subset \mathbb{R}^d$ resides within the ambient parameter space $\mathcal{W}$, the Dirichlet result strictly isolates the marginal distribution of the reallocation coefficients $\bm{\rho}^{(\varpi)}$ near $\mathcal{M}$.

If we assume the setup from \cref{th:simplex} and consider the minimum-norm manifold $\mathcal{M}_\varsigma \subseteq \mathcal{M}$ for a specific assignment or reallocation map $\varsigma$, we can define the induced posterior with $\varepsilon$-tube:
\begin{corollary} \label{cor:simplexprob}
    On $\mathcal{M}_\varsigma$, let $\mathbb{P}_n := \lim_{\varepsilon \downarrow 0}\, \mathbb{P}_n^{\varepsilon}$, assuming this weak limit exists. It holds (1) $\mathbb{E}_{\mathbb{P}_n}(w) = O(M^{-1/2})$, (2) $\operatorname{Cov}_{\mathbb{P}_n}(w,w') = O(M^{-2})$.
\end{corollary}

These results follow from \cref{th:simplex} and imply that the posterior around this manifold will move closer and closer to zero and induce increasingly less dependence (covariance) between weights the more the hidden layer ($M$) grows. 
This statement holds uniformly across assignments $\varsigma$ with comparable group sizes, together with the plausible assumption that the posterior places substantial mass near $\mathcal{M}$. Concurrently, because the likelihood remains nearly constant in the immediate neighborhood of the manifold, the $L_2$ regularization (Gaussian prior) induces approximately quadratic fluctuations strictly in the directions normal to the manifold.  This may therefore explain why highly redundant, one-dimensional weight marginals resemble zero-centered Gaussian distributions, as can be seen in \cref{fig:univariate_overparam_density_concrete}. We emphasize, however, that this analysis is local to the neighborhood of $\mathcal{M}$ and does not by itself determine the shape of the full marginal posterior.

While the preceding analysis is specific to two-layer networks, the phenomena it predicts, namely prior-like marginals, vanishing covariance, and modes near zero, are clearly visible in deeper architectures as well (cf.\ \cref{fig:univariate_overparam_density_concrete,fig:ionospherefireballgrid_main}). The following section explains why these patterns persist beyond the shallow setting.

\begin{figure}[t]
    \vspace{-0.2cm}
    \centering
    \setlength{\tabcolsep}{2pt}
    \renewcommand{\arraystretch}{0}
    \begin{tabular}{cc}
        \includegraphics[width=0.44\linewidth]{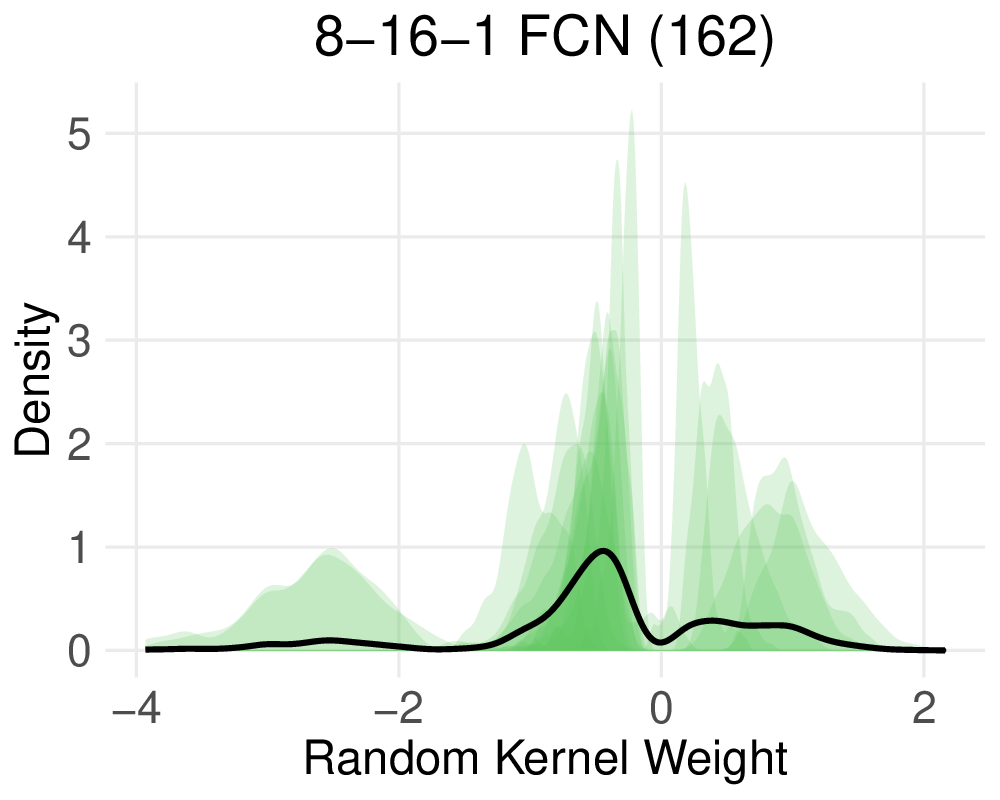} &
        \includegraphics[width=0.44\linewidth]{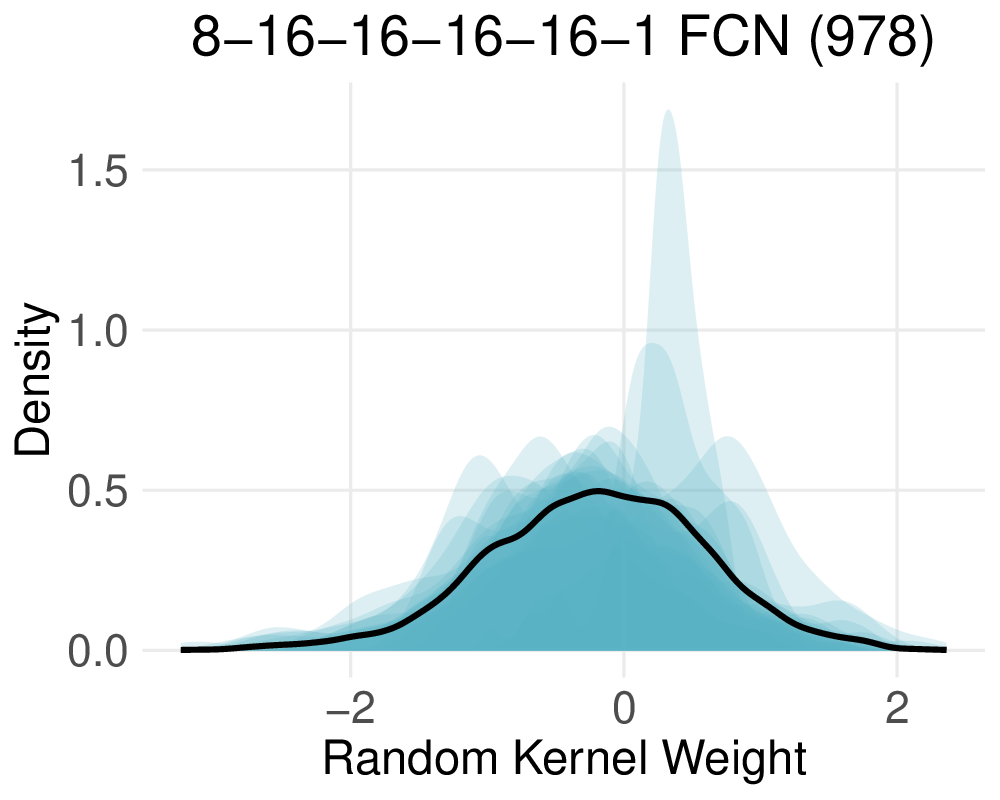} \\
    \end{tabular}
    \vspace{-0.2cm}
    \caption{Univariate marginal distributions of selected first hidden layer weights for sampled BNNs and their change with varying network size (number of parameters in brackets) on the \texttt{concrete} regression task. The 20 chain-wise distributions are overlaid with transparency, and the aggregated joint marginal is shown as a solid black line. Despite similar downstream performance, the marginals differ markedly between the two network parameterizations.}
    \label{fig:univariate_overparam_density_concrete}
\end{figure}

\subsection{Properties of Deeper Networks} \label{sec:deep}

As minimum-norm manifolds become harder to characterize explicitly, we will move from the tube-based analysis to a more general treatment of BNNs based directly on the posterior $\pi$.   

\paragraph{Balancedness} \label{sec:balance}

The following exposition assumes 
\begin{equation} \label{eq:f_balance}
    f(\bx) = \bW_L \phi(\bW_{L-1}\phi(\cdots(\bW_1\bx)))
\end{equation}
with ReLU activation functions $\phi$ and arbitrary depth $L$. Let $\bw=(\bW_1,\dots,\bW_L)$ collect the layer-wise weights and $d_l:=|\text{vec}(\bW_l)|$, the posterior $\pi:= p(\bw|\mathcal{D}) \propto \exp(-\mathcal{L}_{\bm{\tau}}(\bw))$, and assume zero-mean Gaussian priors with variances $\tau_l^2$ for layer $l \in [L]$. 
One can view depth in such homogeneous networks as a form of \emph{multiplicative} overparametrization: positive homogeneity of ReLU induces a positive rescaling symmetry, introducing one continuous degree of freedom per hidden neuron per layer boundary. Whereas \cref{sec:twolayer} characterized how the Gaussian prior organizes \emph{additive} redundancy, the following balancedness result relates to the analogous mechanism for multiplicative redundancy, where the isotropic Gaussian prior breaks the rescaling invariance by penalizing unequal norms across adjacent factors.

For this, we utilize a structural property of homogeneous networks derived in \citet{du2018algorithmic}. In their analysis of gradient flow, the authors establish the following pointwise algebraic identity for all $\bw$ where the gradient exists:
\begin{equation*}
\langle \bW_l, \nabla_{\bW_l}\mathcal{L}(\bw)\rangle_F = \langle \bW_{l+1},\nabla_{\bW_{l+1}}\mathcal{L}(\bw)\rangle_F.
\end{equation*}
Since this equality holds for all $\bw$ due to the 1-homogeneity of the ReLU activation and the chain rule, it necessarily holds a.s.\ in expectation over the stationary distribution $\pi$ assuming it exists and fulfills standard regularity conditions:
\begin{equation*}
\mathbb{E}_\pi[\langle \bW_l, \nabla_{\bW_l}\mathcal{L}(\bw)\rangle_F]=\mathbb{E}_\pi[\langle \bW_{l+1},\nabla_{\bW_{l+1}}\mathcal{L}(\bw)\rangle_F].
\end{equation*}

\begin{remark}
If the continuous-time weight vector $\bw(t)$, e.g., follows the overdamped Langevin diffusion
and is initialized at stationarity, then $\bw(t)\sim\pi$ for all $t$ and the above identities hold for expectations of
observables evaluated at any fixed time. If, in addition, the diffusion is ergodic, then time averages along a single trajectory converge to the corresponding $\pi$-expectations.
\end{remark}
Using the equality of expectations, we have the following balancedness result, which is also demonstrated visually for a 4-hidden layer BNN in \cref{fig:ionospherefireballgrid_main}.
\begin{figure}[t]
    \centering
    \setlength{\tabcolsep}{0pt}
    \renewcommand{\arraystretch}{0}
    
    \begin{tabular}{c@{}c@{}c@{}c@{}c@{}c}
        &
        \includegraphics[width=0.2\linewidth]{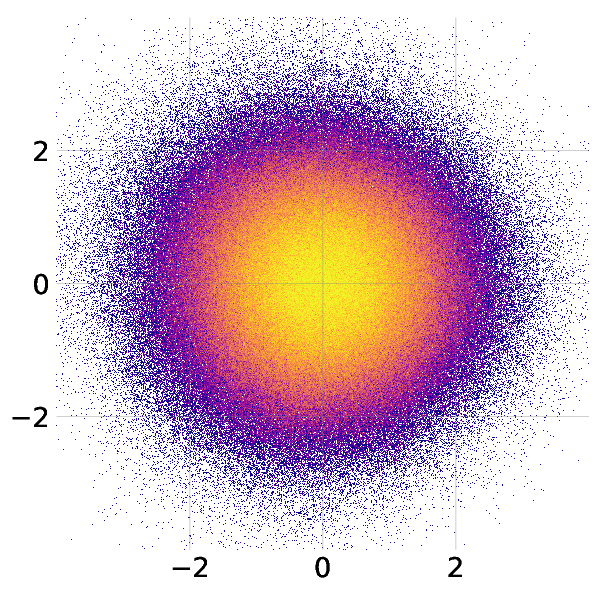} &
        \includegraphics[width=0.2\linewidth]{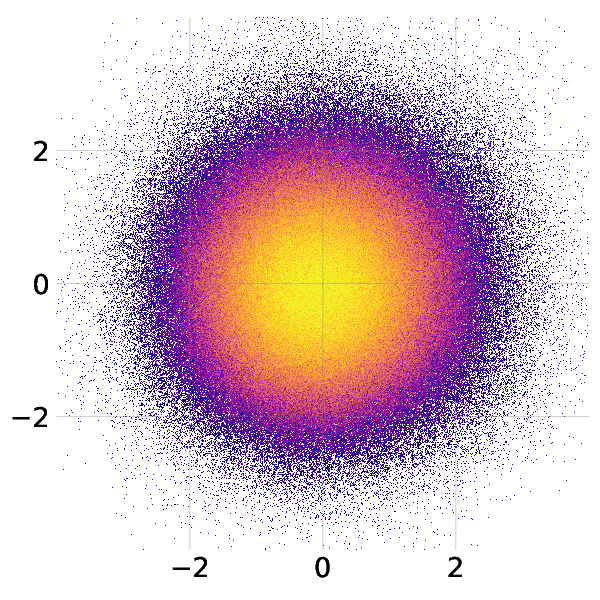} &
        \includegraphics[width=0.2\linewidth]{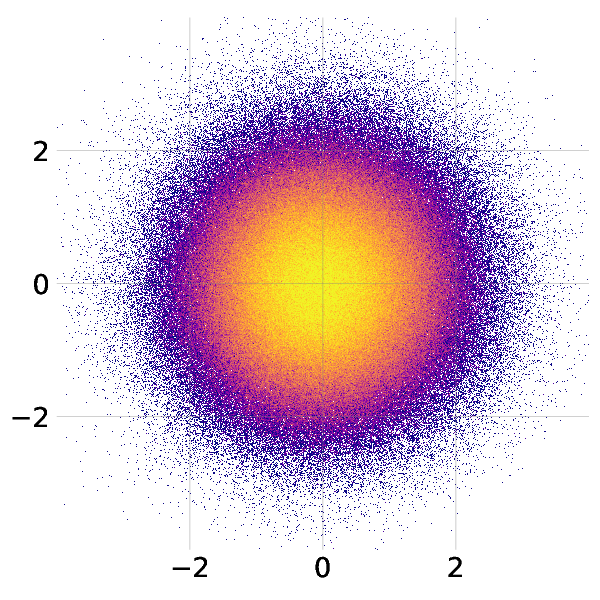} &
        \includegraphics[width=0.2\linewidth]{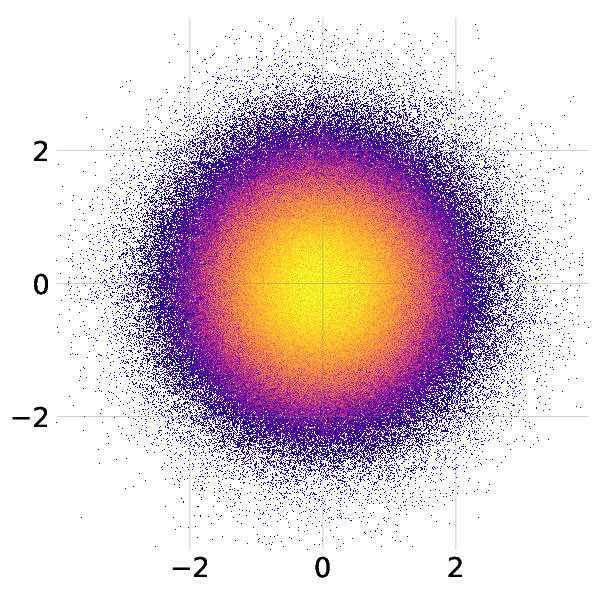} &
        \includegraphics[width=0.2\linewidth]{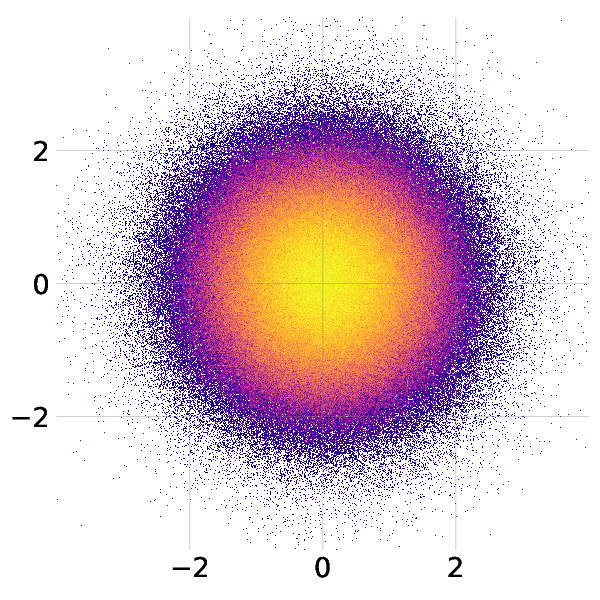} \\
        &  \tiny First &  \tiny 1st Hidden &  \tiny 2nd Hidden &  \tiny 3rd Hidden &  \tiny Last
    \end{tabular}
    \caption{Bivariate marginal posterior densities of a 4-hidden layer BNN fitted on the \texttt{ionosphere} dataset. The grid visualizes the empirical densities of 8 million posterior samples obtained from 8,000 independent chains. Columns display representative densities of different layer weights.}
    \label{fig:ionospherefireballgrid_main}
\end{figure}
\begin{theorem} \label{th:balance}
For $f$ as defined in (\ref{eq:f_balance}) and every adjacent pair $(l,l{+}1)$ it holds at stationarity that
\[
\tau_l^{-2}\,\mathbb E_\pi\!\big[\|\bW_l\|_F^2\big]
\;-\;
{\tau_{l+1}^{-2}}\,\mathbb E_\pi\!\big[\|\bW_{l+1}\|_F^2\big]
\;=\;
d_l - d_{l+1}.
\]
\end{theorem}

As a direct consequence, we obtain the following.

\begin{corollary} \label{cor:equalvar}
If $\tau_l^2\equiv\tau^2 \,\forall l\in[L]$, then $\mathbb E_\pi\!\big[\|\bW_l\|_F^2\big] - \mathbb E_\pi\!\big[\|\bW_{l+1}\|_F^2\big] =\tau^2\,(d_l-d_{l+1})\quad\forall l\in[L-1]$.
In particular, if also $d_l=d_{l+1}$, then $\mathbb E_\pi\!\big[\|\bW_l\|_F^2\big]=\mathbb E_\pi\!\big[\|\bW_{l+1}\|_F^2\big]$. Furthermore, if $\mathbb{E}_\pi[\bW_l] = \mathbb{E}_\pi[\bW_{l+1}] =  \bm{0}$, we have $\operatorname{tr}\!\big(\operatorname{Cov}_\pi(\operatorname{vec}(\mathbf W_l))\big)
=\operatorname{tr}\!\big(\operatorname{Cov}_\pi(\operatorname{vec}(\mathbf W_{l+1}))\big)$.
\end{corollary}

Even when not setting variances to the same value across layers, we obtain an expected balancedness across the network (cf.~\cref{cor:balanced} in \cref{app:proofs_derivations}).

\paragraph{Overparametrization and Prior Conformity} \label{sec:prior_curvature}

A second phenomenon that arose in the previous exposition in single-hidden layer networks was what one could call \emph{prior conformity}, i.e., the tendency for (marginal) weight distributions to resemble the (marginal) prior.
We can try to explain this phenomenon by data-independent directions in the posterior using a local approximation, similar to recent results such as \citet{roy2024reparameterization}. For this, assume a locally optimal reference parameter $\bw^\text{ref}\in\mathcal{W}$ 
and Laplace approximation based on a generalized Gauss-Newton precision matrix
\[
p(\bw\mid\mathcal D)\;\approx\;\mathcal N\!\big(\bw^\text{ref},\,(\bH^\star)^{-1}\big),
\quad \bH^\star := \bH_{\text{data}}^\star + 2\lambda \bI.
\]
Further, let $\bJ\in\mathbb{R}^{n\times d}$ be the Jacobian of the network outputs w.r.t.\ $\bw$ at $\bw^\text{ref}$ (stacked over all $n$ training samples), and $\bH^\star \approx \bJ^\top \bm\Upsilon \bJ + 2\lambda \bI$ for a positive semidefinite weight matrix $\bm\Upsilon$ (e.g., $\bm\Upsilon = \sigma^{-2} \bI$ for $L_2$-loss). If layer $l$ is overparametrized, then 
$
\ker(\bJ)\;\neq\;\{\bm{0}\}
$ and 
$$(\bH^\star)^{-1}\big|_{\ker(\bJ)} \;=\; (2\lambda)^{-1} \bI = \tau^2 \bI.
$$
Hence, the posterior covariance equals the prior covariance along likelihood-flat directions. This can be interpreted as a \emph{prior conformity} of the posterior as the Gaussian prior ``fills in'' singularities in redundant subspaces. Connected to this result are approximate rotational invariances. 

Let $\mathcal{W}=\ker(\bJ)\oplus \img(\bJ)$ be the orthogonal decomposition at $\bw^\text{ref}$ and $\bP_{\ker(\bJ)}$ the projection matrix projecting into $\ker(\bJ)$. 
Further, let $\bZ$ contain the orthonormal columns spanning $\ker(\bJ)$, i.e., $\bP_{\ker(\bJ)} = \bZ\bZ^\top$.
Then, it trivially holds
\[
\bZ^\top(\bw-\bw^\text{ref})\;\sim\;\mathcal N\!\big(\bm 0,\,\tau^2\bI\big),
\]
independent of the data. 
To confirm these statements, we can compute $\bZ$ and project all posterior samples onto $\ker(\bJ)$ to check whether the distribution of the projected and centered weights aligns with the prior (as demonstrated in~\cref{fig:kerJ}). 

When specialized to the shallow ReLU setting, $\ker(\bJ)$ at a minimum-norm solution contains the tangent directions of $\mathcal{M}$ along which neurons can be rebalanced without changing the function, but also includes additional directions that preserve the function only to first order without necessarily preserving the norm. The Laplace argument thus captures a broader class of redundant directions than the Dirichlet results, though with stronger assumptions.

\begin{figure}[t]
    \centering
    \includegraphics[width=0.9\linewidth]{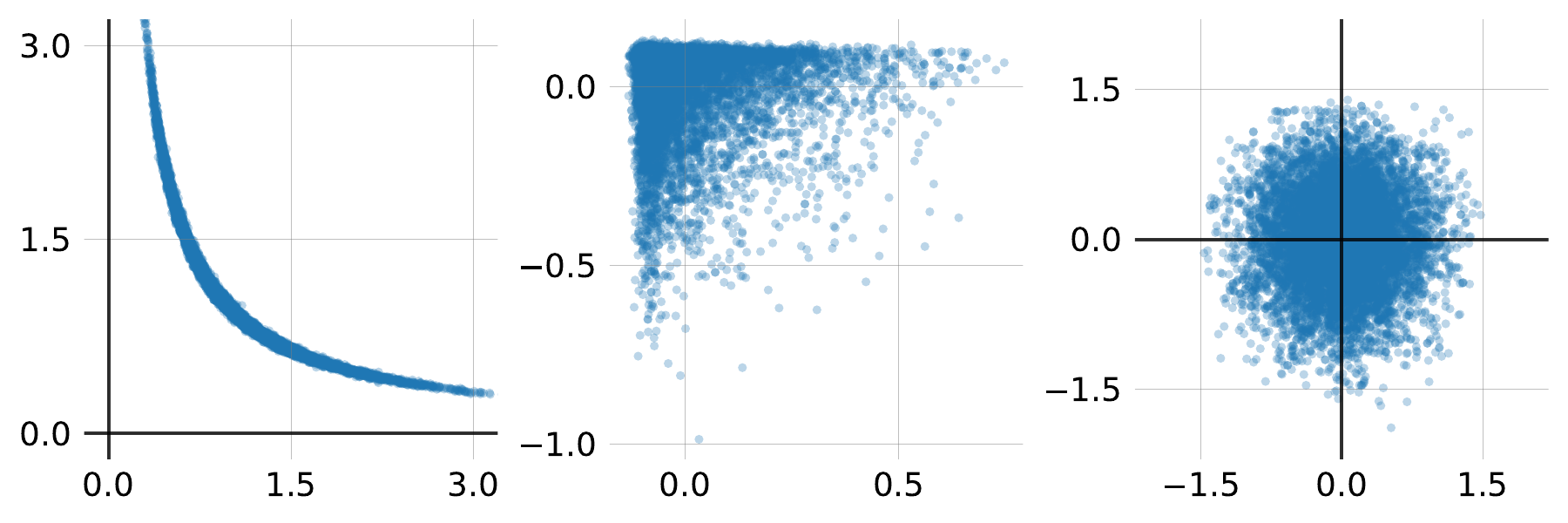}
    \caption{Samples from a 1-2-1 ReLU network on a linear dataset (left) and its separation into image (center) and null space of $\bJ$ (right).}
    \label{fig:kerJ}
\end{figure}

\subsection{The Role of the Bias} \label{sec:deep_bias}

In the previous sections, our exposition mainly focused on layers without a bias. 
When incorporating a bias into layers, e.g., $f(\bx) = \bw_2^\top \phi(\bW_1\bx + \bb_1)$, regularization effects induced by the prior still enforce $\|(\bW_{1})_{[i,:]}\|_2 \;=\; |w_{2i}|$ at optimality. This is, however, only the case if $\bb_1$ is not regularized \citep{parhi2023deep}, i.e., when using an improper flat prior $p(\bb_1) \propto 1$. In general, homogeneity is broken at the entry and exit points of the network: the first layer receives a shift in its input space, and the last layer produces an affine rather than homogeneous mapping, e.g., encoding $\mathbb{E}(y|\bx=\bm{0})$ in a regression setting. As a consequence, the induced posterior distributions, typically in the first and last layers, deviate from the balanced patterns described above and display different marginal shapes (cf.~\cref{fig:maingrid}). With suitable activations in the first layer, however, bias expectations will be zero in subsequent layers, and the ReLU nonlinearity can continue to be positively homogeneous, implying the same rescaling invariances as before.

\begin{figure}[t]
    \centering
    \setlength{\tabcolsep}{0pt}
    \renewcommand{\arraystretch}{0}
    \begin{tabular}{c@{}c@{}c@{}c@{}c@{}c}
        \rotatebox{90}{\quad\, \tiny Kernel} &
        \includegraphics[width=0.19\linewidth]{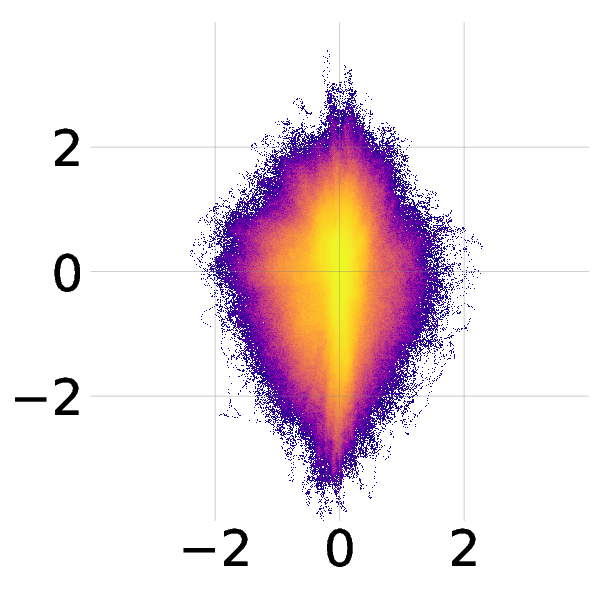} &
        \includegraphics[width=0.19\linewidth]{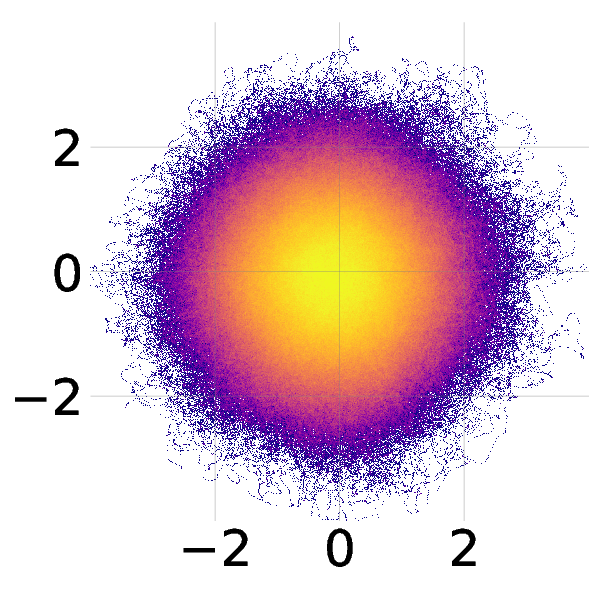} &
        \includegraphics[width=0.19\linewidth]{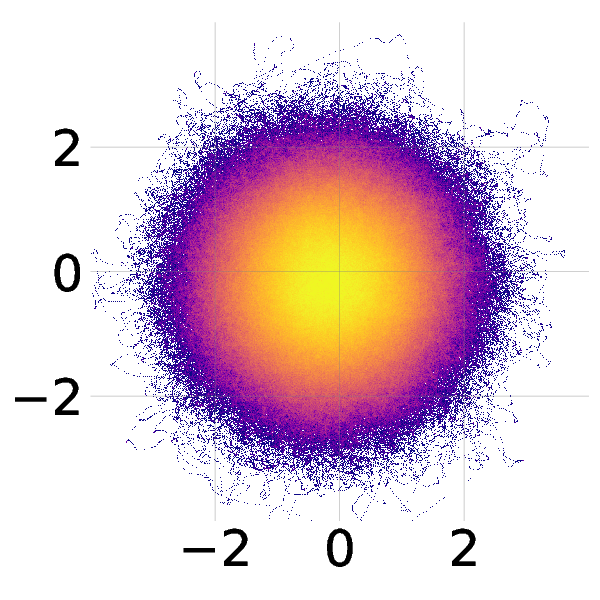} &
        \includegraphics[width=0.19\linewidth]{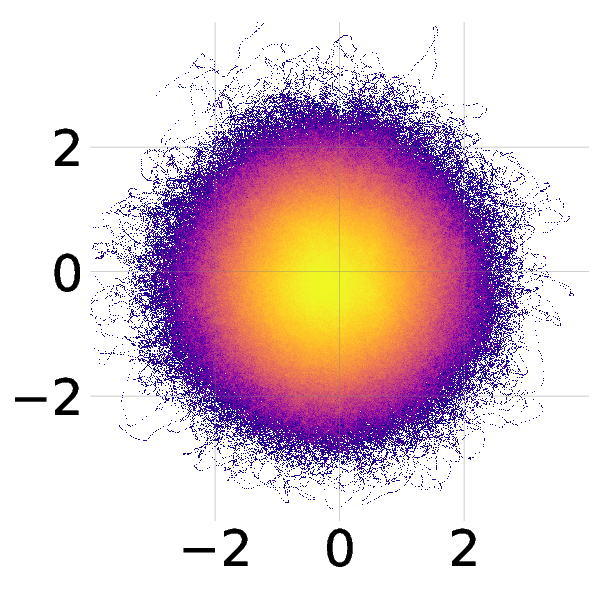} &
        \includegraphics[width=0.19\linewidth]{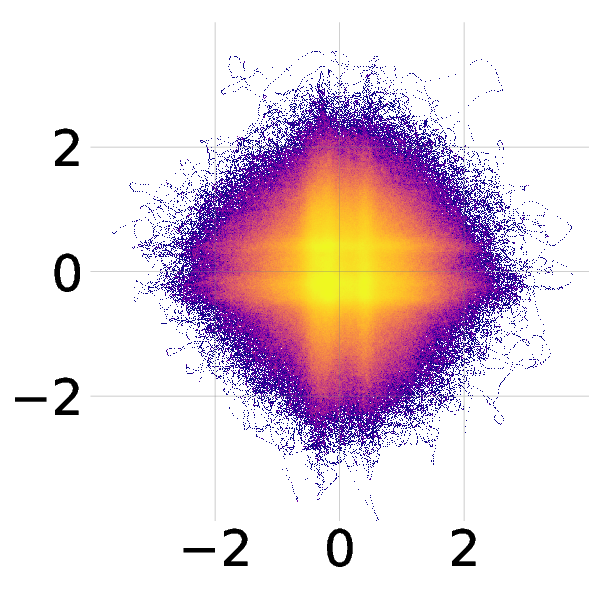} \\

        \rotatebox{90}{\quad\,\, \tiny Bias} &
        \includegraphics[width=0.19\linewidth]{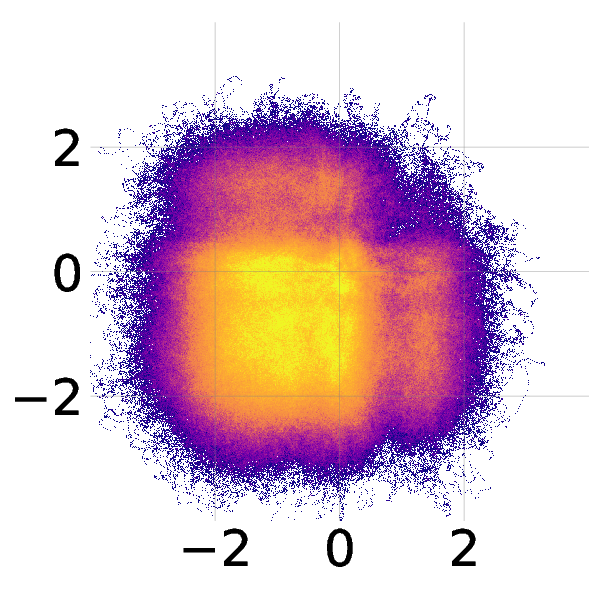} &
        \includegraphics[width=0.19\linewidth]{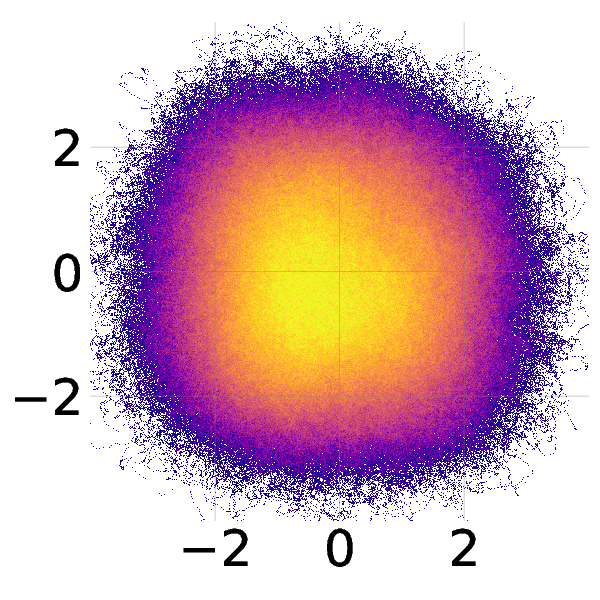} &
        \includegraphics[width=0.19\linewidth]{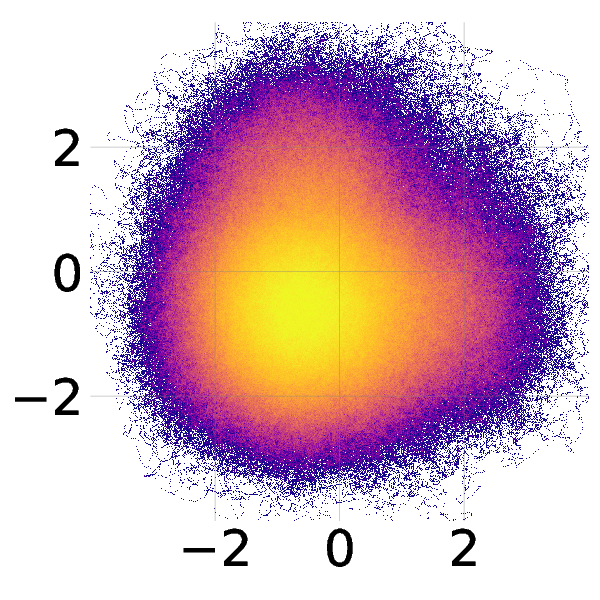} &
        \includegraphics[width=0.19\linewidth]{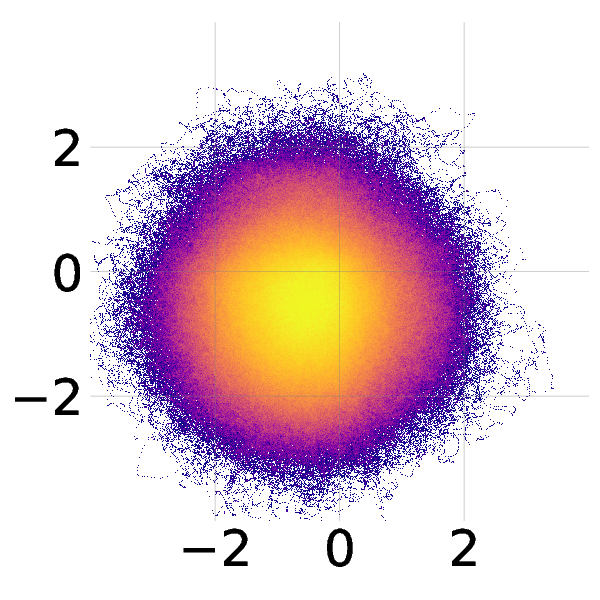} &
        \includegraphics[width=0.19\linewidth]{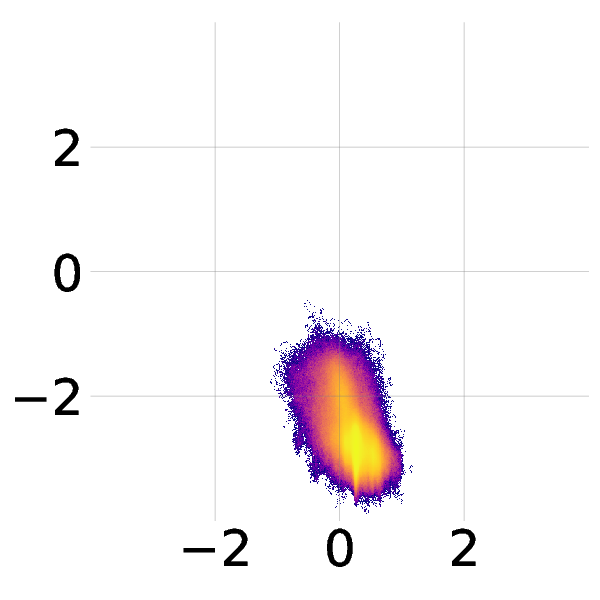} \\
        & \tiny First & \tiny 1st Hidden & \tiny 2nd Hidden & \tiny 3rd Hidden & \tiny Last
    \end{tabular}
    \caption{Bivariate marginal posterior densities of a 4-hidden layer BNN fitted on the \texttt{airfoil} dataset. The grid visualizes the empirical densities of 10M posterior samples obtained from 10k independent chains. The rows and columns display representative densities of randomly chosen layer weights.}
    \label{fig:maingrid}
\end{figure}

\section{OVERPARAMETRIZATION INFLUENCE ON SAMPLING-BASED INFERENCE} \label{sec:exp}

After studying the posterior from a more theoretical perspective, we now provide empirical evidence of how overparametrization affects sampling-based posterior approximations. While our theory focused mainly on fully-connected networks, we also provide results from experiments using other network architectures.  In particular, \cref{app:uci_benchmark} provides a variety of benchmarks to validate empirical findings in recent papers, confirming the good performance of sampling for BNN inference when sufficiently overparametrized, and analyzing their corresponding posterior shapes.

\subsection{Under- vs.\ Overparametrized Models}

We begin by analyzing the functional diversity recoverable by the sampler. As shown in \cref{fig:izmailov}, for smaller architectures the sampler remains confined to distinct modes in parameter space, while larger parameter counts make the marginals more closely aligned with the prior. 
In \cref{sec:deep}, we plot the marginal distributions of two weights per chain as well as aggregated over chains in \cref{fig:univariate_overparam_density_concrete}. The plot implies that marginalizing over differently initialized chains in overparametrized models reveals the suggested marginal posterior pattern, which is not the case for the underparametrized model. A similar pattern can be found in \cref{fig:rohrschach}. This connectedness, alas, comes at the cost of having to sample in higher dimensions. It is therefore not immediately clear whether the smoother posterior landscape helps represent the functional variance of the model more truthfully or whether this is only a consequence of the increased degrees-of-freedom in the model's parameters.

While the answer also depends on the specific sampler employed, the setup in \cref{fig:oneMone:a}, i.e., a \emph{1}-$M$-\emph{1} network, highlights one particular mechanism. In ReLU networks, samplers can become trapped in regions of the posterior landscape where the gradient vanishes \citep{sommer2024connecting}. For a positive input $x>0$ w.l.o.g., and assuming a zero-centered symmetric posterior distribution of the first-layer weights $\mathbf w_1$, the probability that all hidden pre-activations are negative is $2^{-M}$. Hence, this event becomes exponentially unlikely as the hidden width $M$ increases. Combined with smoother marginal distributions in the overparametrized regime, this may improve the effectiveness of some samplers. 

\begin{figure}[t]
    \centering
    \includegraphics[width=0.9\linewidth]{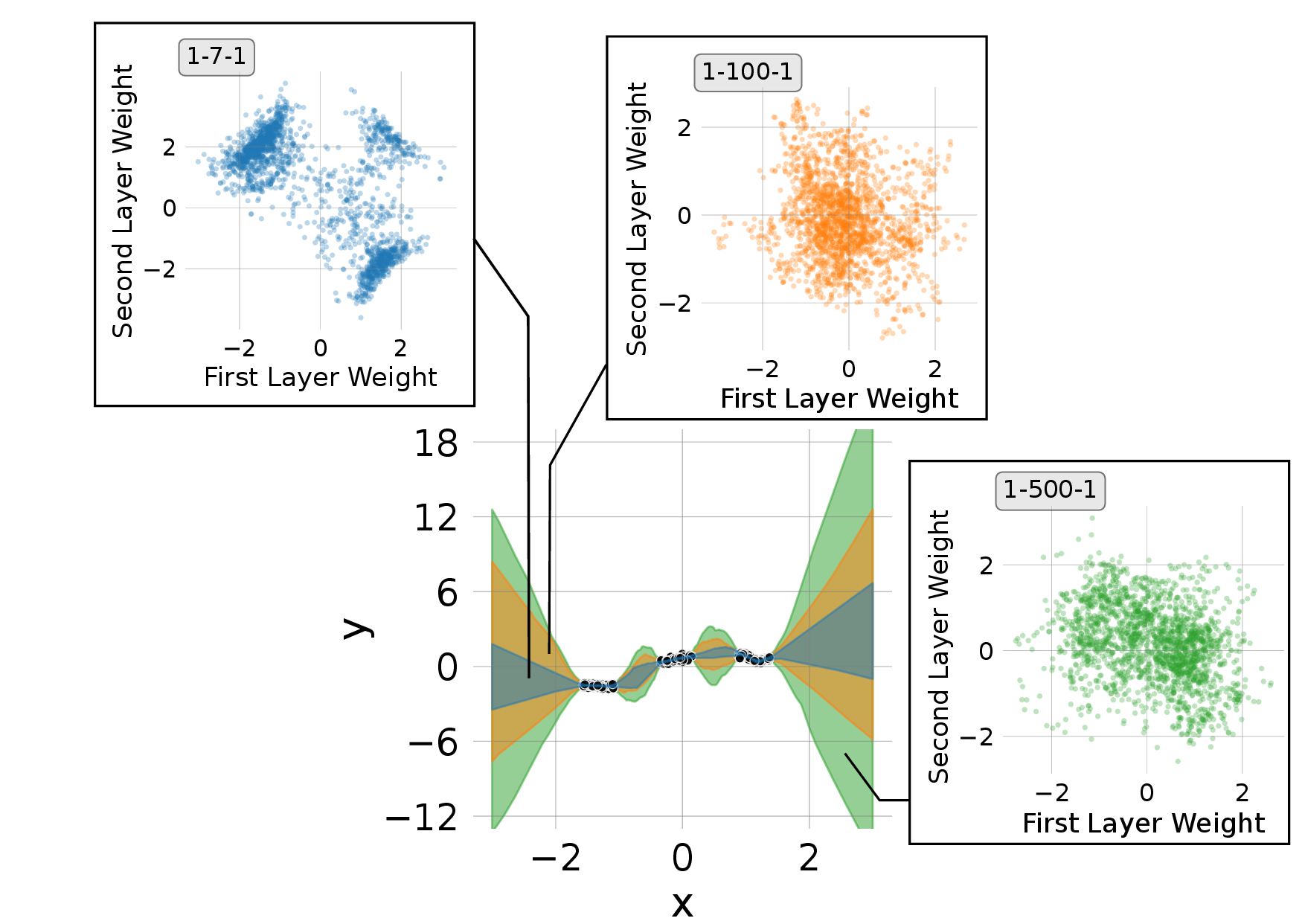}
    \caption{Empirical marginal posterior view for ReLU networks with increasing widths. As the marginal coverage of the posterior increases, the functional variance recovered by the sampler increases as well.}
    \label{fig:izmailov}
\end{figure}

\subsection{Balancedness Experiments}

\cref{fig:maingrid,fig:BIKEfireballgrid,fig:fmnistfireballgrid,fig:ionospherefireballgrid} reveal more nuanced behavior. Observed patterns depend on two factors: the layer type (first, hidden, last) and the parameter type (bias vs.\ kernel). Specifically, weights in input and output layers exhibit distinct, multimodal or highly concentrated margins, whereas intermediate-layer weights yield margins closely matching their isotropic Gaussian priors. This backs our discussion of balancedness in networks with a bias in \cref{sec:deep_bias} as well as \citet{sommer2024connecting}, which demonstrated enhanced sampler exploration in intermediate layers of BNNs.

\begin{figure}[t]
    \centering
    \vspace{-0.5em}
    \setlength{\tabcolsep}{8pt}
    \renewcommand{\arraystretch}{1}
    \begin{tabular}{c}
        \small Covariance \qquad\qquad \qquad Correlation \\
        \hline
        \includegraphics[width=0.88\linewidth]{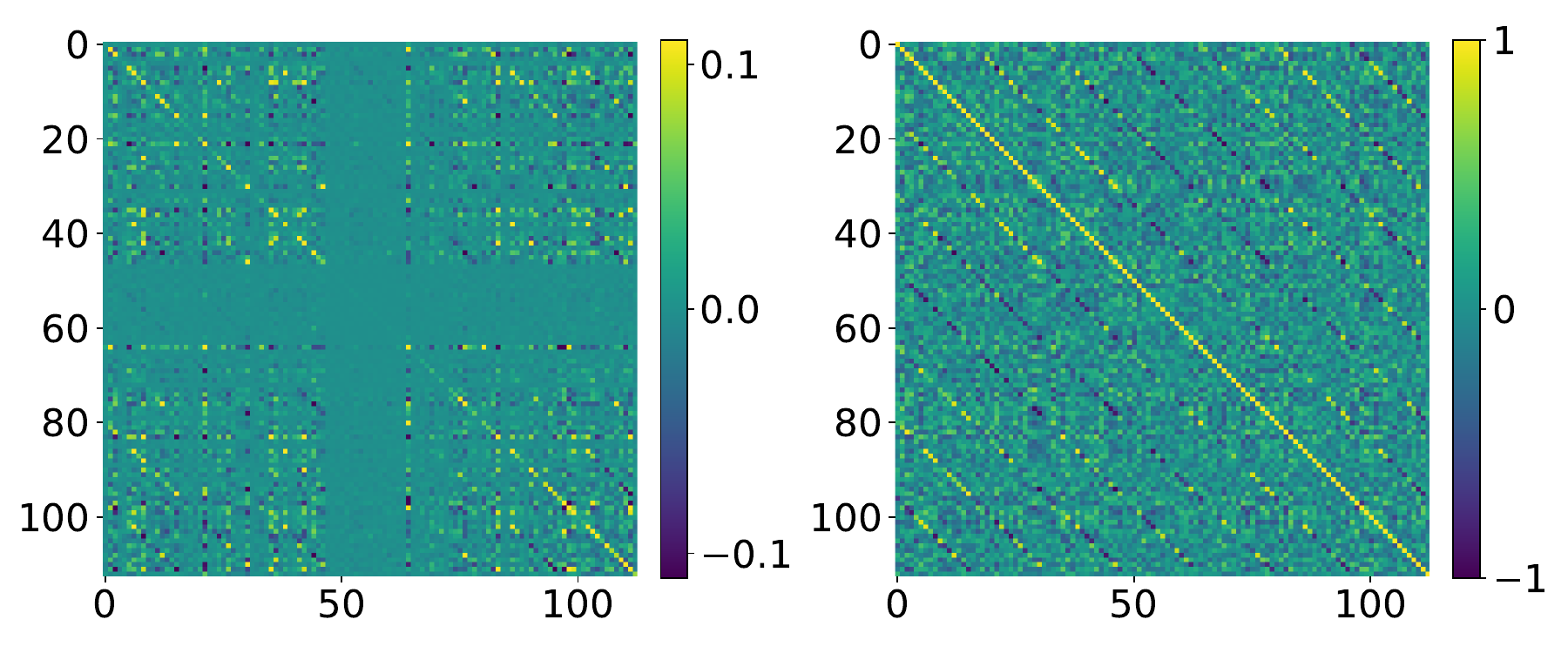} \\ 
        \scriptsize Underparametrized 5-16-1 (113) \\
        \includegraphics[width=0.88\linewidth]{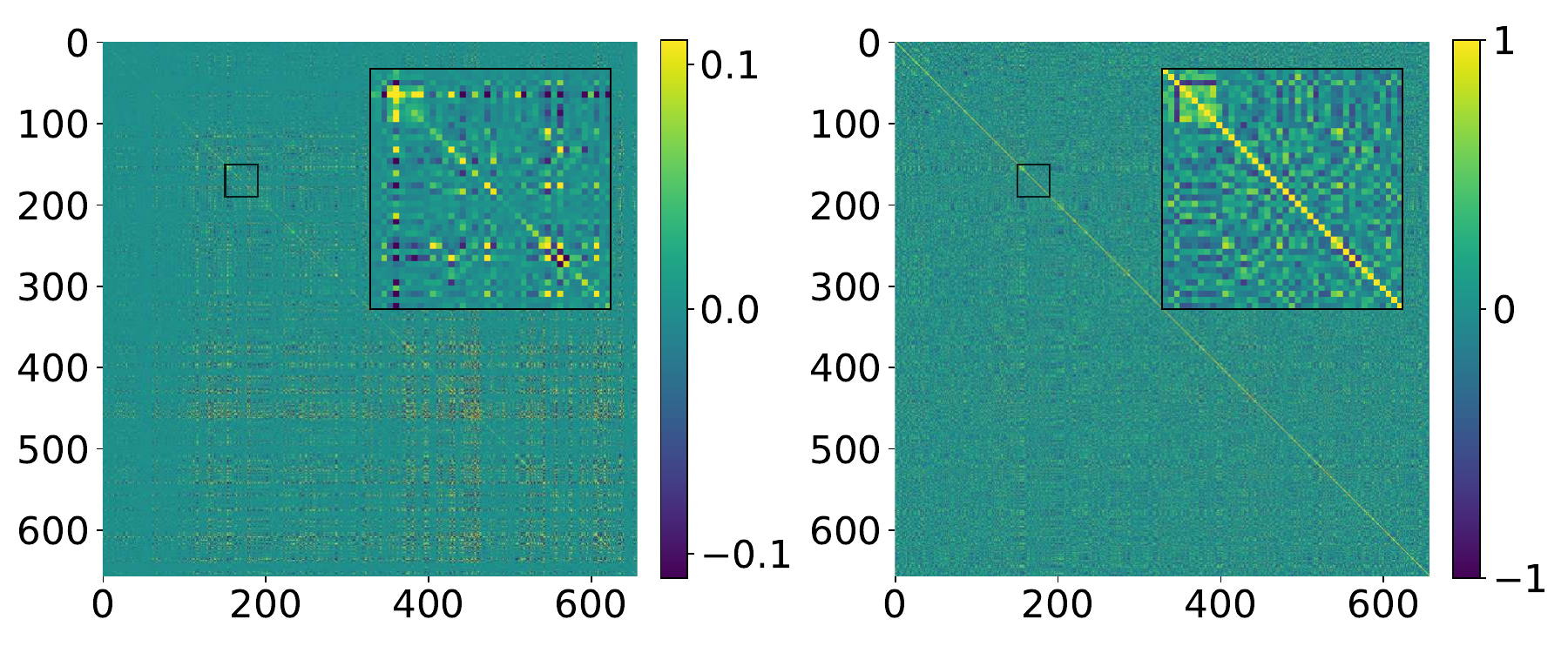} \\
        \scriptsize Overparametrized 5-16-16-16-1 (657) 
    \end{tabular}
    \caption{
    Empirical covariance and correlation of weights $\bw$ across posterior samples on the \texttt{airfoil} dataset. The sampler is able to recover an intricate correlation structure in the parameter space.
    }
    \label{fig:correlation}
    \vspace{-0.1in}
\end{figure}

\subsection{Overparametrization Does Not Imply Data-Independency}

Increasing the parameter count in deep networks can lead to problems in variational approximations of the posterior, potentially even inducing conditional independence of parameters and data \citep[see, e.g.,][]{trippe2017overpruning}. Results from \cref{sec:twolayer} and \cref{fig:oneMone:a} also suggest that the mean of the weights and their covariance decrease as the number of parameters increases, which is in line with this finding. In sampling-based inference, however, a zero-mean weight distribution does not imply a degenerate model, as it does not concentrate on a single solution as variational methods typically do. To investigate whether overparametrization can lead to diminishing covariance in larger (yet finite-width) neural networks, we run sampling-based inference for two networks of different sizes with good performance on the task. \cref{fig:correlation} visualizes the empirical covariance and correlation matrices. By the intricate patterns visible in the plots, we conclude that the sampler is able to recover a complex correlation structure in parameter space, even when navigating a higher-dimensional parameter space. We further do not observe indications of data ignorance as suggested for variational methods \citep{coker2022wide}. This suggests that while (some) marginals appear prior-like and most of them have zero mean (cf.~\cref{fig:maingrid}), there is still substantial correlation structure left, and prior conformity is much more likely a result of marginalization given the many directions with likelihood-flat regions.

\subsection{Flat Likelihood Directions and Prior Conformity}

\begin{figure}[t]
    \centering
    \includegraphics[width=0.8\linewidth]{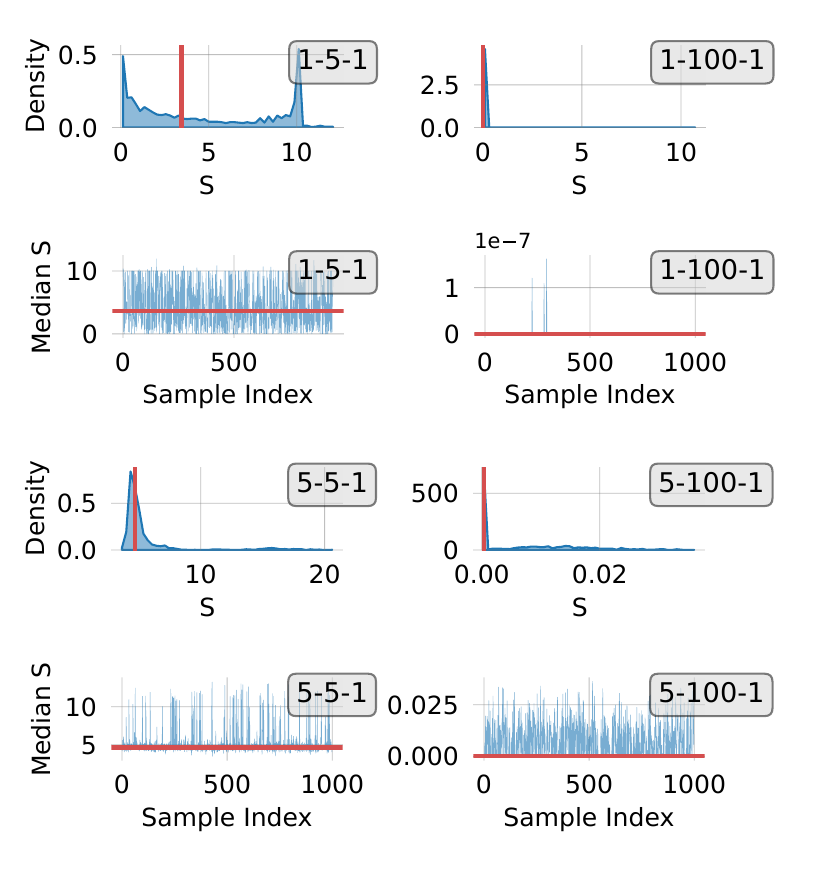}
    \caption{Minimal singular values $S$ calculated for every hidden neuron cluster across all posterior samples, as well as median minimum singular value per neuron cluster per posterior sample $\bw$. The median (first row) or the median of medians (second row) is indicated in red.}
    \label{fig:function_space_sensitivity}
\end{figure}

Following our discussion of \textit{prior conformity} (\cref{sec:deep}), we probe for inter-neuron likelihood-flat directions in single-hidden layer ReLU networks by analyzing samples from a posterior chain, extending work like \citet{ghorbani2019investigation} beyond single parameter points. For each posterior sample $\bw$, we cluster neurons with activation vectors on the training data, $\bm\Xi(\bw)$, that exhibit high cosine similarity. To identify invariant relationships, we project each cluster's activations onto the subspace orthogonal to the all-ones vector. A zero minimum singular value $S$ implies a non-trivial, zero-sum combination of neuron activations that is constant across all data points. This signals a representational redundancy and thus a continuous manifold of non-identifiable parameters producing identical model outputs. We observe in \cref{fig:function_space_sensitivity} that such redundancies, and corresponding likelihood-flat directions, become more prevalent as model size increases, appearing at nearly every posterior sample. This analysis is equivalent to examining $\ker(\bJ)$ at each sample, as depicted in \cref{fig:kerJ} and detailed in \cref{app:proofs_derivations}.

\section{DISCUSSION} \label{sec:discuss}

In this work, we showed that overparametrization reshapes BNN posteriors through the interplay of symmetries and priors. We demonstrated that network symmetries often align with specific weight priors, offering a principled notion of prior choice. Moreover, redundancy and homogeneity induce balancedness across layers, leading to equal-probability manifolds and prior conformity in redundant directions.

\paragraph{Practical Considerations}

The absence of gaps in the marginal posterior of weights and their smoothness when averaging over a large number of chains suggests a smooth connected surface, making a full characterization of the BNN posterior practically possible. To this end, we compare different architectures and BNNs in terms of their cumulative log pointwise predictive density \citep[LPPD, ][]{gelman2014a}, examining the rate at which this metric saturates (details in \cref{app:additional_results}). While a scattered posterior with many different important regions separated by ``loss barriers'' might require a large number of chains or show no signs of saturation in uncertainty metrics, \cref{fig:lppd_convergence} suggests convergence after around 10-20 chains. This result is reassuring as it shows that retrieving an approximate posterior via sampling is indeed practically feasible for BNNs.

\paragraph{Limitations}

A limitation of our work is its theoretical exposition's main focus on fully-connected ReLU networks. Yet, our empirical results indicate that the effects carry over to other architectures, including convolutional and residual networks. Another concern is that overparametrization increases dimensionality, which could, in principle, render sampling inefficient. However, recent results show that samplers like MCLMC \citep{robnik2024fluctuation} retain efficiency in high-dimensional regimes under certain assumptions, supporting the practical feasibility of sampling in large BNNs. Another limitation is the conditioning on a fixed assignment $\varsigma$. While this assumption is not overly restrictive for single-chain approaches---since a single chain is unlikely to traverse between different assignments---a complete posterior characterization requires marginalization over all possible assignments. For ensembles of chains, such marginalization provides a more appropriate theoretical foundation.

\begin{figure}[t]
    \centering
    \begin{subfigure}[b]{0.48\columnwidth}
        \centering
        \includegraphics[width=\textwidth]{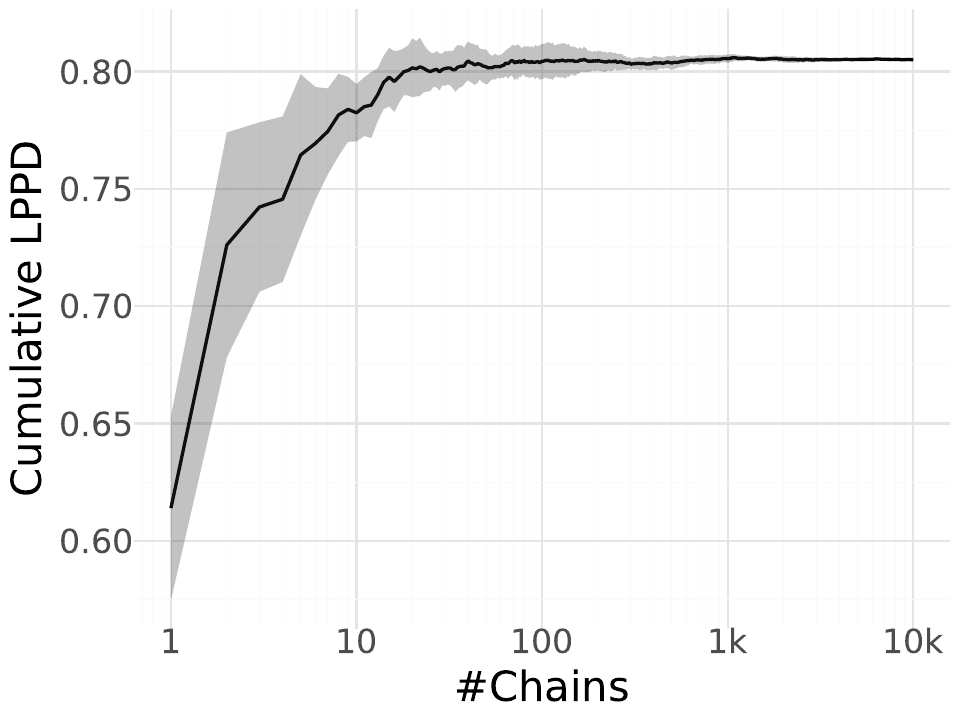}
        \caption{A fully-connected BNN on \texttt{airfoil}.}
        \label{fig:subfig1}
    \end{subfigure}
    \hfill
    \begin{subfigure}[b]{0.48\columnwidth}
        \centering
        \includegraphics[width=\textwidth]{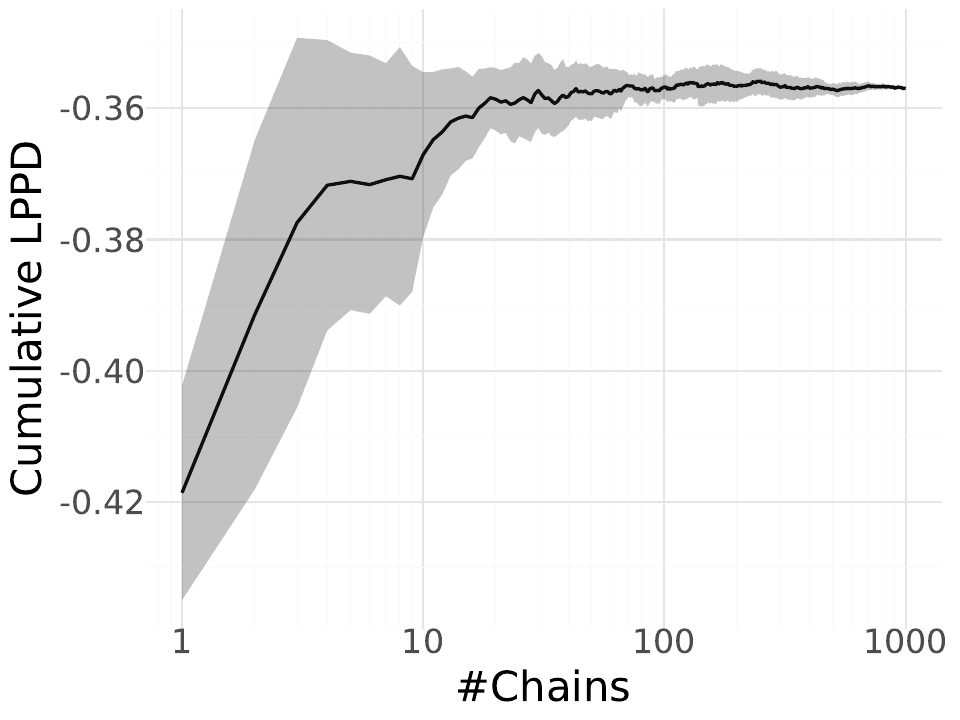}
        \caption{A convolutional BNN on \texttt{Fashion-MNIST}.}
        \label{fig:subfig2}
    \end{subfigure}
    \caption{Cumulative LPPD increase over the number of chains (standard deviation across 5 random chain orderings) for different architectures and datasets.}
    \label{fig:lppd_convergence}
\end{figure}

\bibliography{main}

@inproceedings{ziyin2024removesymmetriescontrolmodel,
      title={{Remove Symmetries to Control Model Expressivity}}, 
      author={Liu Ziyin and Yizhou Xu and Isaac Chuang},
booktitle={The Thirteenth International Conference on Learning Representations},
year={2025}
}

@inproceedings{
ziyin2024parameter,
title={{Parameter Symmetry and Noise Equilibrium of Stochastic Gradient Descent}},
author={Liu Ziyin and Mingze Wang and Hongchao Li and Lei Wu},
booktitle={The Thirty-eighth Annual Conference on Neural Information Processing Systems},
year={2024},
OPTurl={https://openreview.net/forum?id=uhki1rE2NZ}
}

@misc{corti2025microscopic,
      title={Microscopic and collective signatures of feature learning in neural networks}, 
      author={Andrea Corti and Rosalba Pacelli and Pietro Rotondo and Marco Gherardi},
      year={2025},
      eprint={2508.20989},
      archivePrefix={arXiv},
      primaryClass={cond-mat.dis-nn},
      OPTurl={https://arxiv.org/abs/2508.20989}, 
}

@article{liu2020toward,
  title={Toward a theory of optimization for over-parameterized systems of non-linear equations: the lessons of deep learning},
  author={Liu, Chaoyue and Zhu, Libin and Belkin, Mikhail},
  journal={arXiv preprint arXiv:2003.00307},
  volume={7},
  year={2020}
}

@inproceedings{adasghmc,
 author = {Springenberg, Jost Tobias and Klein, Aaron and Falkner, Stefan and Hutter, Frank},
 booktitle = {Advances in Neural Information Processing Systems},
 editor = {D. Lee and M. Sugiyama and U. Luxburg and I. Guyon and R. Garnett},
 pages = {},
 publisher = {Curran Associates, Inc.},
 title = {Bayesian Optimization with Robust Bayesian Neural Networks},
 volume = {29},
 year = {2016}
}

@inproceedings{nguyen2019connected,
  title={On connected sublevel sets in deep learning},
  author={Nguyen, Quynh},
  booktitle={International conference on machine learning},
  pages={4790--4799},
  year={2019},
  organization={PMLR}
}

@inproceedings{sommer2024connecting,
  title={{Connecting the Dots: Is Mode-Connectedness the Key to Feasible Sample-Based Inference in Bayesian Neural Networks?}},
  author={Sommer, Emanuel and Wimmer, Lisa and Papamarkou, Theodore and Bothmann, Ludwig and Bischl, Bernd and R{\"u}gamer, David},
  booktitle={Proceedings of the 41st International Conference on Machine Learning},
  year={2024},
  publisher={PMLR},
}

@inproceedings{wiese2023towards,
  title={{Towards efficient MCMC sampling in Bayesian neural networks by exploiting symmetry}},
  author={Wiese, Jonas Gregor and Wimmer, Lisa and Papamarkou, Theodore and Bischl, Bernd and G{\"u}nnemann, Stephan and R{\"u}gamer, David},
  booktitle={Joint European Conference on Machine Learning and Knowledge Discovery in Databases},
  pages={459--474},
  year={2023},
  organization={Springer}
}

@misc{jax2018github,
  author = {James Bradbury and Roy Frostig and Peter Hawkins and Matthew James Johnson and Chris Leary and Dougal Maclaurin and George Necula and Adam Paszke and Jake Vander{P}las and Skye Wanderman-{M}ilne and Qiao Zhang},
  title = {{JAX}: Composable Transformations of {P}ython+{N}um{P}y programs},
  url = {http://github.com/google/jax},
  version = {0.3.13},
  year = {2018},
}

@inproceedings{daxberger_2021_LaplaceReduxa,
  title = {Laplace {{Redux}} {\textendash} {{Effortless Bayesian Deep Learning}}},
  booktitle = {35th {{Conference}} on {{Neural Information Processing Systems}} ({{NeurIPS}} 2021)},
  author = {Daxberger, Erik and Kristiadi, Agustinus and Immer, Alexander and Eschenhagen, Runa and Bauer, Matthias and Hennig, Philipp},
  year = {2021}
}

@inproceedings{farquhar_2020_LibertyDeptha,
  title = {Liberty or {{Depth}}: {{Deep Bayesian Neural Nets Do Not Need Complex Weight Posterior Approximations}}},
  shorttitle = {Liberty or {{Depth}}},
  booktitle = {Proceedings of the 34th {{Conference}} on {{Neural Information Processing Systems}} ({{NeurIPS}} 2020)},
  author = {Farquhar, Sebastian and Smith, Lewis and Gal, Yarin},
  year = {2020}
}

@inproceedings{freeman_2017_TopologyGeometrya,
  title = {Topology and {{Geometry}} of {{Half-Rectified Network Optimization}}},
  booktitle = {Proceedings of the 5th International Conference on Learning Representations},
  author = {Freeman, C. Daniel and Bruna, Joan},
  year = {2017}
}

@inproceedings{garipov_2018_LossSurfacesb,
  title = {Loss {{Surfaces}}, {{Mode Connectivity}}, and {{Fast Ensembling}} of {{DNNs}}},
  booktitle = {Proceedings of the 32nd {{Conference}} on {{Neural Information Processing Systems}} ({{NeurIPS}} 2018)},
  author = {Garipov, Timur and Izmailov, Pavel and Podoprikhin, Dmitrii and Vetrov, Dmitry P and Wilson, Andrew G},
  year = {2018},
  pages = {10}
}

@incollection{hecht-nielsen_1990_ALGEBRAICSTRUCTUREa,
  title = {{On the Algebraic Structure of Feedforward Network Weight Spaces}},
  booktitle = {Advanced {{Neural Computers}}},
  author = {{Hecht-Nielsen}, Robert},
  year = {1990},
  pages = {129--135},
  publisher = {{Elsevier}}
}

@inproceedings{
loshchilov2018decoupled,
title={{Decoupled Weight Decay Regularization}},
author={Ilya Loshchilov and Frank Hutter},
booktitle={International Conference on Learning Representations},
year={2019},
OPTurl={https://openreview.net/forum?id=Bkg6RiCqY7},
}

@misc{cabezas2024blackjax,
      title={{BlackJAX: Composable Bayesian inference in JAX}},
      author={Alberto Cabezas and Adrien Corenflos and Junpeng Lao and Rémi Louf},
      year={2024},
      eprint={2402.10797},
      archivePrefix={arXiv},
      primaryClass={cs.MS}
}

@inproceedings{izmailov_2020_SubspaceInferencea,
  title = {Subspace {{Inference}} for {{Bayesian Deep Learning}}},
  booktitle = {Proceedings of the Conference on Uncertainty in Artificial Intelligence},
  author = {Izmailov, Pavel and Maddox, Wesley J. and Kirichenko, Polina and Garipov, Timur and Vetrov, Dmitry and Wilson, Andrew Gordon},
  year = {2020},
  pages = {1169--1179}
}

@inproceedings{izmailov_2021_WhatArea,
  title = {What {{Are Bayesian Neural Network Posteriors Really Like}}?},
  booktitle = {Proceedings of the 38th {{International Conference}} on {{Machine Learning}}, {{PMLR}} 139,},
  author = {Izmailov, Pavel and Vikram, Sharad and Hoffman, Matthew D. and Wilson, Andrew Gordon},
  year = {2021}
}

@article{kolb_2023_SmoothingEdgesa,
  title = {{Smoothing the Edges: Smooth Optimization for Sparse Regularization using Hadamard Overparametrization}},
  author = {Kolb, Chris and M{\"u}ller, Christian L. and Bischl, Bernd and R{\"u}gamer, David},
  year = {2026},
  journal = {Machine Learning},
  note = {To appear}
}

@article{
    navon_2023_EquivariantArchitecturesa,
    title={Equivariant Architectures for Learning in Deep Weight Spaces},
    author={
        Navon, Aviv and Shamsian, Aviv and Achituve, Idan and Fetaya, 
        Ethan and Chechik, Gal and Maron, Haggai
    },
    journal={arXiv preprint arXiv:2301.12780},
    year={2023}
}

@article{papamarkou_2022_ChallengesMarkova,
  title = {Challenges in {{Markov Chain Monte Carlo}} for {{Bayesian Neural Networks}}},
  author = {Papamarkou, Theodore and Hinkle, Jacob and Young, M. Todd and Womble, David},
  year = {2022},
  journal = {Statistical Science},
  volume = {37},
  number = {3},
  optIssn = {0883-4237}
}

@inproceedings{pourzanjani_2017_ImprovingIdentifiabilityb,
  title = {Improving the {{Identifiability}} of {{Neural Networks}} for {{Bayesian Inference}}},
  booktitle = {Second Workshop on {{Bayesian Deep Learning}}},
  author = {Pourzanjani, Arya A. and Jiang, Richard M. and Petzold, Linda R.},
  year = {2017}
}

@article{jacot2018neural,
  title={Neural tangent kernel: Convergence and generalization in neural networks},
  author={Jacot, Arthur and Gabriel, Franck and Hongler, Cl{\'e}ment},
  journal={Advances in neural information processing systems},
  volume={31},
  year={2018}
}

@inproceedings{simsek2021geometry,
  title = {Geometry of the {{Loss Landscape}} in {{Overparameterized Neural Networks}}: {{Symmetries}} and {{Invariances}}},
  booktitle = {Proceedings of the 38 Th {{International Conference}} on {{Machine Learning}}},
  author = {Simsek, Berfin and Ged, Fran{\c c}ois and Jacot, Arthur and Spadaro, Francesco and Hongler, Cl{\'e}ment and Gerstner, Wulfram and Brea, Johanni},
  year = {2021}
}

@inproceedings{ghorbani2019investigation,
  title={An investigation into neural net optimization via hessian eigenvalue density},
  author={Ghorbani, Behrooz and Krishnan, Shankar and Xiao, Ying},
  booktitle={International Conference on Machine Learning},
  pages={2232--2241},
  year={2019},
  organization={PMLR}
}

@article{kunin2024get,
  title={Get rich quick: exact solutions reveal how unbalanced initializations promote rapid feature learning},
  author={Kunin, Daniel and Ravent{\'o}s, Allan and Domin{\'e}, Cl{\'e}mentine and Chen, Feng and Klindt, David and Saxe, Andrew and Ganguli, Surya},
  journal={Advances in Neural Information Processing Systems},
  volume={37},
  pages={81157--81203},
  year={2024}
}

@inproceedings{azulay2021implicit,
  title={On the implicit bias of initialization shape: Beyond infinitesimal mirror descent},
  author={Azulay, Shahar and Moroshko, Edward and Nacson, Mor Shpigel and Woodworth, Blake E and Srebro, Nathan and Globerson, Amir and Soudry, Daniel},
  booktitle={International Conference on Machine Learning},
  pages={468--477},
  year={2021},
  organization={PMLR}
}

@article{liu2016stein,
  title={Stein variational gradient descent: A general purpose bayesian inference algorithm},
  author={Liu, Qiang and Wang, Dilin},
  journal={Advances in neural information processing systems},
  volume={29},
  year={2016}
}

@article{martens2020new,
  title={New insights and perspectives on the natural gradient method},
  author={Martens, James},
  journal={Journal of Machine Learning Research},
  volume={21},
  number={146},
  pages={1--76},
  year={2020}
}

@inproceedings{coker2022wide,
  title={Wide mean-field bayesian neural networks ignore the data},
  author={Coker, Beau and Bruinsma, Wessel P and Burt, David R and Pan, Weiwei and Doshi-Velez, Finale},
  booktitle={International Conference on Artificial Intelligence and Statistics},
  pages={5276--5333},
  year={2022},
  organization={PMLR}
}

@article{gelman2014a,
  author = {Gelman, A. and Hwang, J. and Vehtari, A.},
  title = {{Understanding Predictive Information Criteria for Bayesian Models}},
  volume = {24},
  pages = {997–1016},
  year = {2014},
  language = {en},
  journal = {Statistics and Computing},
  number = {6}
}

@InProceedings{wenzel_2020_HowGooda,
  title = 	 {How Good is the {B}ayes Posterior in Deep Neural Networks Really?},
  author =       {Wenzel, Florian and Roth, Kevin and Veeling, Bastiaan and Swiatkowski, Jakub and Tran, Linh and Mandt, Stephan and Snoek, Jasper and Salimans, Tim and Jenatton, Rodolphe and Nowozin, Sebastian},
  booktitle = 	 {Proceedings of the 37th International Conference on Machine Learning},
  year = 	 {2020},
  publisher =    {PMLR}
}

@InProceedings{dold2024,
title={{Semi-Structured Subspace Inference}},
author={Daniel Dold and David R{\"u}gamer and Beate Sick and Oliver D{\"u}rr},
  year={2024},
  booktitle = {Proceedings of the 27th International Conference on Artificial Intelligence and Statistics},
  series = {Proceedings of Machine Learning Research},
  publisher = {PMLR}
}

@article{pavliotis2014stochastic,
  title={Stochastic processes and applications},
  author={Pavliotis, Grigorios A},
  journal={Texts in applied mathematics},
  volume={60},
  year={2014},
  publisher={Springer}
}

@InProceedings{dold2025,
title={{Paths and Ambient Spaces in Neural Loss Landscapes}},
author={Daniel Dold and Julius Kobialka and Nicolai Palm and Emanuel Sommer and David R{\"u}gamer and Oliver D{\"u}rr},
  year={2025},
  booktitle = {Proceedings of the 28th International Conference on Artificial Intelligence and Statistics},
  series = {Proceedings of Machine Learning Research},
  publisher = {PMLR}
}

@misc{Dua.2019 ,
author = "Dua, Dheeru and Graff, Casey",
year = "2017",
title = "{{UCI Machine Learning Repository}}",
url = "http://archive.ics.uci.edu/ml",
institution = "University of California, Irvine, School of Information and Computer Sciences" }

@article{Yeh.1998,
  title={{Modeling of Strength of High-Performance Concrete Using Artificial Neural Networks}},
  author={Yeh, I-C},
  journal={Cement and Concrete research},
  volume={28},
  number={12},
  year={1998},
  publisher={Elsevier}
}

@article{Tsanas.2012,
  title={{Accurate Quantitative Estimation of Energy Performance of Residential Buildings Using Statistical Machine Learning Tools}},
  author={Tsanas, Athanasios and Xifara, Angeliki},
  journal={Energy and Buildings},
  volume={49},
  year={2012},
  publisher={Elsevier}
}

@misc{misc_bike_sharing_dataset_275,
  author       = {Fanaee-T,Hadi},
  title        = {{Bike Sharing Dataset}},
  year         = {2013},
  howpublished = {UCI Machine Learning Repository}
}

@inproceedings{zhang_2020_CyclicalStochastic,
  title = {Cyclical {{Stochastic Gradient MCMC}} for {{Bayesian Deep Learning}}},
  booktitle = {Proceedings of the {{Eighth International Conference}} on {{Learning Representations}}},
  author = {Zhang, Ruqi and Li, Chunyuan and Zhang, Jianyi and Chen, Changyou and Wilson, Andrew Gordon},
  year = {2020},
}

@inproceedings{draxler_2018_EssentiallyNo,
  title = {Essentially {{No Barriers}} in {{Neural Network Energy Landscape}}},
  booktitle = {Proceedings of the 35th {{International Conference}} on {{Machine Learning}}},
  author = {Draxler, Felix and Veschgini, Kambis and Salmhofer, Manfred and Hamprecht, Fred A.},
  year = {2018}
}

@inproceedings{bui2020functional,
  title={{Functional vs. Parametric Equivalence of ReLU Networks}},
  author={Phuong, Mary and Lampert, Christoph},
  booktitle={8th International Conference on Learning Representations},
  year={2020}
}

@inproceedings{
laurent2024a,
title={{A Symmetry-Aware Exploration of Bayesian Neural Network Posteriors}},
author={Olivier Laurent and Emanuel Aldea and Gianni Franchi},
booktitle={{The Twelfth International Conference on Learning Representations}},
year={2024},
OPTurl={https://openreview.net/forum?id=FOSBQuXgAq}
}

@article{bona2023parameter,
  title={{Parameter Identifiability of a Deep Feedforward ReLU Neural Network}},
  author={Bona-Pellissier, Joachim and Bachoc, Fran{\c{c}}ois and Malgouyres, Fran{\c{c}}ois},
  journal={Machine Learning},
  volume={112},
  number={11},
  pages={4431--4493},
  year={2023},
  publisher={Springer}
}

@article{fortuin2022priors,
  title={{Priors in Bayesian Deep Learning: A Review}},
  author={Fortuin, Vincent},
  journal={International Statistical Review},
  volume={90},
  number={3},
  pages={563--591},
  year={2022},
  publisher={Wiley Online Library}
}

@inproceedings{petzka2020notes,
  title={{Notes on the Symmetries of 2-Layer ReLU-Networks}},
  author={Petzka, Henning and Trimmel, Martin and Sminchisescu, Cristian},
  booktitle={Proceedings of the northern lights deep learning workshop},
  volume={1},
  pages={6--6},
  year={2020}
}

@inproceedings{gelberg2024variational,
  title={{Variational Inference Failures Under Model Symmetries: Permutation Invariant Posteriors for Bayesian Neural Networks}},
  author={Gelberg, Yoav and van der Ouderaa, Tycho FA and van der Wilk, Mark and Gal, Yarin},
  booktitle={ICML 2024 Workshop on Geometry-grounded Representation Learning and Generative Modeling},
year={2024}
}

@article{kurle2021symmetries,
  title={{On Symmetries in Variational Bayesian Neural Nets}},
  author={Kurle, Richard and Januschowski, Tim and Gasthaus, Jan and Wang, Yuyang Bernie},
  year={2021}
}

@inproceedings{maron2018invariant,
    author = {Maron, Haggai and Ben-Hamu, Heli and Shamir, Nadav and Lipman, Yaron},
    title = {{Invariant and Equivariant Graph Networks}},
    booktitle = {International Conference on Learning Representations},
    year = {2019},
}

@incollection{sen2024bayesian,
  title={{Bayesian Neural Networks and Dimensionality Reduction}},
  author={Sen, Deborshee and Papamarkou, Theodore and Dunson, David},
  booktitle={Handbook of Bayesian, Fiducial, and Frequentist Inference},
  year={2024},
  publisher={Chapman and Hall/CRC}
}

@phdthesis{nalisnick_priors_2018,
  title = {On {{Priors}} for {{Bayesian Neural Networks}}},
  author = {Nalisnick, Eric Thomas},
  year = {2018},
  school = {{University of California, Irvine}}
}

@inproceedings{rolnick2020reverse,
  title={{Reverse-Engineering Deep ReLU Networks}},
  author={Rolnick, David and Kording, Konrad},
  booktitle={International conference on machine learning},
  year={2020},
  organization={PMLR}
}

@InProceedings{pmlr-v235-papamarkou24b,
  title = 	 {{Position: Bayesian Deep Learning is Needed in the Age of Large-Scale AI}},
  author =       {Papamarkou, Theodore and Skoularidou, Maria and Palla, Konstantina and Aitchison, Laurence and Arbel, Julyan and Dunson, David and Filippone, Maurizio and Fortuin, Vincent and Hennig, Philipp and Hern\'{a}ndez-Lobato, Jos\'{e} Miguel and Hubin, Aliaksandr and Immer, Alexander and Karaletsos, Theofanis and Khan, Mohammad Emtiyaz and Kristiadi, Agustinus and Li, Yingzhen and Mandt, Stephan and Nemeth, Christopher and Osborne, Michael A and Rudner, Tim G. J. and R\"{u}gamer, David and Teh, Yee Whye and Welling, Max and Wilson, Andrew Gordon and Zhang, Ruqi},
  booktitle = 	 {Proceedings of the 41st International Conference on Machine Learning},
  year = 	 {2024},
  publisher =    {PMLR}
}

@inproceedings{mishkin2022fast,
  title={{Fast Convex Optimization for Two-Layer ReLU Networks: Equivalent Model Classes and Cone Decompositions}},
  author={Mishkin, Aaron and Sahiner, Arda and Pilanci, Mert},
  booktitle={International Conference on Machine Learning},
  pages={15770--15816},
  year={2022},
  organization={PMLR}
}

@inproceedings{
chrisICLR,
title={{Deep Weight Factorization: Sparse Learning Through the Lens of Artificial Symmetries}},
author={Kolb, Chris and Weber, Tobias and Bischl, Bernd and R{\"u}gamer, David},
booktitle={{The Thirteenth International Conference on Learning Representations}},
year={2025},
OPTurl={https://openreview.net/forum?id=vNdOHr7mn5}
}

@inproceedings{
kim2025exploring,
title={Exploring The Loss Landscape Of Regularized Neural Networks Via Convex Duality},
author={Sungyoon Kim and Aaron Mishkin and Mert Pilanci},
booktitle={The Thirteenth International Conference on Learning Representations},
year={2025},
OPTurl={https://openreview.net/forum?id=4xWQS2z77v}
}

@inproceedings{trippe2017overpruning,
  title={Overpruning in Variational Bayesian Neural Networks},
  author={Trippe, Brian and Turner, Richard},
  journal={Advances in Approximate Bayesian Inference — NIPS 2017 Workshop},
  year={2017}
}

@inproceedings{
MILE,
title={{Microcanonical Langevin Ensembles: Advancing the Sampling of Bayesian Neural Networks}},
author={Sommer, Emanuel and Robnik, Jakob and Nozadze, Giorgi and Seljak, Uros and R\"ugamer, David},
booktitle={The Thirteenth International Conference on Learning Representations},
year={2025}
}

@inproceedings{chen2014stochastic,
  title={{Stochastic Gradient Hamiltonian Monte Carlo}},
  author={Chen, Tianqi and Fox, Emily and Guestrin, Carlos},
  booktitle={International Conference on Machine Learning},
  pages={1683--1691},
  year={2014},
  organization={PMLR}
}

@inproceedings{
    paulin2025sampling,
    title={Sampling from Bayesian Neural Network Posteriors with Symmetric Minibatch Splitting Langevin Dynamics},
    author={Daniel Paulin and Peter A. Whalley and Neil K. Chada and Benedict J. Leimkuhler},
    booktitle={The 28th International Conference on Artificial Intelligence and Statistics},
    year={2025}
}

@inproceedings{
    roy2024reparameterization,
    title={Reparameterization invariance in approximate Bayesian inference},
    author={Hrittik Roy and Marco Miani and Carl Henrik Ek and Philipp Hennig and Marvin Pf{\"o}rtner and Lukas Tatzel and S{\o}ren Hauberg},
    booktitle={The Thirty-eighth Annual Conference on Neural Information Processing Systems},
    year={2024}
}

@misc{sigillito1989ionosphere,
  author = {Sigillito, Vincent and Wing, Scott and Hutton, Lisa and Baker, K.},
  title = {Ionosphere},
  year = {1989},
  howpublished = {UCI Machine Learning Repository},
  url = {https://doi.org/10.24432/C5W01B}
}

@online{xiao2017/online,
  author       = {Han Xiao and Kashif Rasul and Roland Vollgraf},
  title        = {{Fashion-MNIST: a Novel Image Dataset for Benchmarking Machine Learning Algorithms}},
  date         = {2017-08-28},
  year         = {2017},
  eprintclass  = {cs.LG},
  eprinttype   = {arXiv},
  eprint       = {cs.LG/1708.07747},
}

@article{krizhevsky2009learning,
  title={Learning multiple layers of features from tiny images},
  author={Krizhevsky, Alex and Hinton, Geoffrey and others},
  year={2009},
  publisher={Toronto, ON, Canada}
}

@inproceedings{
    duffield2025scalable,
    title={Scalable Bayesian Learning with posteriors},
    author={Samuel Duffield and Kaelan Donatella and Johnathan Chiu and Phoebe Klett and Daniel Simpson},
    booktitle={The Thirteenth International Conference on Learning Representations},
    year={2025}
}

@inproceedings{cui2023bayes,
  title={Bayes-optimal learning of deep random networks of extensive-width},
  author={Cui, Hugo and Krzakala, Florent and Zdeborov{\'a}, Lenka},
  booktitle={International Conference on Machine Learning},
  pages={6468--6521},
  year={2023},
  organization={PMLR}
}

@inproceedings{ziyin2024symmetry,
  title={Symmetry Induces Structure and Constraint of Learning},
  author={Ziyin, Liu},
  booktitle={International Conference on Machine Learning},
  pages={62847--62866},
  year={2024},
  organization={PMLR}
}

@article{papamarkou2023,
	author  = {T. Papamarkou},
	title   = {Approximate blocked {G}ibbs sampling for {B}ayesian neural networks},
	journal = {Statistics and Computing},
	year    = {2023},
	volume  = {33},
	issue   = {5},
	pages   = {}
}

@inproceedings{vladimirova2019understanding,
  title={Understanding priors in Bayesian neural networks at the unit level},
  author={Vladimirova, Mariia and Verbeek, Jakob and Mesejo, Pablo and Arbel, Julyan},
  booktitle={International Conference on Machine Learning},
  pages={6458--6467},
  year={2019},
  organization={PMLR}
}

@inproceedings{blundell2015weight,
  title={Weight uncertainty in neural network},
  author={Blundell, Charles and Cornebise, Julien and Kavukcuoglu, Koray and Wierstra, Daan},
  booktitle={International conference on machine learning},
  pages={1613--1622},
  year={2015},
  organization={PMLR}
}

@article{ziyin2025neural,
  title={Neural Thermodynamics I: Entropic Forces in Deep and Universal Representation Learning},
  author={Ziyin, Liu and Xu, Yizhou and Chuang, Isaac},
  journal={arXiv preprint arXiv:2505.12387},
  year={2025}
}

@article{du2018algorithmic,
  title={Algorithmic regularization in learning deep homogeneous models: Layers are automatically balanced},
  author={Du, Simon S and Hu, Wei and Lee, Jason D},
  journal={Advances in neural information processing systems},
  volume={31},
  year={2018}
}

@inproceedings{robnik2024fluctuation,
  title={Fluctuation without dissipation: Microcanonical Langevin Monte Carlo},
  author={Robnik, Jakob and Seljak, Uros},
  booktitle={Symposium on Advances in Approximate Bayesian Inference},
  pages={111--126},
  year={2024},
  organization={PMLR}
}

@article{parhi2023deep,
  title={Deep learning meets sparse regularization: A signal processing perspective},
  author={Parhi, Rahul and Nowak, Robert D},
  journal={IEEE Signal Processing Magazine},
  volume={40},
  number={6},
  pages={63--74},
  year={2023},
  publisher={IEEE}
}

@misc{SMILE,
      title={Can Microcanonical Langevin Dynamics Leverage Mini-Batch Gradient Noise?}, 
      author={Emanuel Sommer and Kangning Diao and Jakob Robnik and Uros Seljak and David R\"ugamer},
      year={2026},
      eprint={2602.06500},
      archivePrefix={arXiv},
      primaryClass={cs.LG},
      OPTurl={https://arxiv.org/abs/2602.06500}, 
}

@inproceedings{
kolb2025differentiable,
title={{Differentiable Sparsity via D-Gating: Simple and Versatile Structured Penalization}},
author={Chris Kolb and Laetitia Frost and Bernd Bischl and David R{\"u}gamer},
booktitle={The Thirty-ninth Annual Conference on Neural Information Processing Systems},
year={2025},
OPTurl={https://openreview.net/forum?id=8OGTkEJrmb}
}

\section*{Checklist}

\begin{enumerate}

  \item For all models and algorithms presented, check if you include:
  \begin{enumerate}
    \item A clear description of the mathematical setting, assumptions, algorithm, and/or model. Yes, provided in \cref{sec:background,sec:global,app:proofs_derivations,app:additional_results}.
    \item An analysis of the properties and complexity (time, space, sample size) of any algorithm. Not Applicable.
    \item (Optional) Anonymized source code, with specification of all dependencies, including external libraries. Yes, link provided in \cref{app:exp_setup}.
  \end{enumerate}

  \item For any theoretical claim, check if you include:
  \begin{enumerate}
    \item Statements of the full set of assumptions of all theoretical results. Yes, in \cref{app:defs,app:proofs_derivations}.
    \item Complete proofs of all theoretical results. Yes, in \cref{app:proofs_derivations}.
    \item Clear explanations of any assumptions. Yes, in \cref{sec:background,sec:global,app:defs,app:proofs_derivations}.     
  \end{enumerate}

  \item For all figures and tables that present empirical results, check if you include:
  \begin{enumerate}
    \item The code, data, and instructions needed to reproduce the main experimental results (either in the supplemental material or as a URL). Yes, in \cref{app:additional_results}.
    \item All the training details (e.g., data splits, hyperparameters, how they were chosen). Yes, in \cref{app:additional_results}.
    \item A clear definition of the specific measure or statistics and error bars (e.g., with respect to the random seed after running experiments multiple times). Yes, in \cref{app:additional_results} and all respective experiment descriptions.
    \item A description of the computing infrastructure used. (e.g., type of GPUs, internal cluster, or cloud provider). Yes, in \cref{app:additional_results}.
  \end{enumerate}

  \item If you are using existing assets (e.g., code, data, models) or curating/releasing new assets, check if you include:
  \begin{enumerate}
    \item Citations of the creator if your work uses existing assets.  Yes, we use code and reference it in \cref{app:additional_results}.
    \item The license information of the assets, if applicable. Not Applicable. All code used is open-source under common licenses that allow usage and extension.
    \item New assets either in the supplemental material or as a URL, if applicable. Not Applicable.
    \item Information about consent from data providers/curators. Not Applicable.
    \item Discussion of sensible content if applicable, e.g., personally identifiable information or offensive content. Not Applicable.
  \end{enumerate}

  \item If you used crowdsourcing or conducted research with human subjects, check if you include:
  \begin{enumerate}
    \item The full text of instructions given to participants and screenshots. Not Applicable.
    \item Descriptions of potential participant risks, with links to Institutional Review Board (IRB) approvals if applicable. Not Applicable.
    \item The estimated hourly wage paid to participants and the total amount spent on participant compensation. Not Applicable.
  \end{enumerate}

\end{enumerate}

\clearpage
\appendix
\thispagestyle{empty}

\onecolumn
\aistatstitle{\ourtitle: \\
Supplementary Materials}

\section{ADDITIONAL RELATED WORK} \label{app:rellit}

Previous works related to sampling-based inference frame symmetries as equioutput parameter states, where different parameter sets $\bw$ in the weight space $\mathcal{W}$ of the neural network lead to the same functional mapping \citep{hecht-nielsen_1990_ALGEBRAICSTRUCTUREa, wiese2023towards}, i.e., $\exists \bw,\tilde\bw\in\mathcal{W}, \bw\neq\tilde\bw: f_{\bw}(\bx) = f_{\tilde\bw}(\bx) \,\forall\bx\in\mathcal{X}$, that is, the same unnormalized log-posterior value if $\bw, \tilde{\bw}$ have the same prior probability which is shown for the special case of permutation symmetries in \citet{wiese2023towards}. Showing functional equivalence of networks when parameters admit an equivalence relationship, i.e., $\bw \sim \tilde\bw \Rightarrow f_{\bw}(\bx) = f_{\tilde\bw}(\bx)$ is more straightforward  \citep{bona2023parameter,petzka2020notes,bui2020functional,pourzanjani_2017_ImprovingIdentifiabilityb}, while deriving parameter equivalence from equivalent outputs, i.e., $f_{\bw}(\bx) = f_{\tilde\bw}(\bx) \Rightarrow \bw \sim \tilde\bw$ requires stronger and often impractical assumptions \citep{bona2023parameter,bui2020functional,rolnick2020reverse} which is commonly summarized as the non-identifiability of parameters in neural networks. \citet{ziyin2024symmetry} abstracts symmetries into a common mirror reflection symmetry that structures the loss landscape \citep{ziyin2024parameter}. Each symmetry carves out a low-capacity subspace defined by a linear constraint (e.g., $O^\top\bw=0$) that acts as an ``absorbing state" for optimization dynamics. As detailed in \citet{ziyin2024symmetry, ziyin2024parameter, ziyin2025neural,ziyin2024removesymmetriescontrolmodel}, symmetries in combination with sufficient $L_2$ regularization strength $\tau > \tau_0$ for the penalty $\tau ||\bw||_2^2$ affect SGD optimization by partitioning the parameter space into symmetry-aligned as well as symmetry-orthogonal subspaces. In multiplicative overparametrization, $L_2$ regularization then forces weights to be balanced across groups of shared neurons \citep{kolb2025differentiable,chrisICLR,kolb_2023_SmoothingEdgesa}.

\paragraph{Overparametrization Induces Connectedness}

The optimization landscape of neural networks is profoundly affected by overparametrization. In under-parameterized models, the loss surface can be fraught with poor local minima, but as the number of parameters increases, this landscape often simplifies, facilitating optimization. Seminal work has shown that for sufficiently wide networks without regularization, the sublevel sets of the loss function become connected, effectively eliminating ``bad local valleys" and ensuring that all global minima reside within a single, large basin \citep{freeman_2017_TopologyGeometrya, simsek2021geometry, nguyen2019connected, kim2025exploring}. This connectivity is largely driven by the inherent symmetries of neural networks. For instance, the permutation symmetry, which creates numerous equivalent discrete minima in a minimally-sized network, can generate a single connected manifold of global minima with the addition of just one extra neuron per layer \citep{simsek2021geometry}. This transition to a more benign landscape, which may not be locally convex but often satisfies favorable conditions like the Polyak-Lojasiewicz condition \citep{liu2020toward}, is a key reason why gradient-based optimization succeeds in large-scale models. Adding regularization fundamentally alters the loss landscape compared to the unregularized case. For instance, $L_2$ regularization (i.e., weight decay) breaks the positive homogeneity of ReLU units, which can reduce the number of optimal solutions from infinite to finite and create a more structured set of optimizers \citep{kim2025exploring}. For models trained with mixed $L_1$ and $L_2$ regularization, asymptotic connectivity of sublevel sets has also been shown with the loss barrier to connect two minima shrinking as overparametrization increases \citep{freeman_2017_TopologyGeometrya}.

Next to the trivial case of the infinite-width limit of overparametrization \citep[i.e., lazy regime,][]{jacot2018neural}, the ``proportional regime'' where the ratio of the number of observations and the number of hidden neurons remains fixed as both increase, $\frac{n}{M} = \alpha$ as $n, M \rightarrow \infty$, has been considered in the literature \citep{cui2023bayes, corti2025microscopic}.  As opposed to the lazy regime, here, training elicits complex correlations between the parameters, fundamentally changing the model's internal structure to adapt to the data while resembling a Gaussian process in the function space \citep{corti2025microscopic}. This also relates to \citet{trippe2017overpruning}, analyzing a potential pathology where more expressive variational approximations of the BNN posterior can paradoxically lead to worse performance as the optimization process induces conditional independence between parameters and data.

\paragraph{BNN Priors} Prior choice is a widely debated issue in BDL, with opinions ranging from the claim that priors in weight space are meaningless due to high dimensionality and identifiability concerns, to the assertion that meaningful priors do exist, though not in the form of the commonly used isotropic Gaussian priors with constant variance \citep{fortuin2022priors, vladimirova2019understanding}.

\section{OMITTED DEFINITIONS AND ASSUMPTIONS} \label{app:defs}

We start by stating our assumptions and defining overparametrization as a network with excessive neurons that overparametrizes a minimum norm interpolant. This is the basis for our results in \cref{sec:twolayer}.

\begin{assumption}[Overparametrization] \label{ass:interpol}
    Let $f(\bx)=\sum_{m=1}^M w_{2,m} \phi(\bw_{1,m}^\top \bx)$ and assume there is a $p$-$M^\ast$-\emph{1} interpolant $
    f^\ast(\bx)=\sum_{m=1}^{M^\ast} w^\ast_{2,m} \phi(\bw_{1,m}^{\ast\top} \bx)$
    that attains the minimal norm cost among all exact interpolants of width $\le M^\ast$. Assume the feature vectors $\Phi(\bw_{1,m}^\ast)\in\mathbb R^n$ with $\Phi(\bv)=\{\phi(\langle \bv,\bx_i\rangle)\}_{i=1}^n$ are nonzero, linearly independent (i.e., the network is not overparametrized) and identifiable up to permutations.
\end{assumption}

Furthermore, in our experiments, we use the following layerwise overparametrization notion.

\begin{definition}[Layerwise overparametrization]
Let layer $l$ have $M_l$ hidden units and define the layer’s feature matrix on the training inputs $\bX=(\bx_i)_{i=1}^n$ by
\[
\bm{\Xi}_l(\bw)\;:=\;\big(\xi_{l,m}(\bx_i;\bw)\big)_{i\in[n],\,m\in[M_l]}\in\mathbb R^{n\times M_l},
\]
where $\xi_{l,m}$ is the contribution of hidden unit $m$ to the next layer. We call layer $l$ \emph{overparametrized} if $\operatorname{rank}\big(\bm{\Xi}_l(\bw^\text{ref})\big)=:M_l^\ast < M_l$ at a reference parameter $\bw^\text{ref}$. The rank deficiency is $r_l:=M_l-M_l^\ast>0$.
\end{definition}

Overparametrization means that the column span of $\bm{\Xi}_l$ has dimension $M_l^\ast$ while $M_l$ units parameterize it. Any two parameter vectors with the same $\bm{\Xi}_l$ produce the same next-layer input up to the linear map $W_{l+1}$.

For transparency, we collect all assumptions underlying the results in \cref{sec:global}. \cref{cor:alpha_discrepancy} requires \cref{ass:regularity} and \cref{ass:factorization}. \cref{cor:simplexprob} further requires \cref{ass:weaklimit}. The experimental results do not rely on any of these assumptions and serve as independent validation.

\begin{assumption}[Manifold regularity] \label{ass:regularity}
For a fixed assignment $\varsigma$, the minimum-norm manifold $\mathcal{M}_\varsigma$ is a closed embedded $C^2$ manifold of $\mathbb{R}^d$ with positive reach, i.e., every point in $\mathcal{M}_\varsigma^\varepsilon$ has a unique nearest point on $\mathcal{M}_\varsigma$.
\end{assumption}

\begin{assumption}[Volume factorization] \label{ass:factorization}
In the $\varepsilon$-tube around $\mathcal{M}_\varsigma$, the contribution of directions orthogonal to the reallocation coordinates $\bm\rho^{(\varpi)}$ to the tube volume is independent of $\bm\rho^{(\varpi)}$.
\end{assumption}

\begin{assumption}[Existence of the weak limit] \label{ass:weaklimit}
The tube-conditioned posterior $\mathbb{P}_n^\varepsilon$ (\cref{def:mnman_tube}) converges weakly as $\varepsilon \downarrow 0$ to a well-defined limit $\mathbb{P}_n$ on $\mathcal{M}_\varsigma$. 
\end{assumption}

\cref{ass:weaklimit} is required for the moment bounds in \cref{cor:simplexprob}. Under \cref{ass:regularity}, this is expected to hold, but we do not verify it for specific ReLU network architectures.

\section{OMITTED THEORETICAL RESULTS, PROOFS AND DERIVATIONS} \label{app:proofs_derivations}

The statements made in this section, referring to statements and proofs of \cref{sec:twolayer}, assume the following overparametrization.

\begin{definition}
We define the surjective map $\varsigma: [M] \to [M^\ast]$, referred to as the assignment or reallocation map, such that for every $\varpi\in[M^\ast]$, $\sum_{m\in G_\varpi} \rho_m = 1$ with coefficients $(\rho_m)_{m\in[G_\varpi]} \in \Delta^{k_\varpi-1}$ on a simplex with dimension defined by $k_\varpi := |G_\varpi|$ and $G_\varpi := \{m\in[M]:\varsigma(m) = \varpi\}$. Further, let $\Delta := \prod_{\varpi\in[M^\ast]} \Delta^{k_\varpi -1}$.
\end{definition}

\begin{lemma} \label{lem:simplex}
    Under fixed assignment map $\varsigma$, \cref{ass:interpol} and according to \cref{def:mnman},
    \begin{align*} \label{eq:wmanifold}
       \mathcal{V}_\varpi := \{ & w_{2,m}, \bw_{1,m}, m\in [M] :\\  &\bw_{1,m} = \sqrt{\rho_m} \bw_{1,\varsigma(m)}^\ast, w_{2,m} = \operatorname{sign}(w_{2,\varsigma(m)}^\ast)\sqrt{\rho_m} |w_{2,\varsigma(m)}^\ast|, (\rho_m)_{m\in[M]}\in\Delta \} \; \in \; \mathcal{M}.
    \end{align*}
\end{lemma}

\begin{proof}
We must show (1) \(f(\bx,\bw)=f^\ast(\bx,\bw^\ast)\) for all \(\bx\in\mathcal{X}\) and all \(\bw\in\mathcal V_{\varpi}\), and (2) \(\mathcal R(\bw)=\mathcal R(\bw^\ast)\), where \(\mathcal R(\bw):=\sum_{m=1}^M \|\bw_{1,m}\|_2^2 + w_{2,m}^2\).
By the definition of \(\mathcal V_{\varpi}\), for each hidden index \(m\in[M]\), we set
\[
\bw_{1,m}=\sqrt{\rho_m}\,\bw_{1,\varsigma(m)}^\ast,\qquad
w_{2,m}=\operatorname{sign}\!\big(w_{2,\varsigma(m)}^\ast\big)\,\sqrt{\rho_m}\,\big|w_{2,\varsigma(m)}^\ast\big|,
\]
with coefficients \(\rho_m\ge 0\) satisfying \(\sum_{m\in G_\varpi}\rho_m=1\) for every \(\varpi\in[M^\ast]\).
Using the 1-homogeneity of ReLU, \(\phi(c\,t)=c\,\phi(t)\) for all \(c\ge 0\), we have the following.

\smallskip
\noindent\emph{1) Function equality.}
\begin{align*}
f(\bx,\bw)
&=\sum_{m=1}^{M} w_{2,m}\,\phi(\bw_{1,m}^\top \bx)\\
&=\sum_{m=1}^{M} \operatorname{sign}\!\big(w_{2,\varsigma(m)}^\ast\big)\sqrt{\rho_m}\,|w_{2,\varsigma(m)}^\ast|\,
\phi\!\big(\sqrt{\rho_m}\,\bw_{1,\varsigma(m)}^{\ast\top}\bx\big)\\
&=\sum_{m=1}^{M} \operatorname{sign}\!\big(w_{2,\varsigma(m)}^\ast\big)\,\rho_m\,|w_{2,\varsigma(m)}^\ast|\,
\phi\!\big(\bw_{1,\varsigma(m)}^{\ast\top}\bx\big)\\
&=\sum_{\varpi=1}^{M^\ast}\operatorname{sign}(w_{2,\varpi}^\ast)\,|w_{2,\varpi}^\ast|\,
\phi\!\big(\bw_{1,\varpi}^{\ast\top}\bx\big)(\textstyle\sum_{m\in G_\varpi}\rho_m)\\
&=\sum_{\varpi=1}^{M^\ast} w_{2,\varpi}^\ast\,\phi\!\big(\bw_{1,\varpi}^{\ast\top}\bx\big)
\;=\; f^\ast(\bx,\bw^\ast),
\end{align*}
since \(\operatorname{sign}(a)\,|a|=a\) and \(\sum_{m\in G_\varpi}\rho_m=1\).

\smallskip
\noindent\emph{2) Norm equality.}
\begin{align*}
\mathcal R(\bw)
&=\sum_{m=1}^M \|\bw_{1,m}\|_2^2 + w_{2,m}^2\\
&=\sum_{m=1}^M \rho_m\,\|\bw_{1,\varsigma(m)}^\ast\|_2^2
\;+\;\rho_m\,\big|w_{2,\varsigma(m)}^\ast\big|^2\\
&=\sum_{\varpi=1}^{M^\ast}\Big(\|\bw_{1,\varpi}^\ast\|_2^2+\big|w_{2,\varpi}^\ast\big|^2\Big)
(\textstyle\sum_{m\in G_\varpi}\rho_m)\\
&=\sum_{\varpi=1}^{M^\ast}\Big(\|\bw_{1,\varpi}^\ast\|_2^2+\big|w_{2,\varpi}^\ast\big|^2\Big)
\;=\;\mathcal R(\bw^\ast).
\end{align*}

Thus \(\bw\in\mathcal M\) by Definition~\ref{def:mnman}, which proves the claim.
\end{proof}
    
\subsection{Derivation of \cref{th:simplex}}
\label{th_app:simplex_formal}

We first restate \cref{th:simplex} again using the more formal setup of \cref{lem:simplex}:

\begin{theorem} \label{th:simplex-formal}
Assume the setup of \cref{lem:simplex}. Fix any assignment $\varsigma:[M]\to[M^\ast]$. Let $\bm\rho^{(\varpi)} = (\rho_m)_{m\in G_\varpi}$. On $\mathcal{M}_\varsigma$,
\[
\bm\rho^{(\varpi)} \sim \mathrm{Dirichlet}\left(\tfrac12,\dots,\tfrac12\right)
\quad\text{for each }\varpi=1,\dots,M^\ast,
\]
and the random vectors $\bm\rho^{(1)},\dots,\bm\rho^{(M^\ast)}$ are independent.
\end{theorem}

\begin{proof}
By Lemma~\ref{lem:simplex}, for fixed $\varsigma$, the minimum–norm manifold is parametrized blockwise with blocks $G_\varpi$. Within one of these blocks $G_\varpi$, all solutions $\bm\rho^{(\varpi)} \in \Delta^{k_\varpi -1}$. 
We first prove the distribution result and then derive its origin. First, parameterize each block $G_\varpi, \varpi\in[M^\ast]$, by nonnegative amplitudes $\bm \alpha^{(\varpi)}=(\alpha_m)_{m\in G_{\varpi}}\in \mathbb S^{k_\varpi-1}_+:=\left\{(\alpha_m)_{m\in G_{\varpi}}\in \mathbb{R}_{\ge 0}^{k_\varpi}:\sum_{m\in G_\varpi} \alpha_m^2=1\right\}$ and define $\bm \rho^{(\varpi)}=(\rho_m)_{m\in G_\varpi}$, $\rho_m:={\alpha_m^2}$. The Dirichlet distribution follows from the change of measure. Let $\bv^{(\varpi)} = (v_m)_{m\in G_\varpi} \sim \mathcal N(\bm 0, \mathbf I_{k_\varpi})$. Then $\bm\alpha^{(\varpi)}\stackrel d=\frac{|\bv^{(\varpi)}|}{\|\bv^{(\varpi)}\|_2}$ and $\rho_m = \alpha_m^2=\frac{v_m^2}{\sum_{j\in G_\varpi} v_j^2}$. Since $v_m^2$ are independent $\chi^2_1$ variables, $\bm\rho^{(\varpi)}$ is $\mathrm{Dirichlet}(\tfrac12,\dots,\tfrac12)$. Independence follows from independence of elements in $\bv^{(\varpi)}$. 

It remains to justify that, conditional on the fixed assignment $\varsigma$ and on $\bw\in\mathcal M_\varsigma$,
the induced law of the amplitude vector $\bm \alpha^{(\varpi)}\in\mathbb S^{k_\varpi-1}_+$ is the uniform surface measure
(on the positive orthant), and that different blocks factorize.
By Lemma~\ref{lem:simplex}, the likelihood is constant on $\mathcal M_\varsigma$ as is the (radial) prior. Hence the posterior restricted to $\mathcal M_\varsigma$ is proportional to the intrinsic volume measure on $\mathcal M_\varsigma$.
On $\mathcal M_\varsigma$,
$\boldsymbol\omega_m=\alpha_m\boldsymbol\omega_\varpi^\ast$ with $\bm \alpha^{(\varpi)}\in\mathbb S^{k_\varpi-1}_+$.
The map $\bm \alpha^{(\varpi)}\mapsto(\boldsymbol\omega_m)_{m\in G_\varpi}$ is linear and its Jacobian has constant magnitude
$(\|\boldsymbol\omega_\varpi^\ast\|_2)^{k_\varpi-1}$, independent of $\bm \alpha^{(\varpi)}$. Therefore, the induced density of
$\bm \alpha^{(\varpi)}$ is constant on $\mathbb S^{k_\varpi-1}_+$, i.e., $\bm \alpha^{(\varpi)}$ is uniform on $\mathbb S^{k_\varpi-1}_+$.
Moreover, the parametrization is a product over blocks, so the intrinsic volume measure (hence the posterior)
factorizes across blocks and independence of $\bm \alpha^{(1)},\dots,\bm \alpha^{(M^\ast)}$ follows.
\end{proof}

\subsection{Proof of \cref{cor:alpha_discrepancy}}

\begin{proof}
For a fixed block $G_\varpi$, write each duplicated neuron as
\[
\boldsymbol\omega_m = r_m \bm u_\varpi^\ast,
\qquad
r_m := \|\boldsymbol\omega_m\|
= \sqrt{\rho_m}\,\|\boldsymbol\omega_\varpi^\ast\|,
\]
where $\bm u_\varpi^\ast := \boldsymbol\omega_\varpi^\ast/\|\boldsymbol\omega_\varpi^\ast\|$
is a unit vector.
In $\R^{p+1}$, the set of vectors with fixed norm $r_m$
forms a $p$-dimensional sphere $\mathbb S^p(r_m)$.
Hence, in a small $\varepsilon$-neighborhood of the ray
$\{t\, \bm u_\varpi^\ast : t\ge 0\}$,
the orthogonal directions correspond locally to perturbations
within this $p$-dimensional sphere.
The $p$-dimensional surface measure of $\mathbb S^p(r_m)$ scales as
\[
\mathrm{Vol}\big(\mathbb S^p(r_m)\big)
\propto r_m^{\,p}.
\]
Therefore, the cross-sectional volume in the $p$ orthogonal directions
associated with the $m$-th duplicated neuron scales as $r_m^{\,p}
= (\sqrt{\rho_m}\,\|\boldsymbol\omega_\varpi^\ast\|)^{p}
\propto \rho_m^{p/2}$.
All remaining factors are independent of $\bm\rho$ and cancel after normalization.
Therefore, the total tube-volume weight attached to a configuration $\bm\rho$ is
\[
\mathrm{Vol}(\bm\rho)
\;\propto\;
\prod_{m=1}^k \rho_m^{p/2}.
\]
To obtain the density of $\bm\rho$, we multiply this volume by the density
$\prod_{m=1}^k \rho_m^{-1/2}$ corresponding to Dirichlet$(\tfrac12)$. This yields
\[
p_{\text{tube}}(\bm\rho)
\;\propto\;
\prod_{m=1}^k
\rho_m^{-1/2}\,
\rho_m^{p/2}
=
\prod_{m=1}^k
\rho_m^{\frac{p+1}{2}-1}.
\]
This is exactly the density of a symmetric Dirichlet distribution
with concentration parameters $(\tfrac{p+1}{2},\dots,\tfrac{p+1}{2})$.
Independence across blocks follows from the product structure of the
parameterization for fixed $\varsigma$.
\end{proof}

\subsection{Proof of \cref{cor:simplexprob}}

We again state \cref{cor:simplexprob} more precisely. 

\begin{corollary}\label{cor:simplexprob_again}
Under the setup of Theorem~\ref{th:simplex-formal}, fixing an assignment $\varsigma$ and under the tube-law $\mathbb{P}_n$ induced on the corresponding manifold $\mathcal{M}_\varsigma$, let $G_\varpi$ be a block of size $k_\varpi$ and let $w_m$ and $w_{m'}$ denote two scalar weights from the same block with $w_m \neq w_{m'}$. Then
\[
\mathbb{E}_{\mathbb{P}_n}(w_m)=O(k_\varpi^{-1/2}),\qquad 
\operatorname{Cov}_{\mathbb{P}_n}(w_m,w_{m'})=O(k_\varpi^{-2}).
\]
Assuming the group size $k_\varpi$ to increase proportionally with $M$, i.e., $\exists c>0: k_\varpi>cM$, one has $\mathbb{E}_{\mathbb{P}_n}(w)=O(M^{-1/2})$ and $\operatorname{Cov}_{\mathbb{P}_n}(w_m,w_{m'})=O(M^{-2})$. If $w_m$ and $w_{m'}$ belong to different blocks, $\operatorname{Cov}_{\mathbb{P}_n}(w_m,w_{m'})=0$.
\end{corollary}

\begin{proof}
In the following, we assume all expectations and covariances to be taken w.r.t.\ the distribution $\mathbb{P}_n$ and drop the corresponding index. Further, to make a statement w.r.t. $M$ while working with group size $k_\varpi$, we assume that both grow proportionally, i.e., $\exists c>0: k_\varpi>cM$.

Fix a block $G_\varpi$ of size $k:=k_\varpi$ and write the duplicate amplitudes as $a_m\ge 0$ with $\sum_{m\in G_\varpi}a_m^2=1$ and $\rho_m:=a_m^2$. Following \cref{cor:alpha_discrepancy} for fixed assignment, $\bm{\rho}:=(\rho_m)_{m\in G_\varpi}\sim\mathrm{Dirichlet}(\frac{p+1}{2},\ldots,\frac{p+1}{2})$. This implies that the marginal distribution of each reallocation is $\rho_m \sim \text{Beta}(\alpha, (k-1)\alpha)$ for $\alpha=(p+1)/2$. \\

\emph{Mean}: Since for every $w_m, m\in G_\varpi$, it holds $w_m = \sqrt{\rho_m} w_\varpi^\ast$, it follows directly
$$
\mathbb{E}[w_m] = \mathbb{E}[\sqrt{\rho_m} w^\ast_\varpi] =  \mathbb{E}[\sqrt{\rho_m}]\ w^\ast_\varpi = \mathbb{E}[\sqrt{\rho_m}]\ w^\ast_\varpi = \frac{\Gamma(\alpha+\tfrac12)\Gamma(k\alpha)}{\Gamma(\alpha)\Gamma(k\alpha+\tfrac12)} w_\varpi^\ast = \Theta(k^{-1/2}) = \mathcal{O}(M^{-1/2}).
$$
assuming $\exists c>0: k>cM$.\\

\emph{Covariance within a block}: Applying the same logic for two weights $w_m, w_{m'}\in G_\varpi$ with $w_m\neq w_{m'}$, we have
\begin{align*}
\text{Cov}(w_m,w_{m'}) &= \mathbb{E}[w_m w_{m'}] - \mathbb{E}[w_m]\mathbb{E}[w_{m'}] = (w^\ast_\varpi)^2 \mathrm{Cov}(w_m,w_{m'})\\
&=(w_\varpi^\ast)^2\left(\frac{\Gamma(\alpha+\tfrac12)^2,\Gamma(k\alpha)}{\Gamma(\alpha)^2,\Gamma(k\alpha+1)} - 
\left(\frac{\Gamma(\alpha+\tfrac12)\Gamma(k\alpha)}{\Gamma(\alpha)\Gamma(k\alpha+\tfrac12)}\right)^2
\right) =\mathcal O(k^{-2}) =\mathcal O(M^{-2}).
\end{align*} 

\emph{Covariance between blocks}: Last, since $\text{Cov}(\rho_m,\rho_{m'}) = 0$ for $\varsigma(m) \neq \varsigma(m')$, it follows $\operatorname{Cov}_{\mathbb{P}_n}(w_m,w_{m'})=0$.

\end{proof}

\subsection{Proof of \cref{th:balance}}

\begin{proof}
We first repeat that we have 
\begin{equation} \label{eq:expectduagain}
\mathbb{E}_\pi[\langle \bW_l, \nabla_{\bW_l}\mathcal{L}(\bw)\rangle_F]\;=\;\mathbb{E}_\pi[\langle \bW_{l+1},\nabla_{\bW_{l+1}}\mathcal{L}(\bw)\rangle_F],    
\end{equation}
which is based on the pointwise positive homogeneity of $f(\bx) = \bW_L \phi(\bW_{L-1}\phi(\cdots(\bW_1\bx)))$ \citep{du2018algorithmic}. This provides the inner identity of (\ref{eq:expectduagain}), which we extend under standard regularity conditions to the expectation in the stationary limit of a consistent sampler \citep{pavliotis2014stochastic}.

Stein’s identity for $\pi$ with the test function $f(\bW)=\bW_l$ gives
\[
\mathbb E_\pi\!\big[\langle\bW_l,\nabla_{\bW_l}\mathcal{L}_{\bm{\tau}}(\bw)\rangle_F\big]\;=\;d_l.
\]
Using $\nabla_{\bW_l}\mathcal{L}_{\bm{\tau}}(\bw)=\nabla_{\bW_l}\mathcal{L}(\bw) + \tau_l^{-2}\bW_l$,
\[
\mathbb E_\pi\!\big[\langle\bW_l,\nabla_{\bW_l}\mathcal{L}(\bw)\rangle_F\big]
\;+\;
\frac{1}{\tau_l^2}\,\mathbb E_\pi\!\big[\|\bW_l\|_F^2\big]
\;=\;d_l.
\]
Apply the same identity with $l{+}1$ and subtract. By the homogeneity identity,
\[
\frac{1}{\tau_l^2}\,\mathbb E_\pi\!\big[\|\bW_l\|_F^2\big]
-\frac{1}{\tau_{l+1}^2}\,\mathbb E_\pi\!\big[\|\bW_{l+1}\|_F^2\big]
\;=\;
d_l-d_{l+1}.
\]
\end{proof}

\begin{corollary}\label{cor:balanced}
Under the assumptions of the Layer-balance theorem, define
\[
B_l \;:=\; \frac{1}{\tau_l^2}\,\mathbb E_\pi\!\big[\|\bW_l\|_F^2\big] \;-\; d_l, 
\qquad h=1,\dots,L.
\]
Then $B_l$ is constant across layers, i.e.,
\[
B_1 \;=\; B_2 \;=\; \cdots \;=\; B_L.
\]
In particular, if there exists an index $h_0$ such that $\mathbb E_\pi[\|\bW_{h_0}\|_F^2]=\tau_{h_0}^2 d_{h_0}$, then
\[
\mathbb E_\pi\!\big[\|\bW_l\|_F^2\big] \;=\; \tau_l^2 d_l
\quad\text{for all } h=1,\dots,L.
\]
\end{corollary}

\subsection{Proof of \cref{cor:balanced}}

\begin{proof}
By the Layer-balance theorem, for each adjacent pair $(l,l{+}1)$,
\[
\frac{1}{\tau_l^2}\,\mathbb E_\pi\!\big[\|\bW_l\|_F^2\big]
-\frac{1}{\tau_{l+1}^2}\,\mathbb E_\pi\!\big[\|\bW_{l+1}\|_F^2\big]
\;=\;
d_l-d_{l+1}.
\]
Rearranging gives
\(
\big(\tau_l^{-2}\mathbb E_\pi\|\bW_l\|_F^2-d_l\big)
=
\big(\tau_{l+1}^{-2}\mathbb E_\pi\|\bW_{l+1}\|_F^2-d_{l+1}\big),
\)
hence $B_l=B_{l+1}$ for all $l$. Transitivity yields $B_1=\cdots=B_L$. If for some $l_0$ we have $\mathbb E_\pi\|\bW_{l_0}\|_F^2=\tau_{l_0}^2 d_{l_0}$, then $B_{l_0}=0$ and thus $B_l=0$ for all $l$, which implies $\mathbb E_\pi\|\bW_l\|_F^2=\tau_l^2 d_l$ for all $l$.
\end{proof}

\begin{proposition}
    Assume a homogeneous network $f(\bw)$ that admits a rescaling symmetry between layer weights of layer $l$ and $l+1$ and a layer-specific penalty $\mathcal{R}(\bw) = \ldots + \lambda_l ||\bW_l||_F^2 + \lambda_{l+1} ||\bW_{l+1}||_F^2 + \ldots$. Along any rescaling that preserves the represented function, the penalty is minimized precisely when
$$
\frac{\|\bW_l\|_F}{\|\bW_{l+1}\|_F}\;=\;\sqrt{\frac{\lambda_{l+1}}{\lambda_l}}.
$$
\end{proposition}

\begin{proof}
    By positive 1-homogeneity, for any $a>0$ the transformation

$$
\bW_l' = a\,\bW_l,\qquad \bW_{l+1}' = a^{-1}\bW_{l+1},
$$

leaves $f(\bw)$ unchanged. Hence $\mathcal L(\bw)$ is invariant along this 1-dimensional rescaling orbit, and minimizing $\mathcal L_\lambda$ on the orbit reduces to minimizing

$$
\mathcal{R}_{l,l+1}(a)\;=\;\lambda_l \| \bW_l' \|_F^2 + \lambda_{l+1}\| \bW_{l+1}' \|_F^2
\;=\;\lambda_l a^2 x^2 + \lambda_{l+1} a^{-2} y^2,
$$

where $x=\|\bW_l\|_F$, $y=\|\bW_{l+1}\|_F$. The function $\mathcal{R}_{l,l+1}(a)$ is strictly convex in $a^2$ and coercive, so it has a unique minimizer. Differentiating gives

$$
\mathcal{R}_{l,l+1}'(a)=2\lambda_l a x^2 - 2\lambda_{l+1} a^{-3} y^2=0
\;\;\Longleftrightarrow\;\;
\lambda_l a^4 x^2=\lambda_{l+1} y^2.
$$

At the minimizer,

$$
\frac{\|\bW_l'\|_F}{\|\bW_{l+1}'\|_F}
=\frac{a x}{a^{-1} y}
=a^2 \frac{x}{y}
=\sqrt{\frac{\lambda_{l+1}}{\lambda_l}}.
$$
\end{proof}

\begin{remark}
    The following results are based on the balancedness property induced by the Gaussian prior. For simplicity, we assume $\tau_l \equiv \tau \,\forall l\in[L]$, but note that the following results, \cref{th:simplex} and \cref{cor:simplexprob} could be analogously derived using layer-specific priors.
\end{remark}

\begin{corollary}\label{cor:inout}
Under the assumptions of \cref{th:balance}, let $\ba_{l,j}$ denote the incoming weight vector of hidden unit $j$ in layer $l$ 
and $\bv_{l,j}$ its outgoing weight vector. Following the same logic as in \cref{th:balance} by transferring  \citet{du2018algorithmic} to our sampling setup, we have
$$\mathbb E_\pi\!\big[\ba_{l,j}^\top\nabla_{\ba_{l,j}}\mathcal{L}(\bw)\big]\;=\;\mathbb E_\pi\!\big[\bv_{l,j}^\top\nabla_{\bv_{l,j}}\mathcal{L}(\bw)\big].$$
Then for every hidden unit $j$ in layer $l$,
\[
\frac{1}{\tau_l^2}\,\mathbb E_\pi\!\big[\|\ba_{l,j}\|_2^2\big]
\;-\;
\frac{1}{\tau_{l+1}^2}\,\mathbb E_\pi\!\big[\|\bv_{l,j}\|_2^2\big]
\;=\;
d^{\mathrm{in}}_{l,j}\;-\;d^{\mathrm{out}}_{l,j},
\]
where $d^{\mathrm{in}}_{l,j}$ and $d^{\mathrm{out}}_{l,j}$ are the input and output dimensions of neuron $(l,j)$, respectively.
\end{corollary}

\begin{remark}
Summing the identities of Corollary~\ref{cor:inout} over all neurons $j$ in a given layer $h$ recovers the layer-level balance of Theorem~\ref{th:balance}.
\end{remark}

\subsection{Proof of \cref{cor:inout}}

\begin{proof}
Apply Stein’s identity \citep[see, e.g.,][]{liu2016stein} to the test function $f(\bW)=\ba_{l,j}$ to obtain
\[
\mathbb E_\pi\!\big[\ba_{l,j}^\top\nabla_{\ba_{l,j}}\mathcal{L}(\bw)\big]
\;+\;\frac{1}{\tau_l^2}\,\mathbb E_\pi\!\big[\|\ba_{l,j}\|_2^2\big]
\;=\; d^{\mathrm{in}}_{l,j}.
\]
Similarly, for $f(\bW)=\bv_{l,j}$,
\[
\mathbb E_\pi\!\big[\bv_{l,j}^\top\nabla_{\bv_{l,j}}\mathcal{L}(\bw)\big]
\;+\;\frac{1}{\tau_{l+1}^2}\,\mathbb E_\pi\!\big[\|\bv_{l,j}\|_2^2\big]
\;=\; d^{\mathrm{out}}_{l,j}.
\]
Subtracting these two equalities and invoking the neuron-wise backprop homogeneity identity cancels the likelihood terms, which yields the claim.
\end{proof}

\subsection{Assumptions \& Details for \cref{sec:deep}: Overparametrization and Prior Conformity}

The approximation $\bH^\star \approx \bJ^\top \bm\Upsilon \bJ + 2\lambda \bI$ is the Gauss-Newton approximation of the true Hessian of the negative log-posterior. This approximation is valid under the conditions that $\bw^\text{ref}$ represents a good model fit, i.e., the gradients of the loss w.r.t.\ the outputs are close to zero, as well as the network output function being nearly linear w.r.t.\ the weights $\bw$ around $\bw^\text{ref}$ \citep{martens2020new}.

The statement $\ker(\bJ)\cap \mathcal{W} \;\neq\;\{\bm{0}\}$ follows directly from the definition of an overparametrized layer in \cref{app:defs}.

We now show that $(\bH^\star)^{-1}\big|_{\ker(\bJ)} \;=\; (2\lambda)^{-1} \bI = \tau^2 \bI$, so that the matrix $(\bH^\star)^{-1}$, when restricted to the subspace $\ker(\bJ)$, acts as a scalar multiplication by $(2\lambda)^{-1}.$

Let $\bm{v} \in \ker(\bJ)$ be an arbitrary non-zero vector. By definition, this means $\bJ \bm{v} = \bm{0}$. Then applying the Hessian approximation $\bH^\star$ yields:

\[
\bH^\star \bm{v} = (\bJ^\top \bm\Upsilon \bJ + 2\lambda \bI) \bm{v} = \bJ^\top \bm\Upsilon \bJ \bm{v} + 2 \lambda(\bI \bm{v}) = 2 \lambda \bm{v}.
\]

We obtain $(\bH^\star)^{-1} \bm{v} = (2 \lambda)^{-1} \bm{v}$ and it follows that $(\bH^\star)^{-1}\big|_{\ker(\bJ)} \;=\; (2\lambda)^{-1} \bI = \tau^2 \bI$.

\section{EXPERIMENTAL DETAILS AND FURTHER ANALYSES} \label{app:additional_results}

\subsection{Experimental Setup}\label{app:exp_setup}

\paragraph{Software}
For sampling-based inference results, we implement our analysis in Python and mainly rely on the \texttt{jax} \citep{jax2018github} and \texttt{BlackJAX} \citep{cabezas2024blackjax} libraries. For some comparisons, we used the \texttt{posteriors} package \citep{duffield2025scalable}.
Our code is available at \url{https://github.com/EmanuelSommer/bnn_interplay_priors_overparam}.

\paragraph{Computing Environment} The experiments were conducted on two NVIDIA RTX A6000 GPUs and an AMD Ryzen™ Threadripper™ PRO 5000WX/3000WX CPU with 64 cores. For most experiments, 10 chains were sampled in parallel on the CPU, enabling efficient parallelization and allowing multiple experiments to run concurrently. For larger-scale experiments involving thousands of chains, 50 chains were sampled in parallel to maximize resource utilization. For larger CNNs, we used parallel training on GPUs.

\paragraph{Datasets} \cref{tab:dataoverview} summarizes the benchmark datasets utilized in our experiments. For all tabular benchmarks, if not specified otherwise, we use a 70\% train, 10\% validation, and 20\% test split as well as a fully connected model architecture with $3$ hidden layers, 16 neurons per layer. For all image classification benchmarks, we use the suggested train/test split and CNNs of varying size.

\begin{table}[h]
\begin{small}
\begin{center}
\caption{Benchmark datasets overview.} \label{tab:dataoverview}
\vskip 0.1in
\resizebox{0.5\columnwidth}{!}{%
\begin{tabular}{lrrl}
Dataset  & Size & Features & Source \\ \hline 
Airfoil & 1503 & 5 &  \citet{Dua.2019}  \\
Bikesharing & 17379 & 13 & \citet{misc_bike_sharing_dataset_275}  \\
Concrete & 1030 & 8 &  \citet{Yeh.1998} \\
Energy & 768 & 8 &  \citet{Tsanas.2012} \\
Ionosphere & 351 & 34 & \citet{sigillito1989ionosphere} \\
F(ashion)-MNIST & 60000 & 28x28 & \citet{xiao2017/online} \\
CIFAR-10 & 60000 & 28x28 & \citet{krizhevsky2009learning} \\
\hline
\end{tabular}
}
\end{center}
\end{small}
\end{table}

\paragraph{Performance Evaluation}

To quantify the quality of the posterior predictive approximation and thus the UQ capabilities of the models we use the log posterior predictive density (LPPD) \citep{gelman2014a, wiese2023towards, MILE} over a test set $\mathcal{D}_{\text{test}}$, defined as
\begin{align} \label{eq:lppd}
    \text{LPPD} = \frac{1}{n_{\text{test}}} \sum_{(\bm{y}^\ast, \bm{x}^\ast) \in \mathcal{D}_{\text{test}}} \log \left(\frac{1}{K \cdot S} \sum_{k=1}^K \sum_{s=1}^{S} p \left(\bm{y}^\ast | \bm{\theta}^{(k,s)}(\bm{x}^\ast)\right)\right).
\end{align}
Here, $K$ denotes the number of chains, $S$ the number of samples per chain, and $\bm{\theta}^{(k,s)}$ the parameters from the $s$-th sample of the $k$-th chain. Intuitively, the LPPD quantifies how well the predictive distribution aligns with the observed labels, with higher values indicating higher density coverage, i.e., improved UQ performance.

In addition, we employ the root mean squared error (RMSE) for regression and the accuracy for classification tasks to check for the accuracy of point predictions.

\subsection{Experimental Details for \cref{fig:kerJ}}

Below, let data be $(x_i,y_i)_{i=1}^n$, noise $\varepsilon_i\sim\mathcal N(0,\sigma^2)$ with known $\sigma^2$, and
$$
f(x)=a\,\mathrm{ReLU}(b x)+c\,\mathrm{ReLU}(d x),\qquad a,b,c,d\stackrel{\text{i.i.d.}}{\sim}\mathcal N(0,1).
$$
Define $\phi_{1,i}(b)=\mathrm{ReLU}(b x_i)$, $\phi_{2,i}(d)=\mathrm{ReLU}(d x_i)$, the design
$$
\Phi(b,d)=
\begin{bmatrix}
\phi_{1,1}(b) & \phi_{2,1}(d)\\
\vdots & \vdots\\
\phi_{1,n}(b) & \phi_{2,n}(d)
\end{bmatrix}\in\mathbb R^{n\times 2},\quad
\theta=\begin{bmatrix}a\\ c\end{bmatrix},\quad y=(y_1,\dots,y_n)^\top.
$$
The likelihood and prior are given by
$$
p(y\mid \theta,b,d)\;=\;\mathcal N\!\big(y;\,\Phi(b,d)\theta,\;\sigma^2 I_n\big),\qquad
p(\theta,b,d)\;=\;\mathcal N(\theta;0,I_2)\,\mathcal N(b;0,1)\,\mathcal N(d;0,1).
$$
Then, the kernel of the joint posterior is
$$
p(\theta,b,d\mid y)\;\propto\;
\exp\!\Big(-\tfrac{1}{2\sigma^2}\|y-\Phi(b,d)\theta\|^2-\tfrac12\|\theta\|^2-\tfrac12 b^2-\tfrac12 d^2\Big).
$$
Because the model is linear in $(a,c)$ given $(b,d)$, we get conjugacy for $\theta\mid b,d,y$ and a conditional posterior for $a,c$ given $b,d$ given by
$$
\theta\mid b,d,y\;\sim\;\mathcal N\!\big(\mu_{ac}(b,d),\,\Sigma_{ac}(b,d)\big),
$$
with
$$
\Sigma_{ac}(b,d)\;=\;\Big(I_2+\tfrac{1}{\sigma^2}\Phi(b,d)^\top \Phi(b,d)\Big)^{-1},\qquad
\mu_{ac}(b,d)\;=\;\Sigma_{ac}(b,d)\,\tfrac{1}{\sigma^2}\Phi(b,d)^\top y.
$$
Therefore, the marginal likelihood for $b,d$ and their posterior (by integrating out $(a,c)$) is
$$
p(y\mid b,d)\;=\;\mathcal N\!\big(y;\,0,\;\sigma^2 I_n+\Phi(b,d)\Phi(b,d)^\top\big),
$$
hence
$$
p(b,d\mid y)\;\propto\; \mathcal N\!\big(y;\,0,\;\sigma^2 I_n+\Phi(b,d)\Phi(b,d)^\top\big)\,\mathcal N(b;0,1)\,\mathcal N(d;0,1).
$$
This pair has no closed form. It can be explored by MAP optimization of $\log p(b,d\mid y)$ or by MCMC. Further, the full posterior factorization is given by
$$
p(a,b,c,d\mid y)\;=\;p(a,c\mid b,d,y)\;p(b,d\mid y),
$$
with $p(a,c\mid b,d,y)$ Gaussian as above and $p(b,d\mid y)$ given up to a normalizing constant.

\subsection{The Role of the Bias}

Extending the exposition about the role of the bias in \cref{sec:deep}, we provide an illustrative marginal bias distribution across layers of a ReLU network in \cref{fig:marginal_bias}.

\begin{figure}[!htb]
    \centering
    \includegraphics[width=\linewidth]{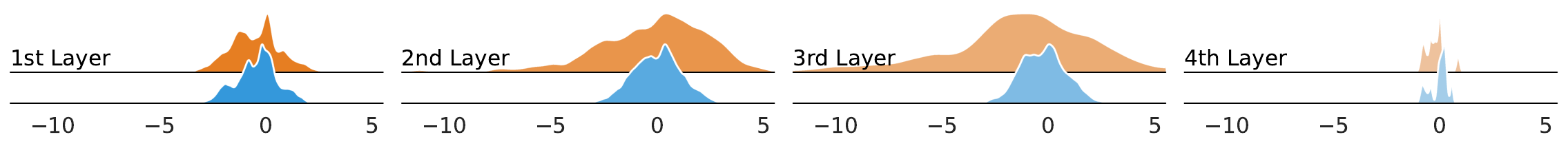}
    \caption{Empirical marginal distribution of the bias of a four-hidden-layer MLP with a \textcolor{snsblue}{normal prior on the bias} as well as a \textcolor{snsorange}{uniform prior on the bias}. }
    \label{fig:marginal_bias}
\end{figure}

In frequentist training of ReLU neural networks, it is common practice not to regularize the bias terms, as they primarily shift the activation thresholds and are not prone to uncontrolled growth like weights. In contrast, BNNs explicitly place priors over all parameters, including biases, which induces a form of regularization. \cref{fig:marginal_bias} illustrates this difference: when a uniform prior is imposed, the distributions appear wider, reflecting greater posterior uncertainty, whereas introducing isotropic Gaussian priors yields narrower marginals. Importantly, both cases exhibit similar shapes in their high-density regions near the mode, indicating that regularization primarily reduces tail spread without altering the local geometry of the posterior around the most probable values. 

\subsection{Experimental Details of \cref{fig:function_space_sensitivity}}

For each posterior sample, we identify clusters of neurons with similar activation patterns. We compute the cosine similarity matrix $\tilde{\bm{\Xi}}^\top \tilde{\bm{\Xi}}$ of the standardized feature columns (excluding dead or constant units) and form clusters as connected components of the thresholded similarity graph, where the threshold is calibrated via a Bonferroni correction under the null hypothesis of independent activation columns. For each cluster with indices $j$, we compute
$$
  S=\min_{\substack{\bm{v}^\top\mathbf{1}=0\\ \|\bm{v}\|_2=1}}\|\bm{\Xi}_{:,j} \bm{v}\|_2.
$$
Here $\bm{v}$ represents a zero-sum reweighting of outgoing weights within the cluster. We use a zero-sum reweighting to redistribute weight among neurons within a given cluster. Small $S$ indicates that such reweightings barely affect the network output, which is a likelihood-flat direction. We use unstandardized features $\bm{\Xi}$ (not $\tilde{\bm{\Xi}}$) for this computation to ensure $S$ reflects the magnitude of the actual output change under reweighting, not just the collinearity of activations.

\subsection{Exploring the Limits of BDEs}\label{app:limitsbde}

We extend the analysis of \citet{sommer2024connecting}, who only consider 12 to 10k chains of 1k samples each. We also use a more than twice as large fully-connected neural network (4 hidden layers of 16 neurons each) to perform distributional regression. For this, we use the recently proposed MILE approach \citep{MILE} and configure it exactly as suggested by the authors. Due to the
immense computational load of sampling this amount of chains and also evaluating the posterior samples (compressing the samples roughly amounts to 100GB for a single experiment), we carefully select four benchmark datasets, namely, \texttt{airfoil} and \texttt{bikesharing} for distributional regression, \texttt{ionosphere} for tabular classification, and \texttt{Fashion-MNIST} for image classification. For the latter, we do not use MILE but rather employ scale-adapted SGHMC \citep{adasghmc} for computational efficiency. In our analysis, we focus on two major aspects. First, we analyze how the performance of the model develops when adding chains to the Bayesian Model Average (BMA). Second, in the spirit of \citet{sommer2024connecting} we take a closer look at bivariate margins of the empirical posterior derived from SAI.

\paragraph{Performance Evolution}
The cumulative performance, which we---focusing on UQ quality---measure with the LPPD, of adding chains to the BMA obviously depends on the order in which chains are added. Thus, we consider 5 different orderings and report means and standard deviations of the cumulative LPPD over chains in \cref{fig:cum_wasser} and \ref{fig:fireball_cum_lppd_fmnist}. The results suggest that with even a rather small number of chains, the performance saturates quite fast, but slowly increases further until exhibiting a very strong performance for 10k chains. Parallelizing 10-20 chains on modern hardware is very easy and comes with no considerable cost overhead over single-chain sampling. This also has very positive implications on memory requirements and inference time, rendering the approach practically feasible.

\begin{figure}
    \centering
    \includegraphics[width=0.35\linewidth]{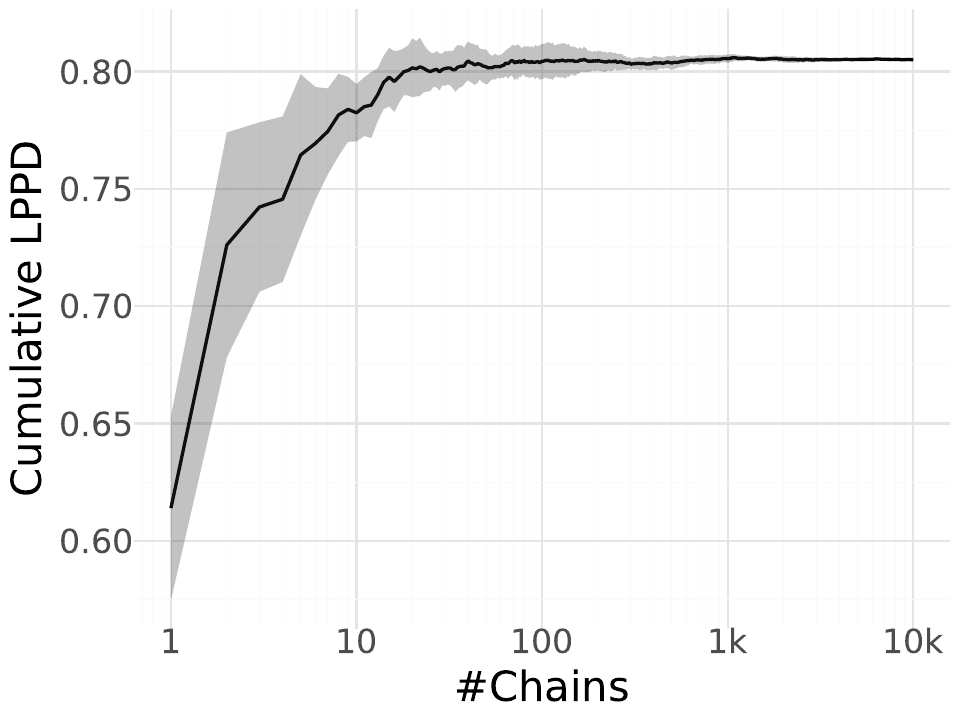}
    \caption{Cumulative LPPD over the number of chains (standard deviation across 5 random chain orderings) on the \texttt{airfoil} dataset.}
    \label{fig:cum_wasser}
\end{figure}

\begin{figure}[h]
    \centering
    \begin{subfigure}[b]{0.45\textwidth}
        \centering
        \includegraphics[width=0.8\textwidth]{imgs/cumulative_lppd_fmnist.pdf}
        \caption{\texttt{CNNv1} on Fashion-MNIST with SG-MCMC.}
        \label{fig:subfig1_app}
    \end{subfigure}
    \hfill
    \begin{subfigure}[b]{0.45\textwidth}
        \centering
        \includegraphics[width=0.8\textwidth]{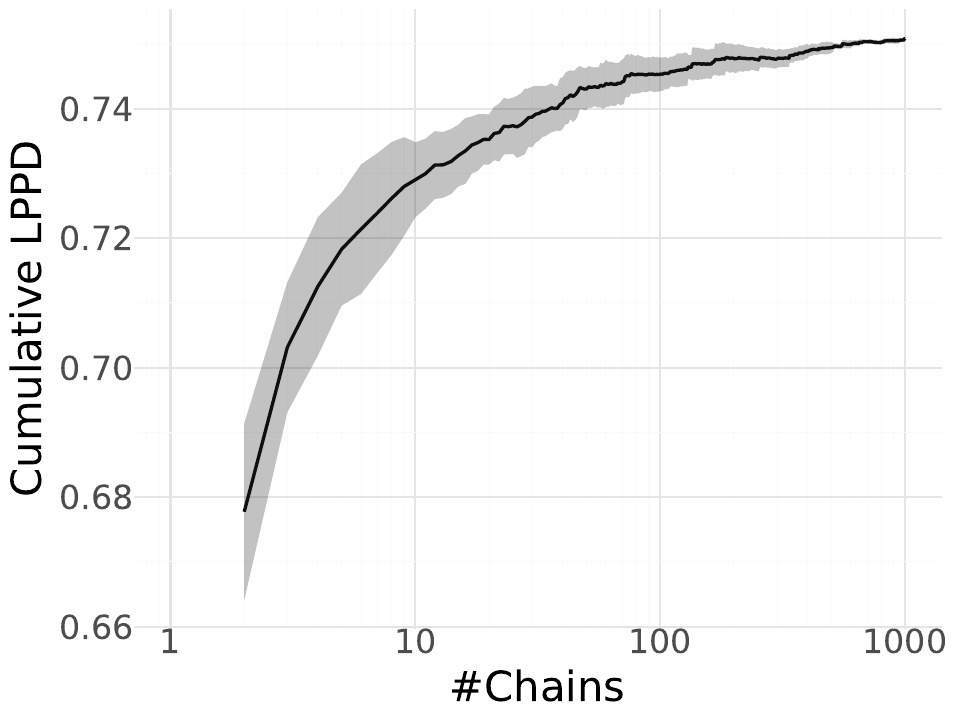}
        \caption{FCN on \texttt{bikesharing} with MILE.}
        \label{fig:subfig2_app}
    \end{subfigure}
    \caption{Cumulative LPPD over the number of chains (standard deviation across 5 random chain orderings) for the same \texttt{CNNv1} setting on Fashion-MNIST as in \cref{fig:fmnistfireballgrid} (left) and the same fully-connected neural network (FCN) as the one in \cref{fig:cum_wasser} but fitted on the larger \texttt{bikesharing} dataset. As a reference, the DE with 1k members only achieves an LPPD of -0.3627 for \texttt{CNNv1} and 0.4935 for the FCN. Also, the predictive accuracy of the BDE (1k) exceeds that of the DE (1k) in both settings (left: accuracy $0.8715 > 0.8706$, right: RMSE $0.2258 < 0.2508$).}
    \label{fig:fireball_cum_lppd_fmnist}
    \vspace{-0.2in}
\end{figure}

\begin{figure}[h]
    \centering
    \setlength{\tabcolsep}{0pt}
    \renewcommand{\arraystretch}{0}
    \begin{tabular}{c@{}c@{}c@{}c@{}c@{}c}
        \rotatebox{90}{\;\;\;\quad\quad   \footnotesize Kernel} &
        \includegraphics[width=0.18\linewidth]{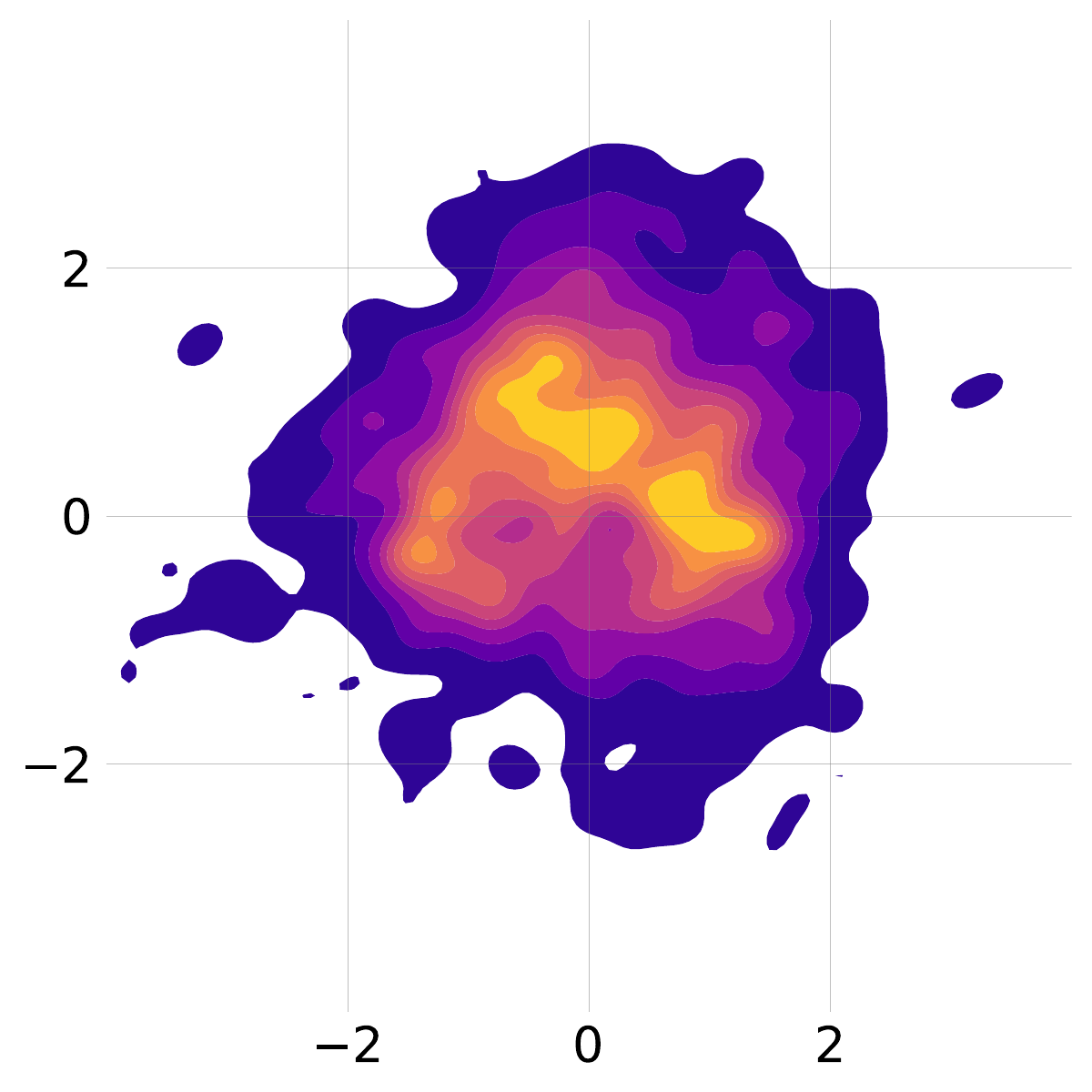} &
        \includegraphics[width=0.18\linewidth]{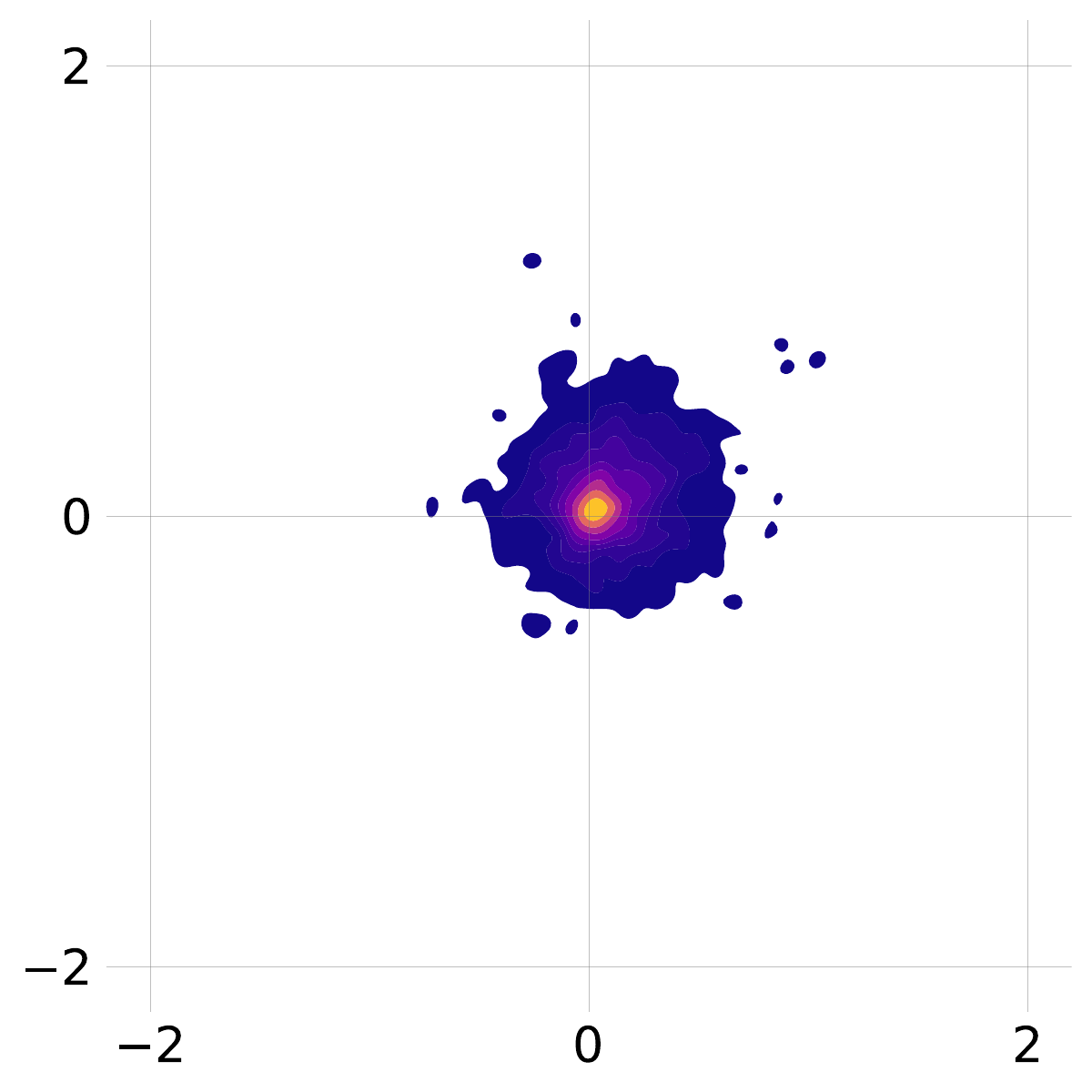} &
        \includegraphics[width=0.18\linewidth]{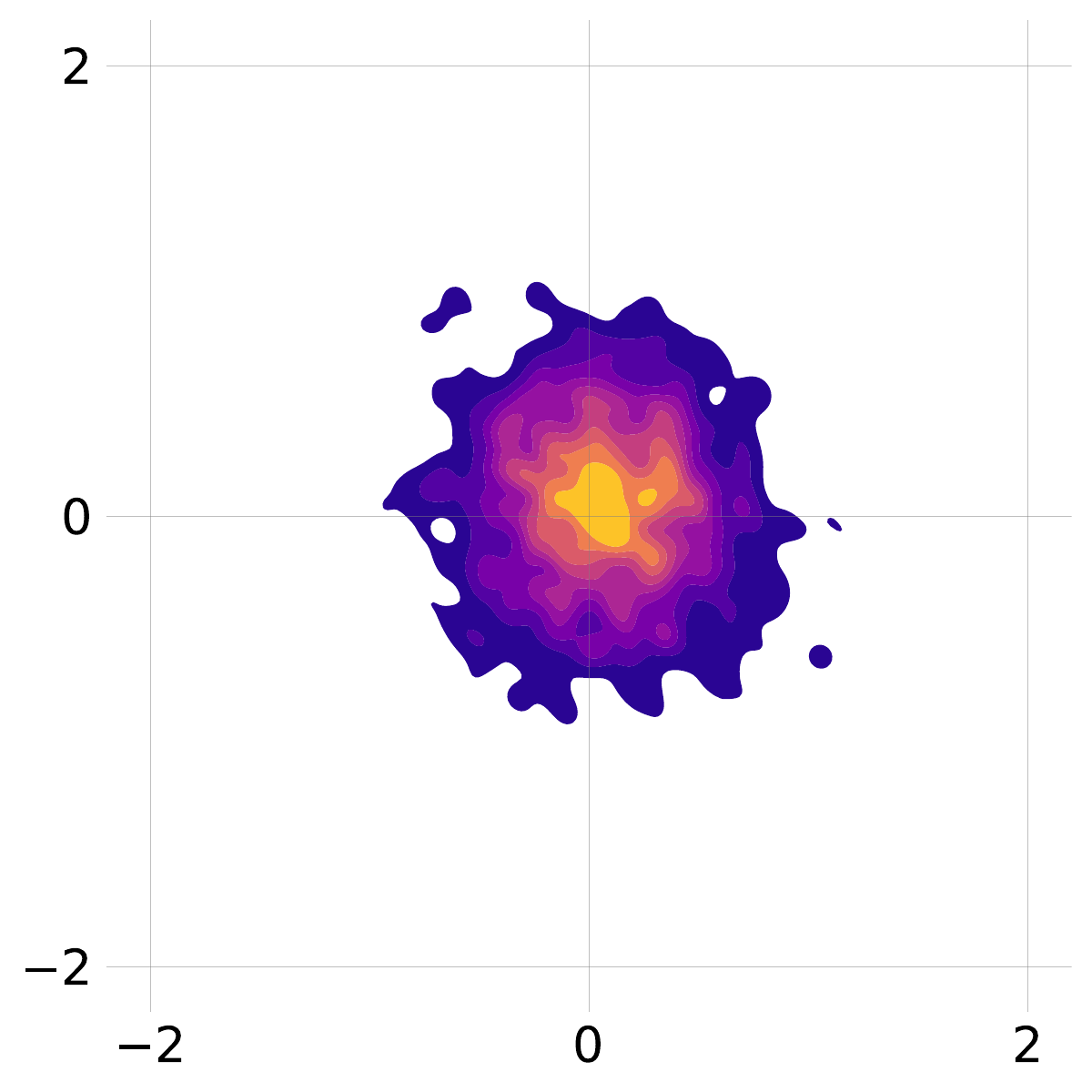} &
        \includegraphics[width=0.18\linewidth]{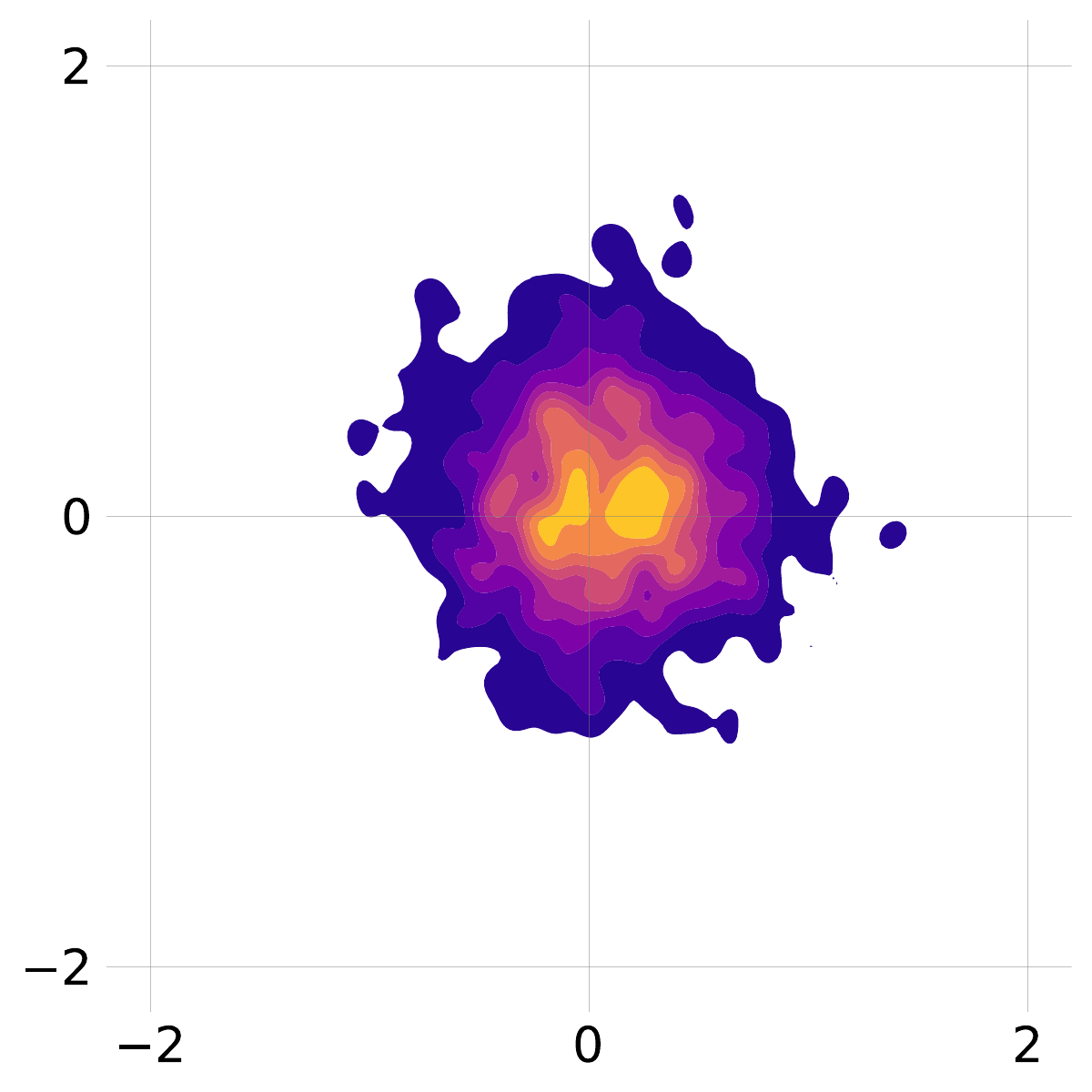} &
        \includegraphics[width=0.18\linewidth]{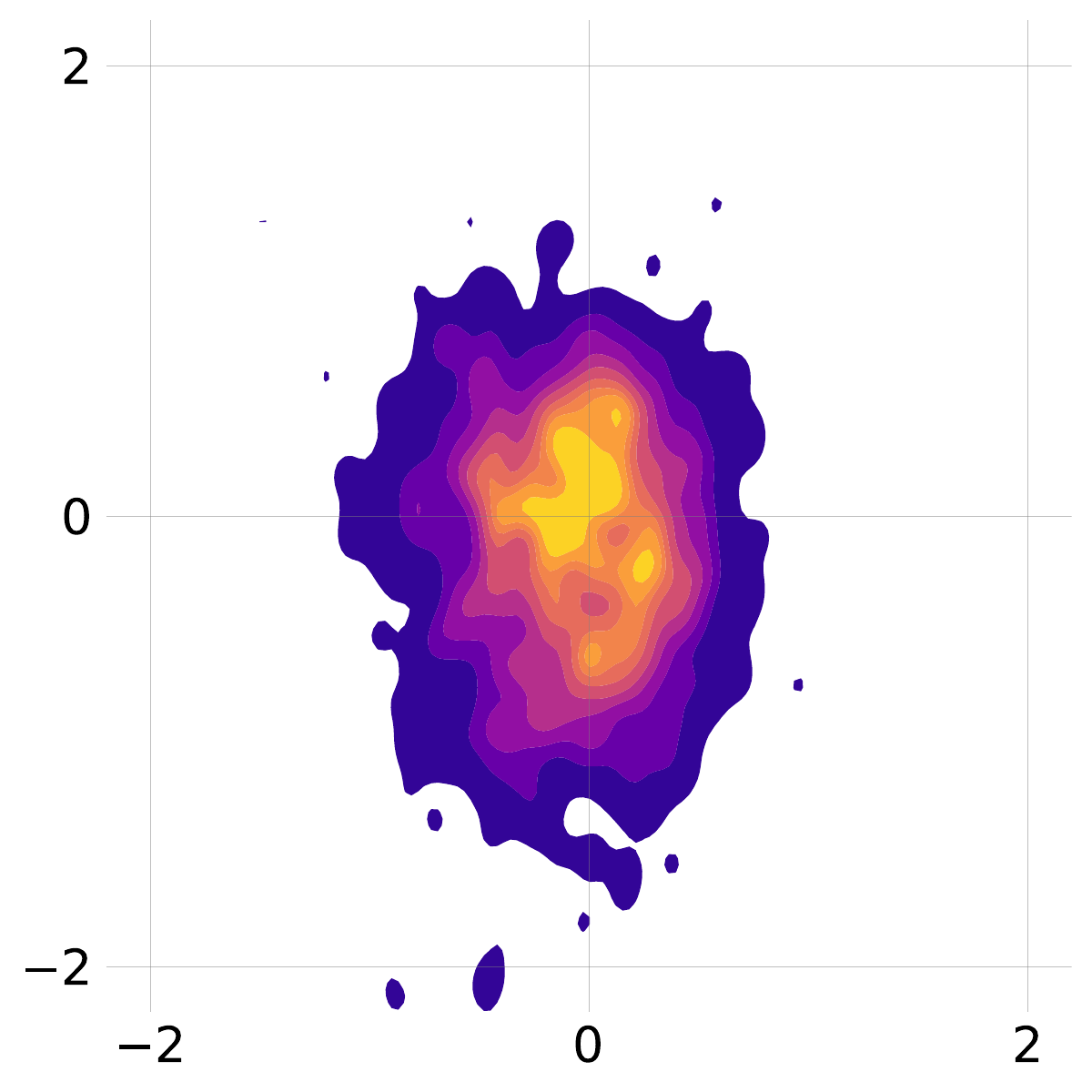} \\

        \rotatebox{90}{\;\;\;\quad\quad  \footnotesize Bias} &
        \includegraphics[width=0.18\linewidth]{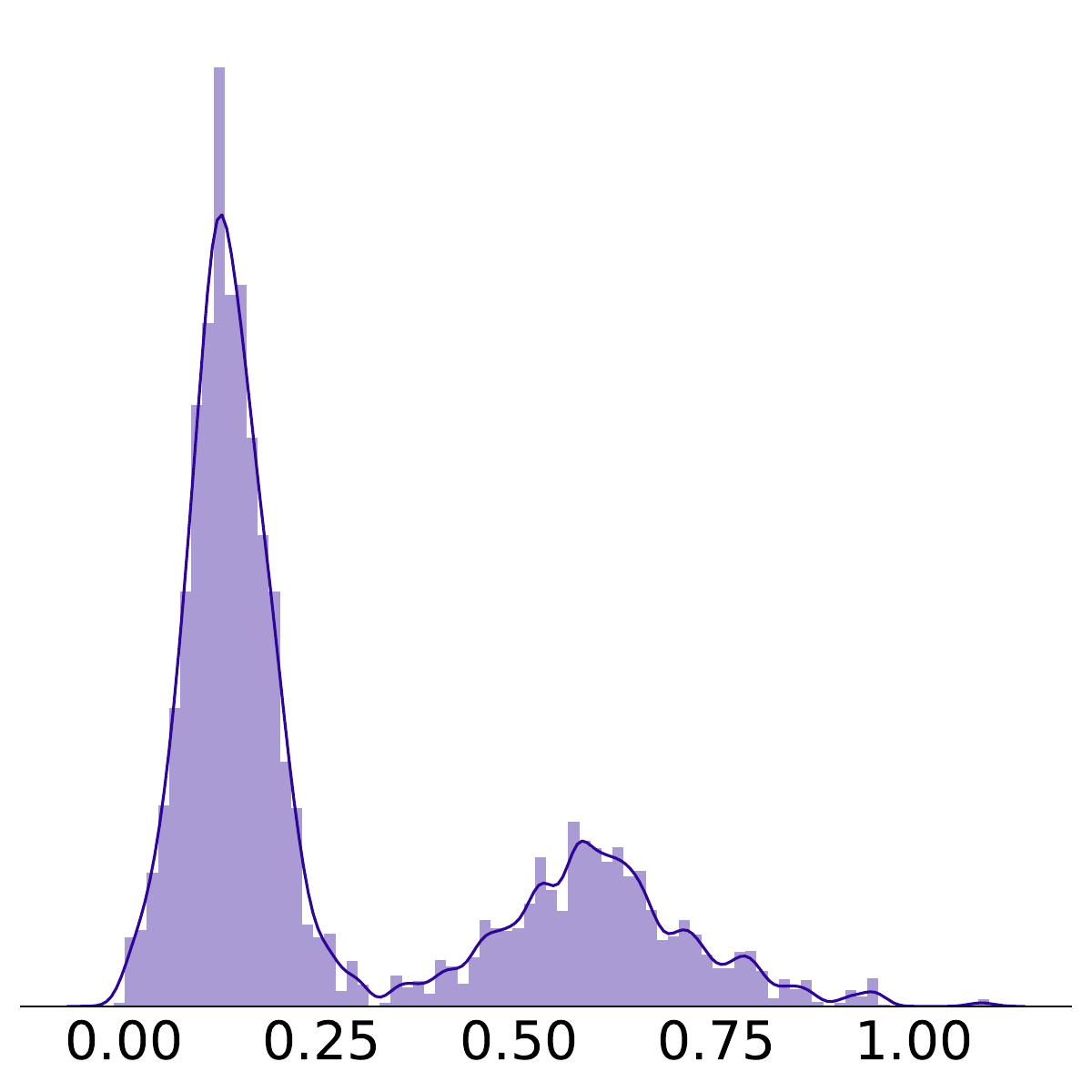} &
        \includegraphics[width=0.18\linewidth]{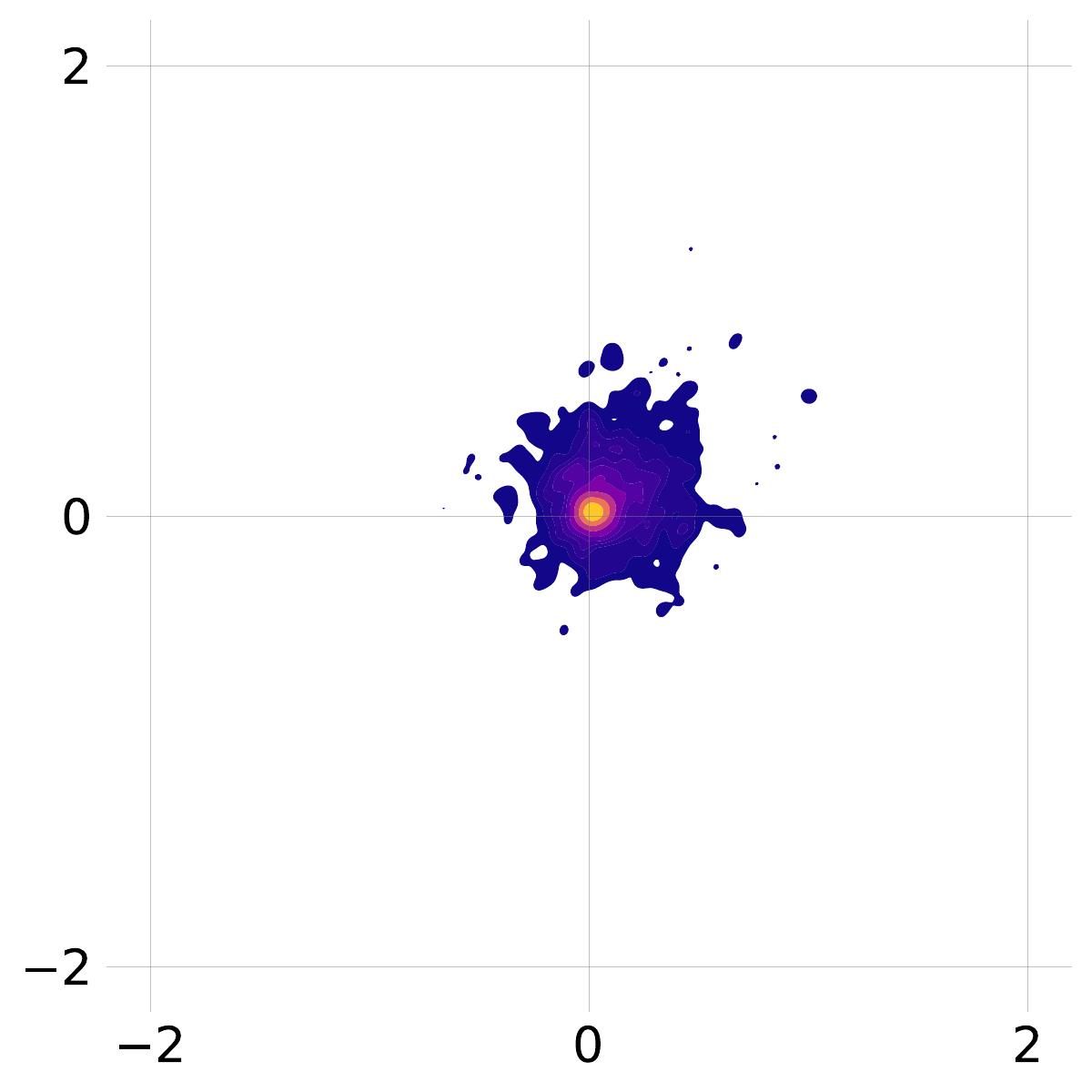} &
        \includegraphics[width=0.18\linewidth]{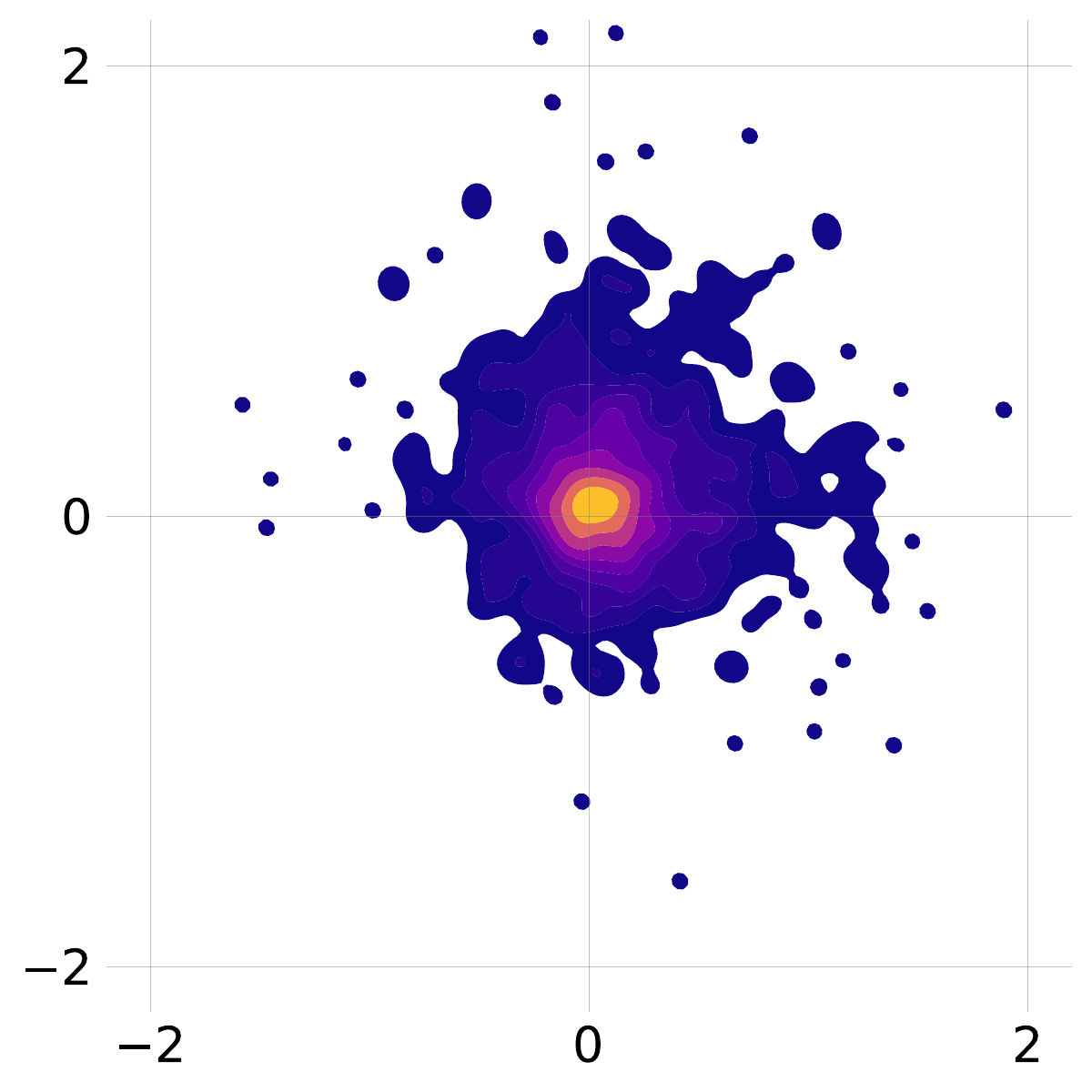} &
        \includegraphics[width=0.18\linewidth]{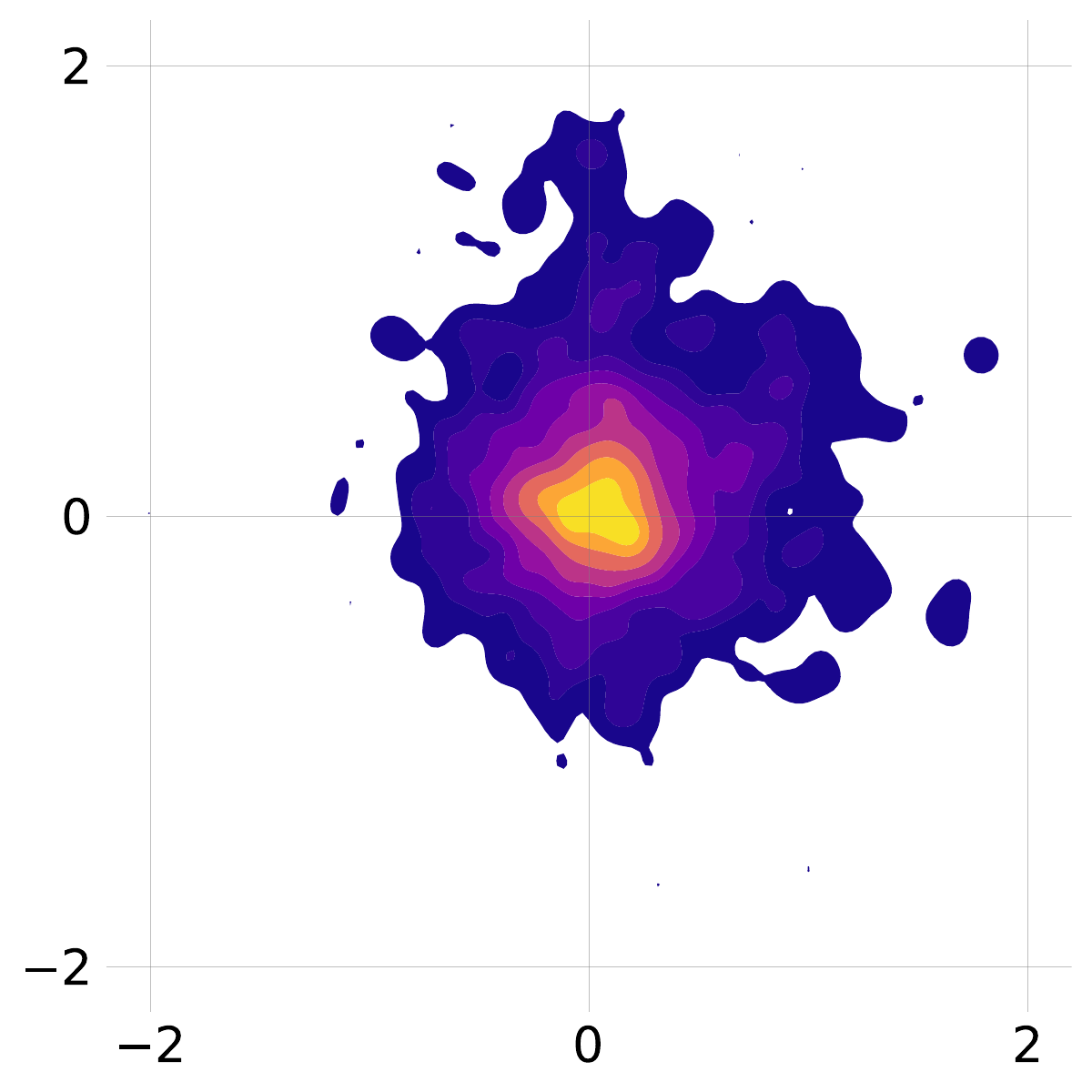} &
        \includegraphics[width=0.18\linewidth]{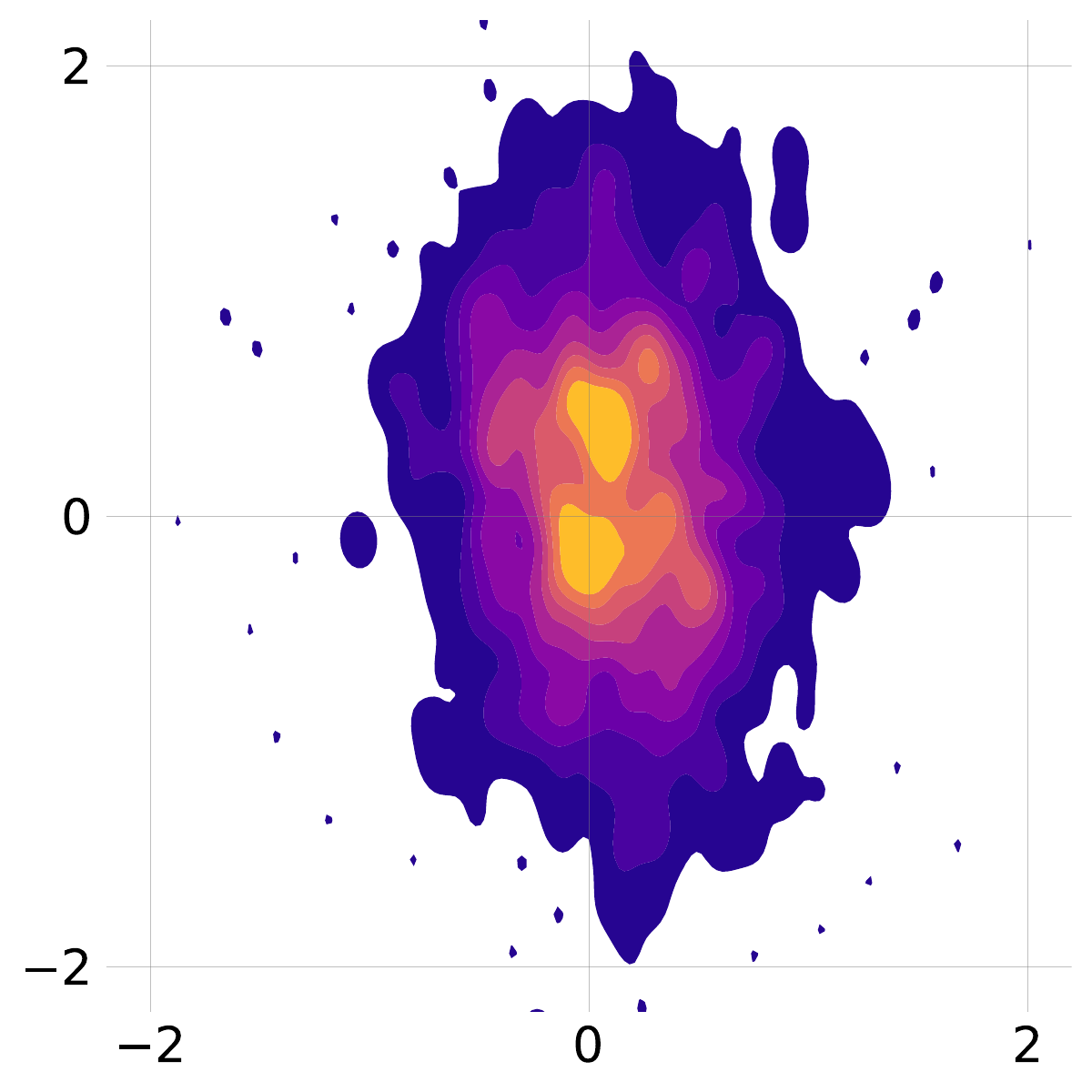} \\
        &  \footnotesize 3x3 Conv Filter &  \footnotesize Early Dense &   \footnotesize Midpoint Dense &   \footnotesize Later Dense &   \footnotesize Last Dense
    \end{tabular}
    \caption{(Bivariate) marginal posterior densities of a small 4-hidden layer convolutional BNN fitted on the Fashion-MNIST dataset (\texttt{CNNv1} of \citealp{MILE}). The grid visualizes the empirical densities of 1M posterior samples obtained from 1k independent chains. The rows and columns (Conv, three hidden, and output weights) display representative densities of randomly chosen weights of the network. We employ scale-adapted SGHMC \citep{adasghmc} for sampling, use an ensemble of 10 with 10k warmup steps, 10k sampling steps, thinning of 10, step size 0.001, momentum decay 0.05, batch size 256, and standard normal isotropic priors.}
    \label{fig:fmnistfireballgrid}
\end{figure}

\begin{figure}[h]
    \centering
    \setlength{\tabcolsep}{0pt}
    \renewcommand{\arraystretch}{0}
    \begin{tabular}{c@{}c@{}c@{}c@{}c@{}c}
        \rotatebox{90}{\;\quad\quad\quad   \footnotesize Kernel} &
        \includegraphics[width=0.2\linewidth]{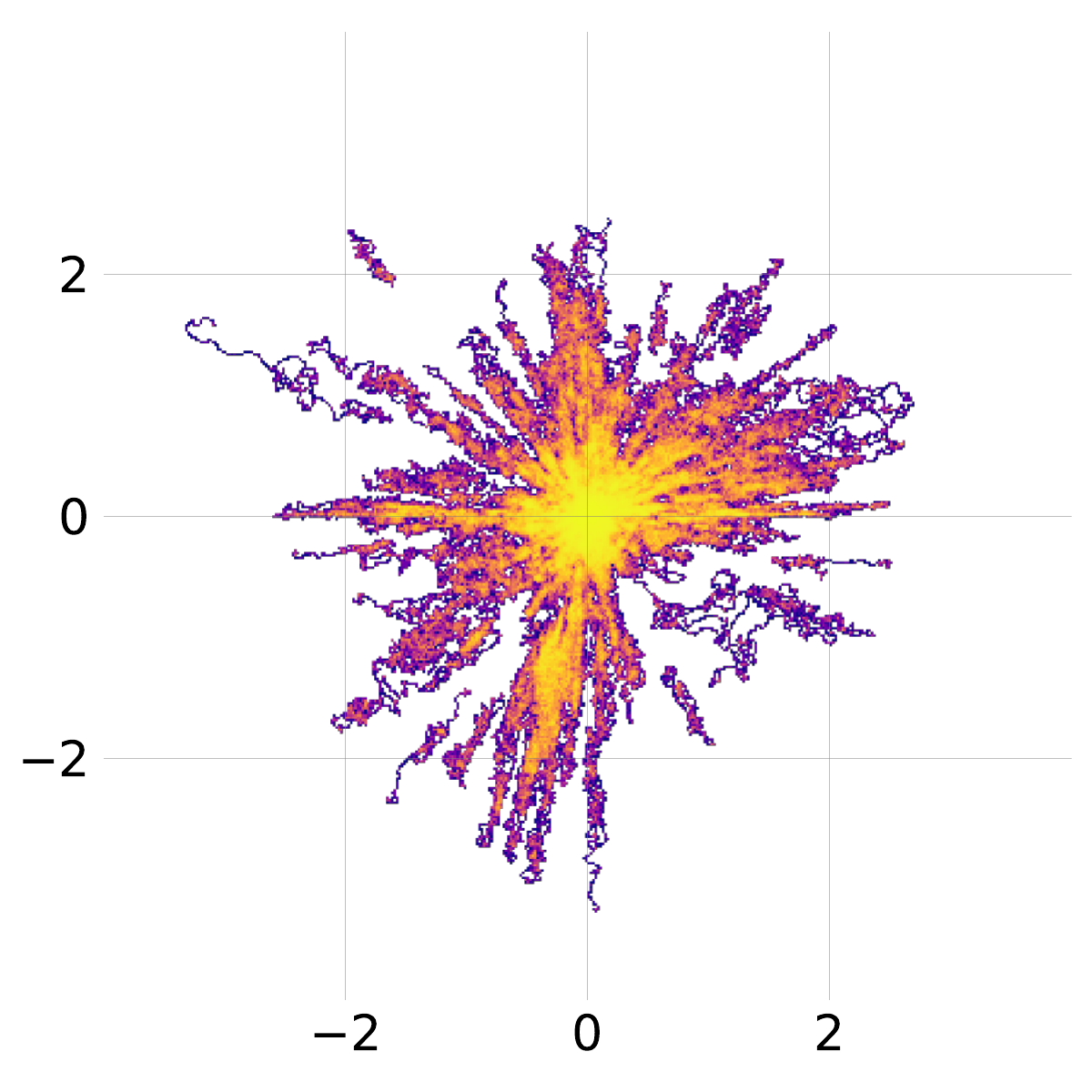} &
        \includegraphics[width=0.2\linewidth]{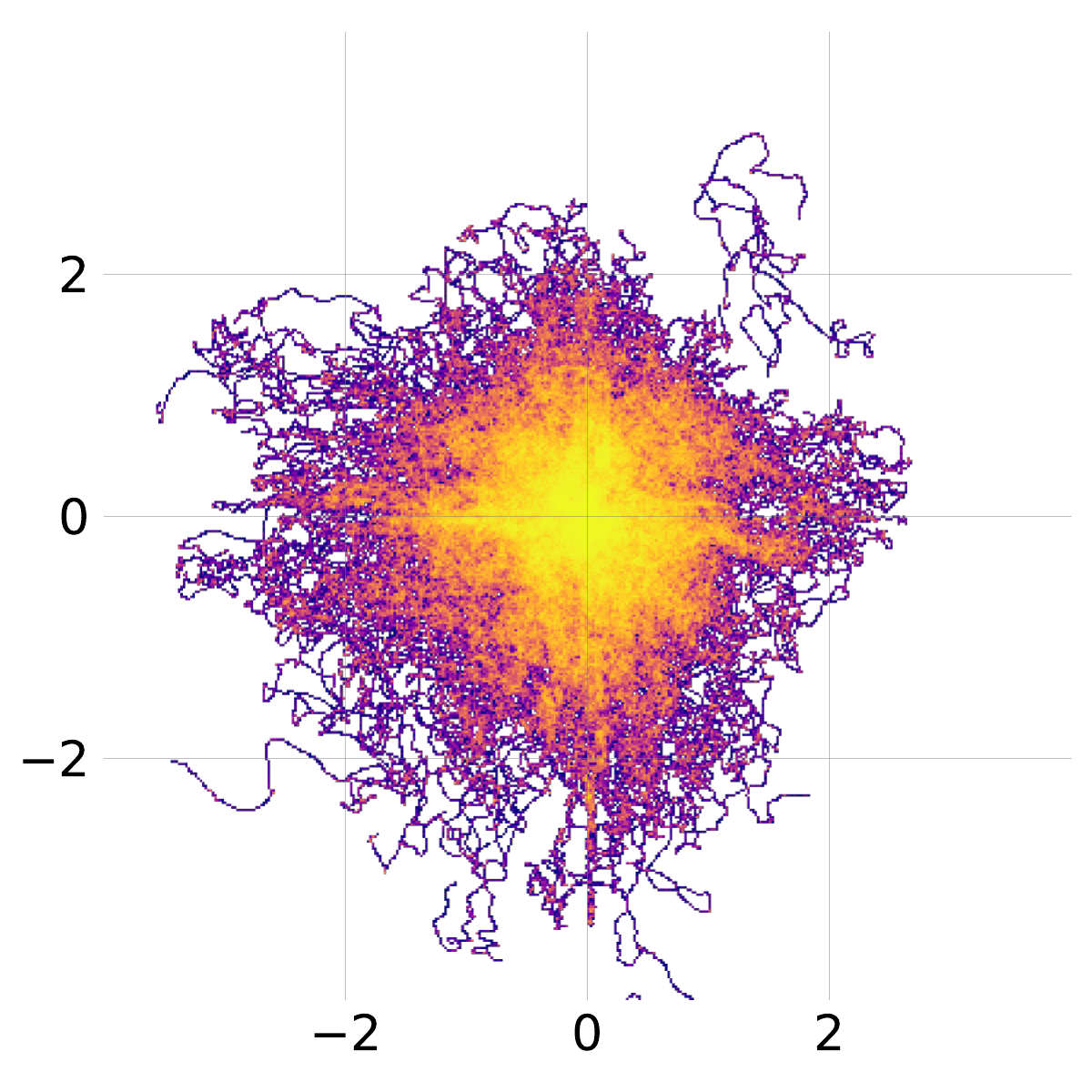} &
        \includegraphics[width=0.2\linewidth]{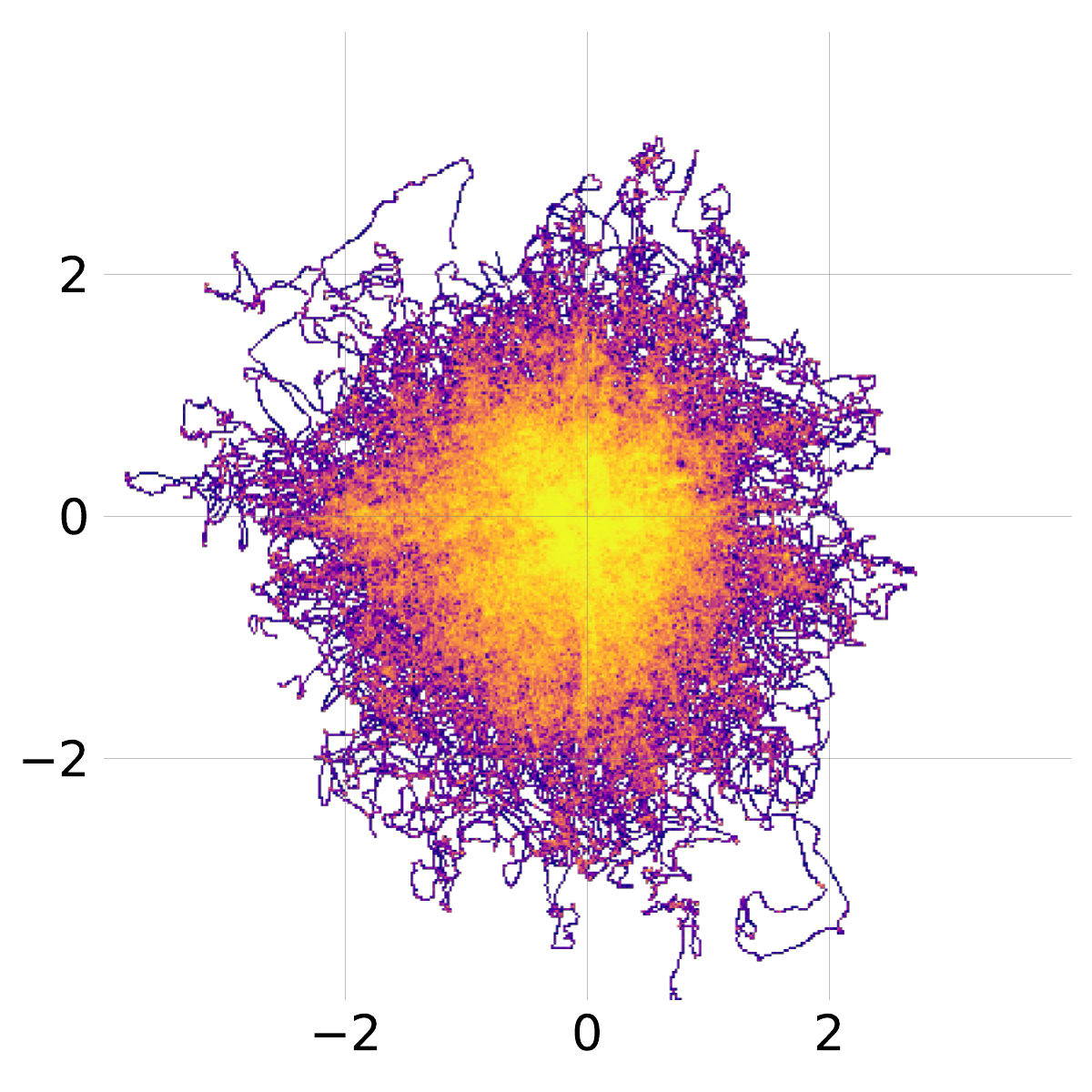} &
        \includegraphics[width=0.2\linewidth]{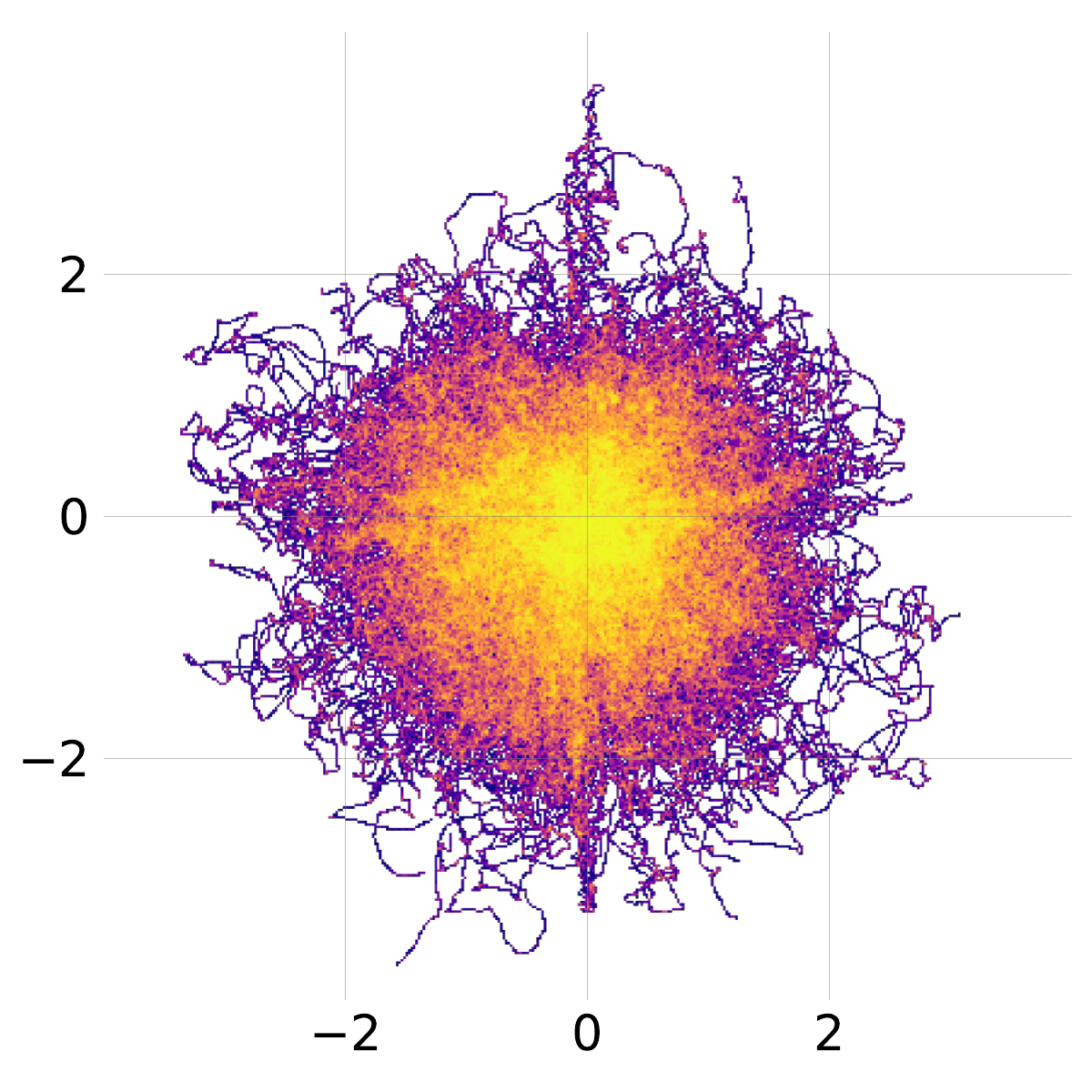} &
        \includegraphics[width=0.2\linewidth]{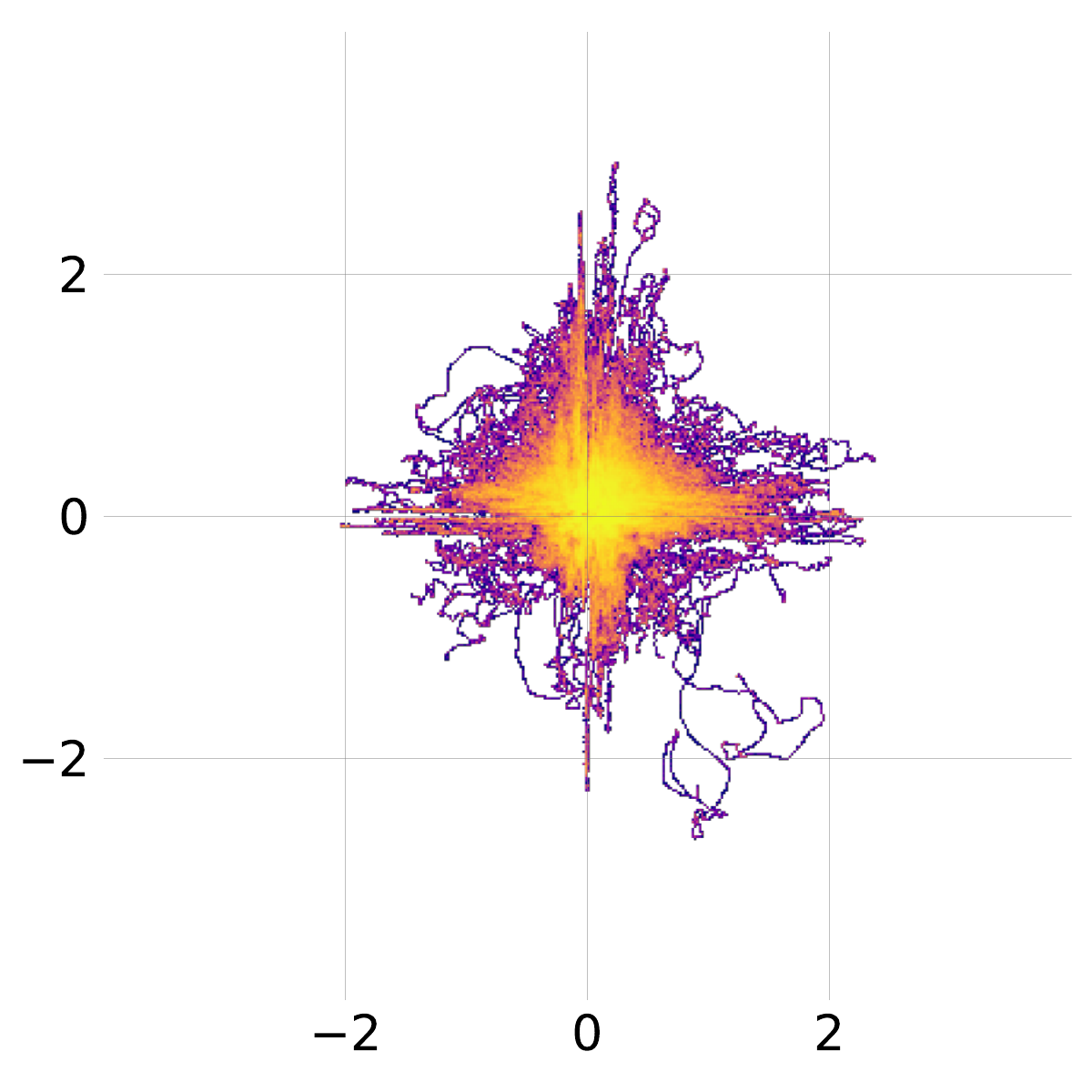} \\

        \rotatebox{90}{\;\quad\quad\quad  \footnotesize Bias} &
        \includegraphics[width=0.2\linewidth]{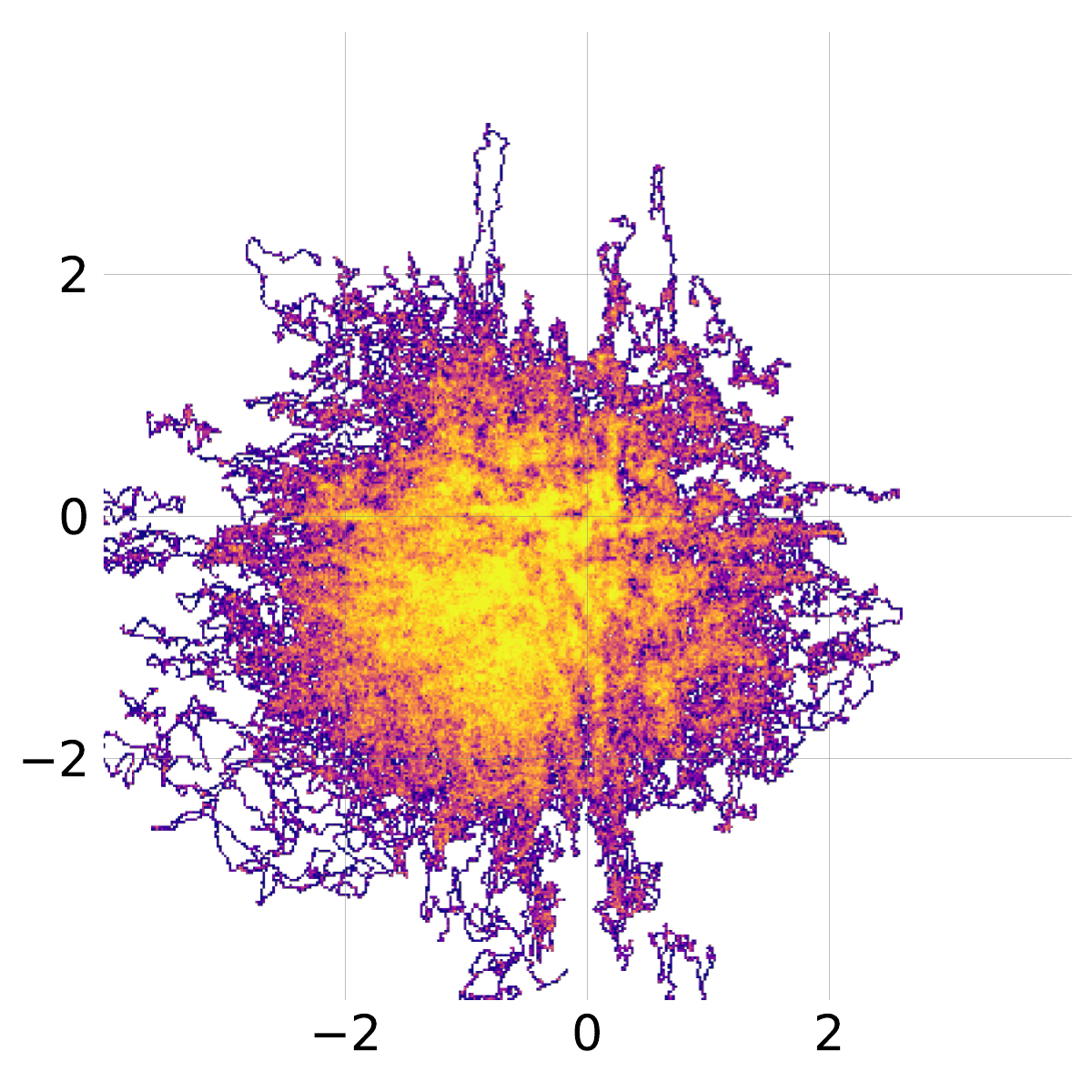} &
        \includegraphics[width=0.2\linewidth]{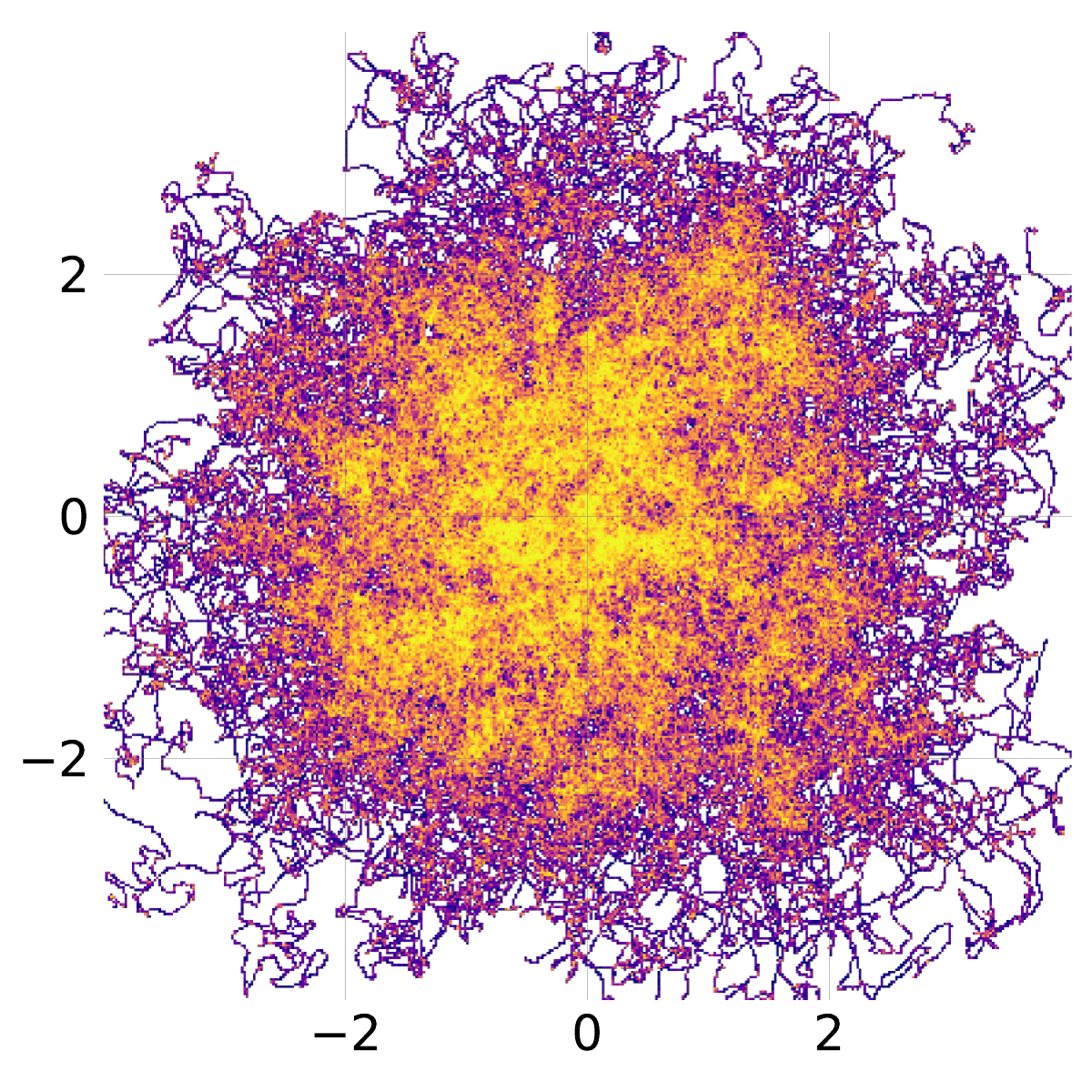} &
        \includegraphics[width=0.2\linewidth]{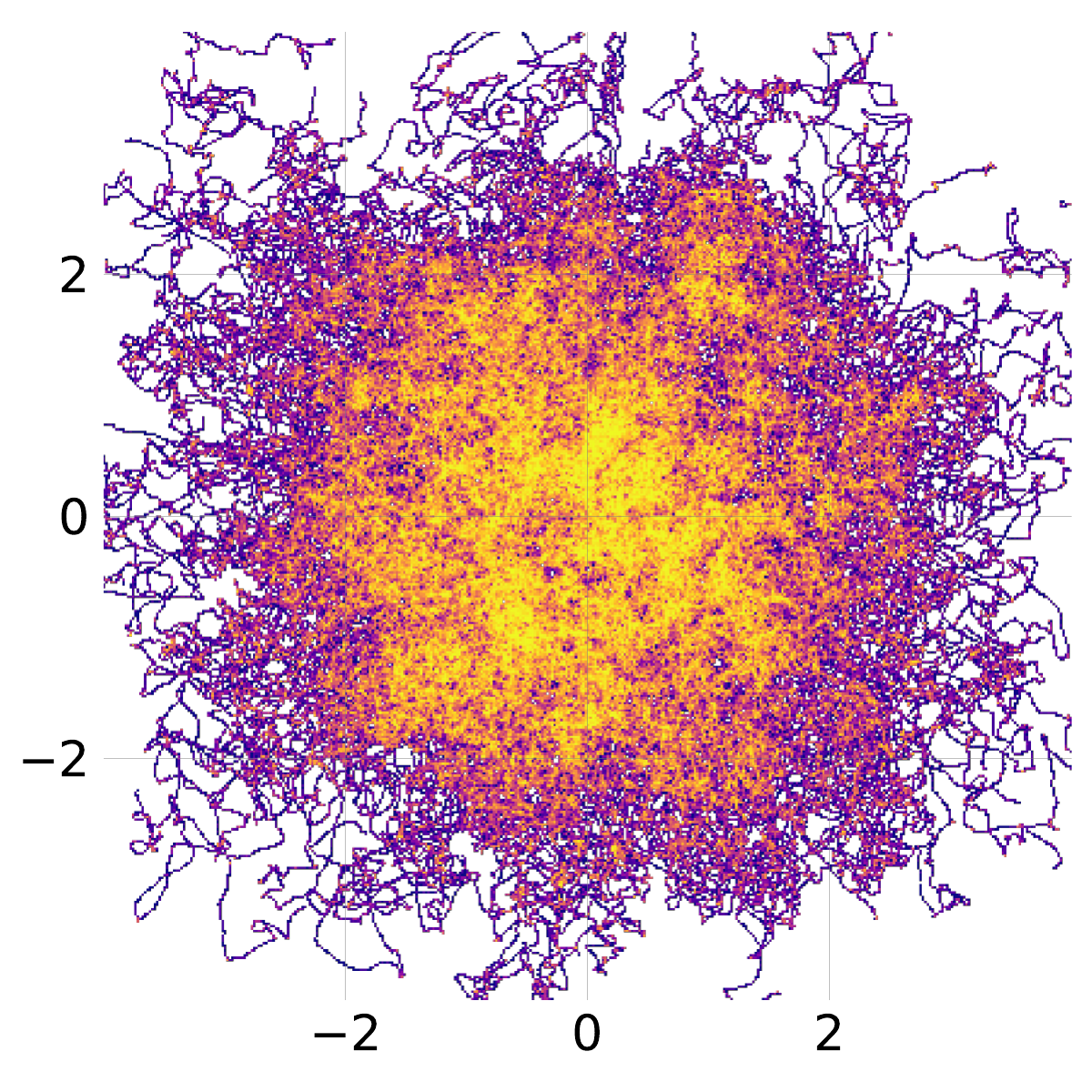} &
        \includegraphics[width=0.2\linewidth]{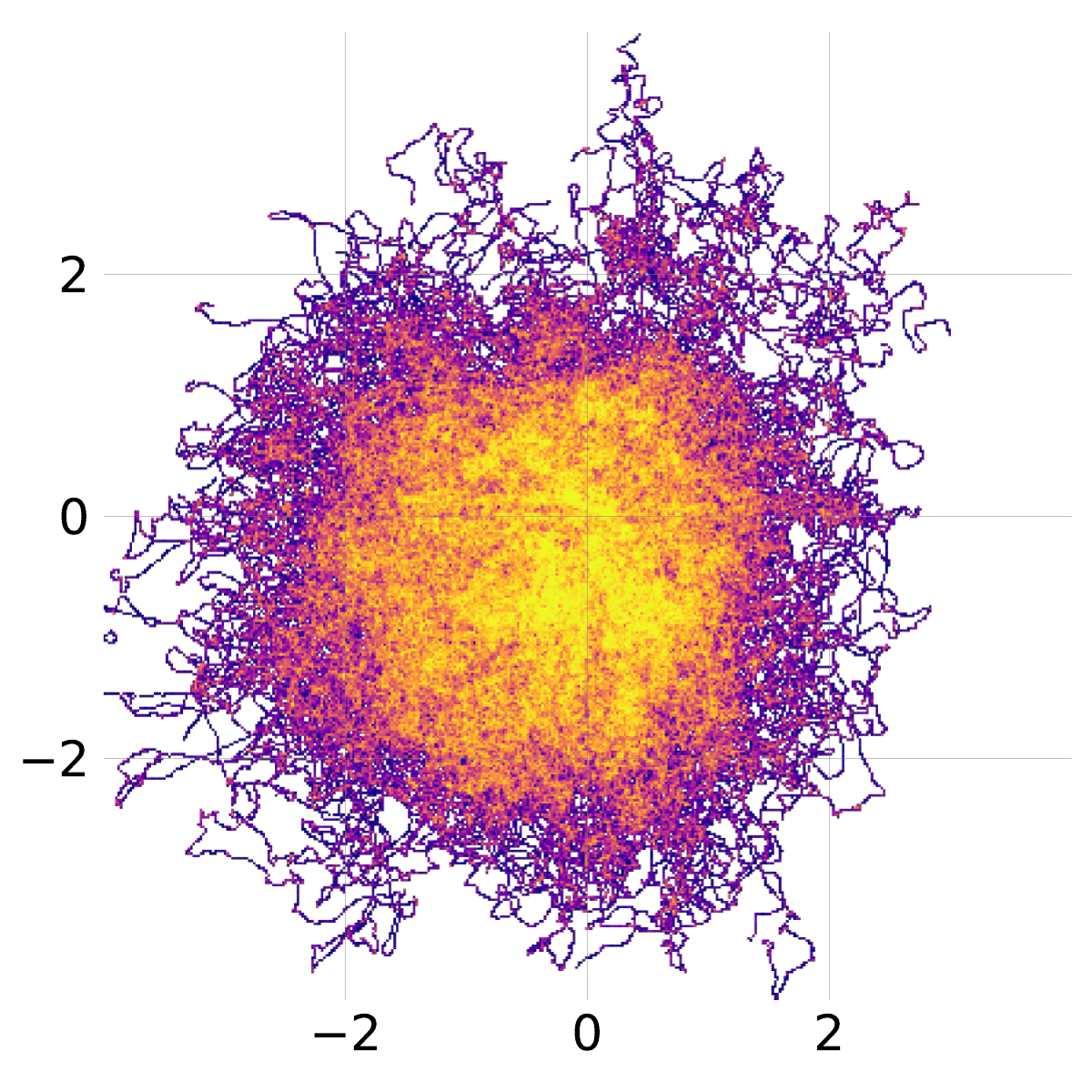} &
        \includegraphics[width=0.2\linewidth]{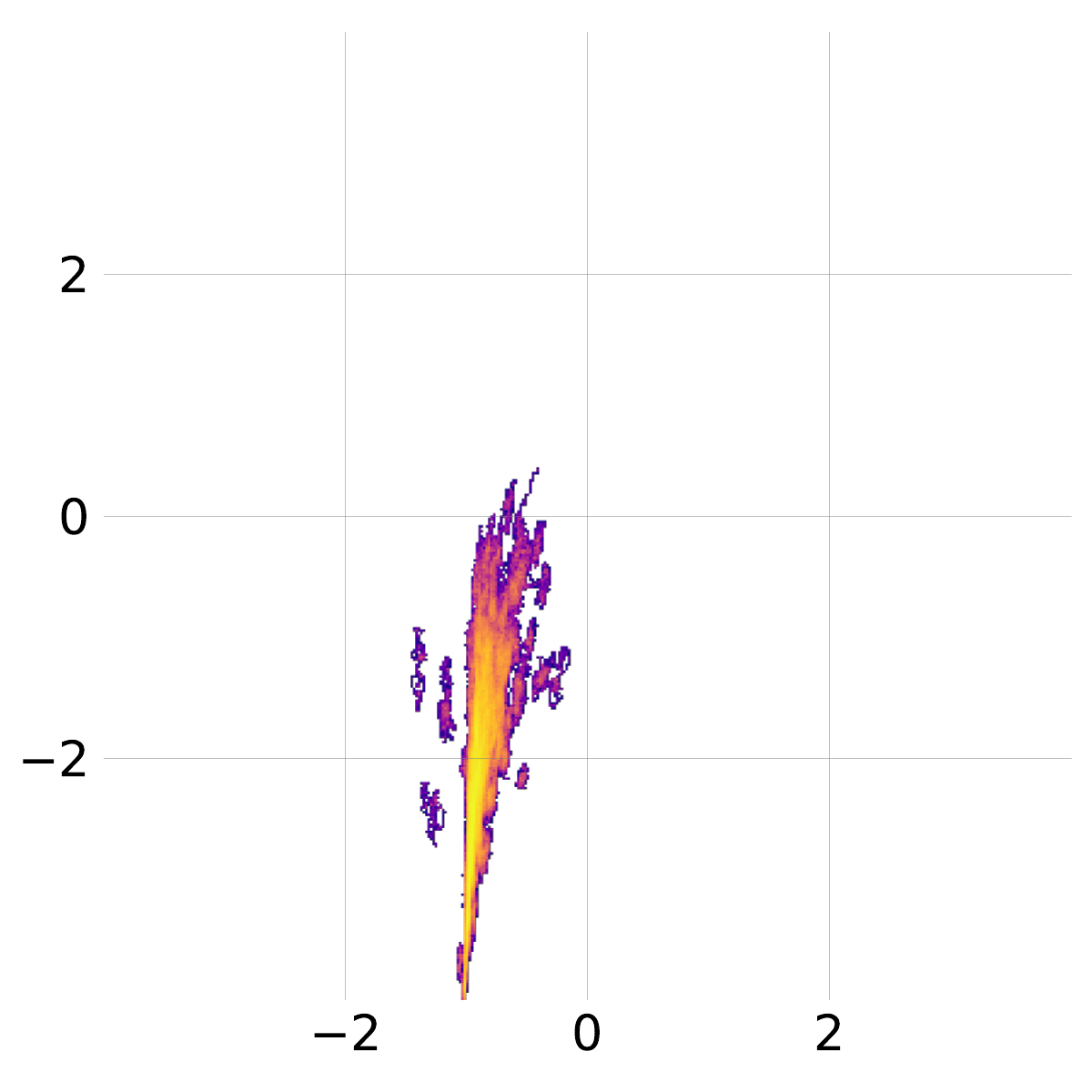} \\
        &   \footnotesize First &  \footnotesize Early Hidden &   \footnotesize Midpoint Hidden &   \footnotesize Later Hidden &   \footnotesize Last
    \end{tabular}
    \caption{(Bivariate) marginal posterior densities of a 4-hidden layer fully-connected BNN fitted on the \texttt{bikesharing} dataset (regression task, and same architecture as in \cref{fig:maingrid}). The grid visualizes the empirical densities of 1M posterior samples obtained from 1k independent chains. The rows and columns (input, three hidden, and output weights) display representative densities of randomly chosen weights of the network.}
    \label{fig:BIKEfireballgrid}
\end{figure}

\begin{figure}[h]
    \centering
    \setlength{\tabcolsep}{0pt}
    \renewcommand{\arraystretch}{0}
    \begin{tabular}{c@{}c@{}c@{}c@{}c@{}c}
        \rotatebox{90}{\;\;\quad\quad\quad   \footnotesize Kernel} &
        \includegraphics[width=0.2\linewidth]{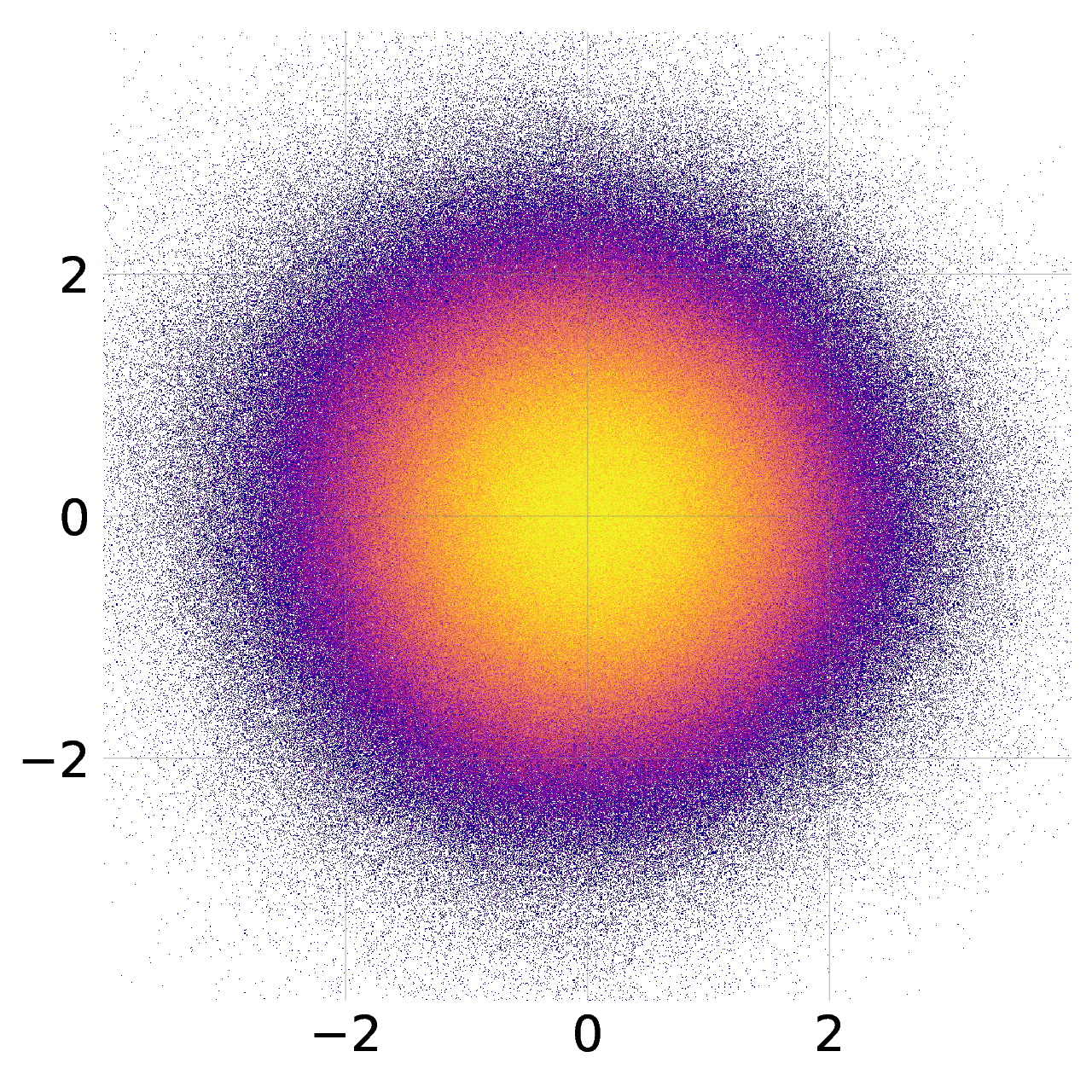} &
        \includegraphics[width=0.2\linewidth]{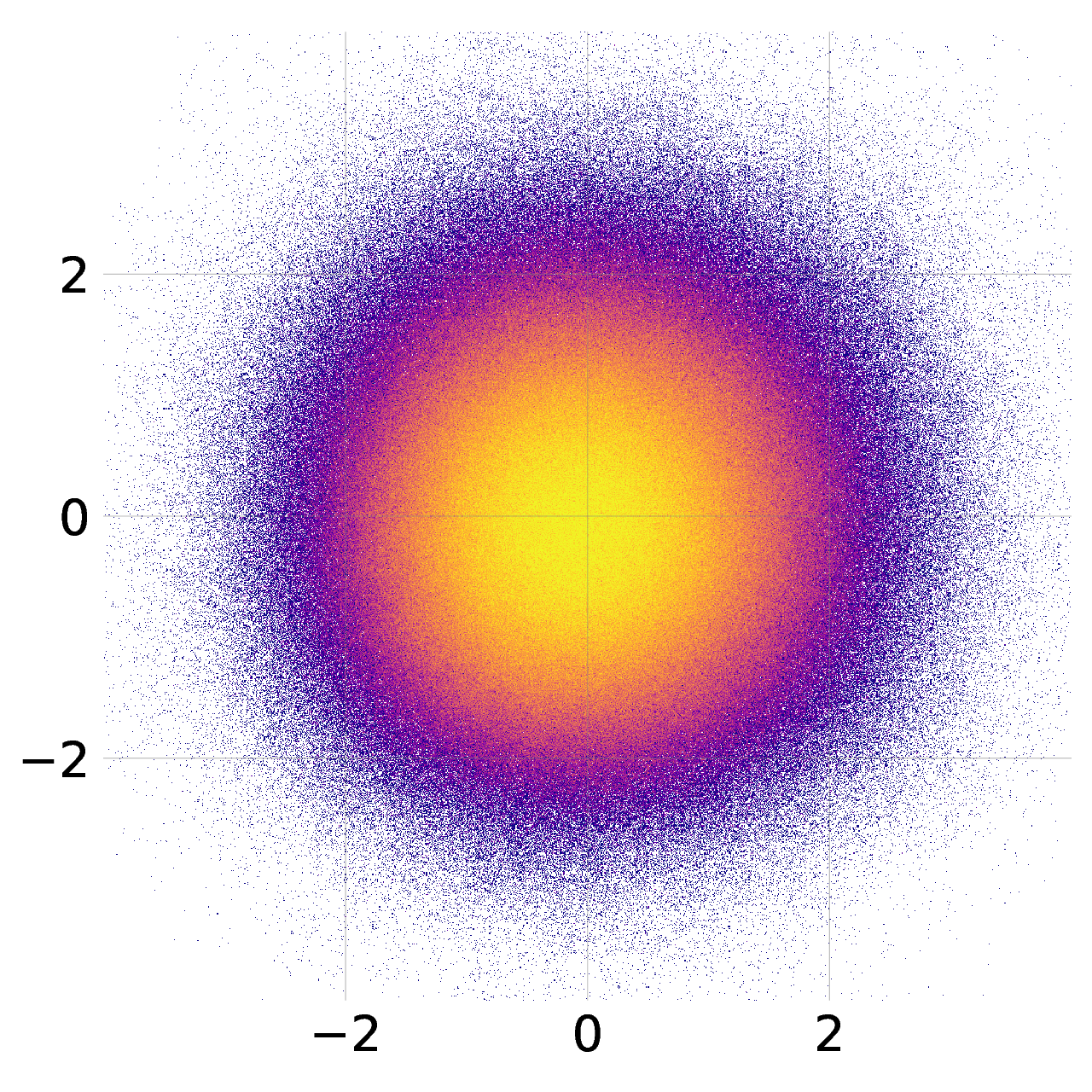} &
        \includegraphics[width=0.2\linewidth]{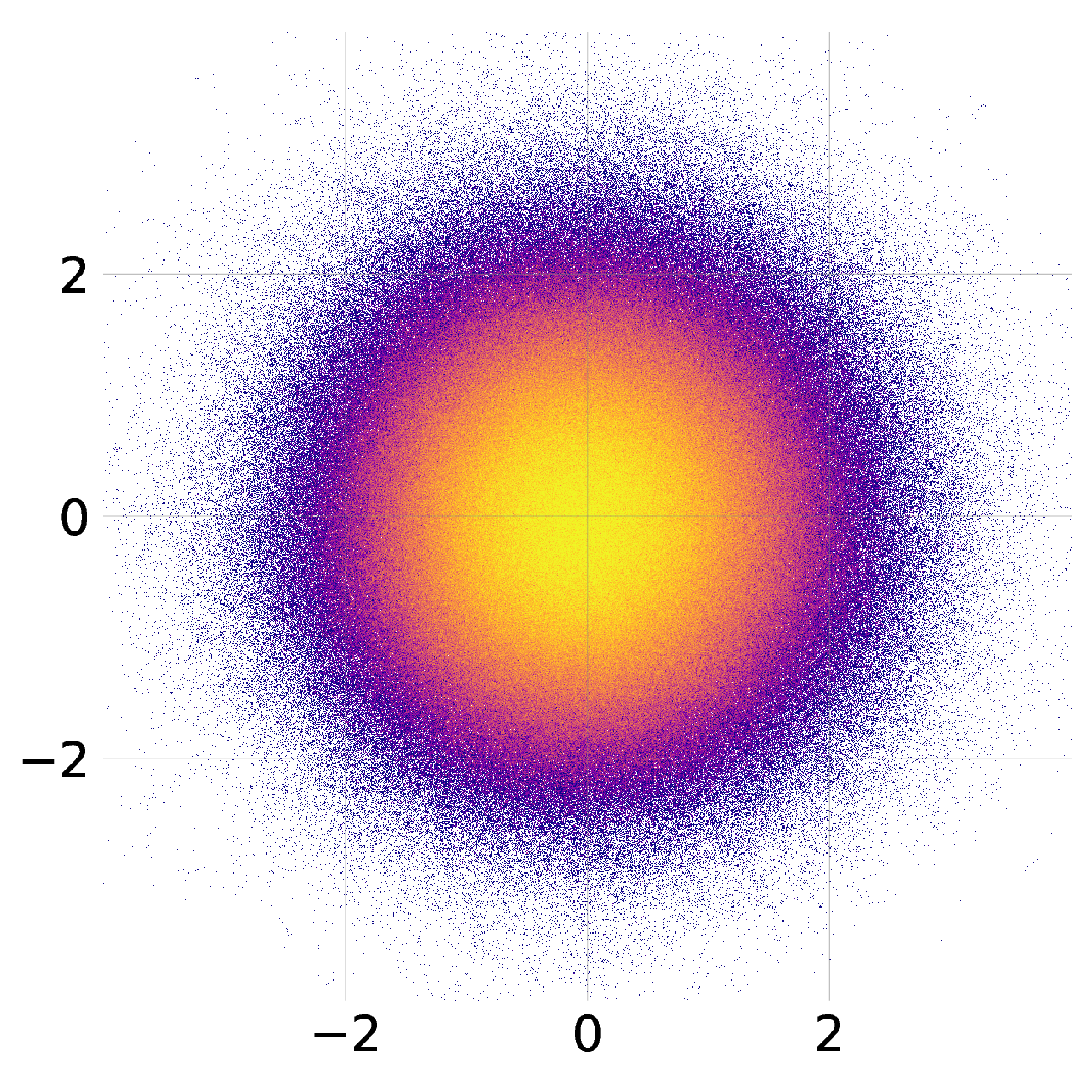} &
        \includegraphics[width=0.2\linewidth]{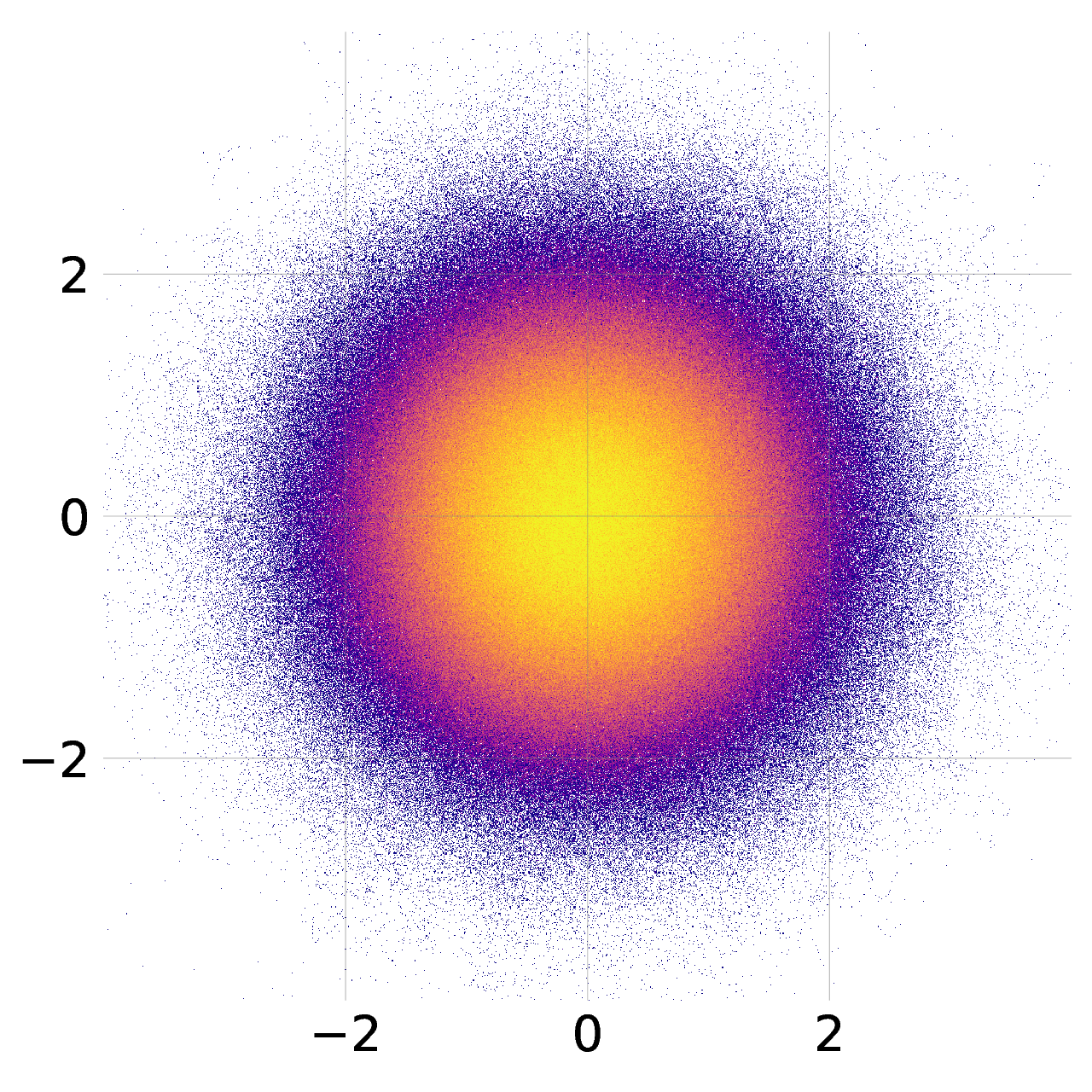} &
        \includegraphics[width=0.2\linewidth]{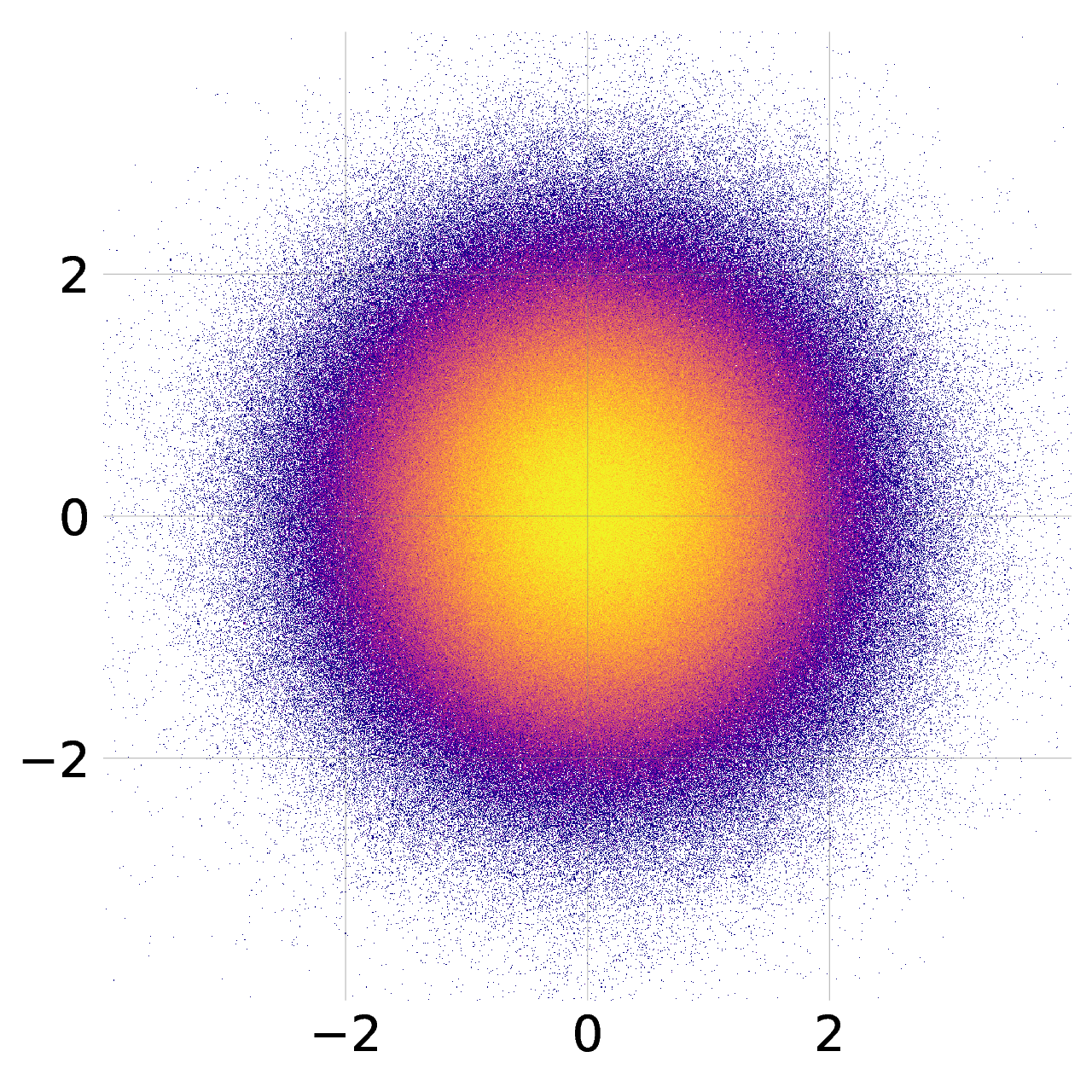} \\

        \rotatebox{90}{\;\;\quad\quad\quad  \footnotesize Bias} &
        \includegraphics[width=0.2\linewidth]{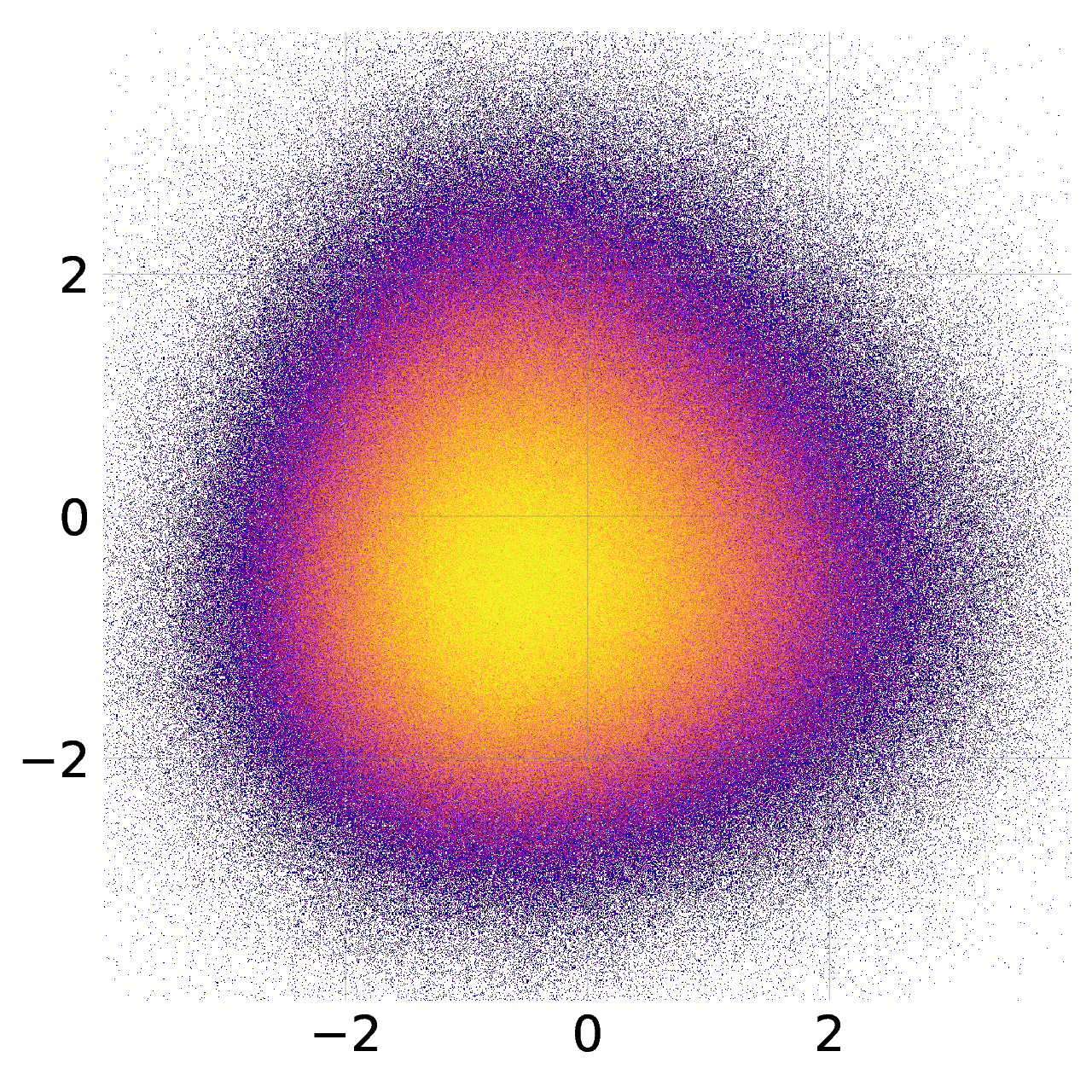} &
        \includegraphics[width=0.2\linewidth]{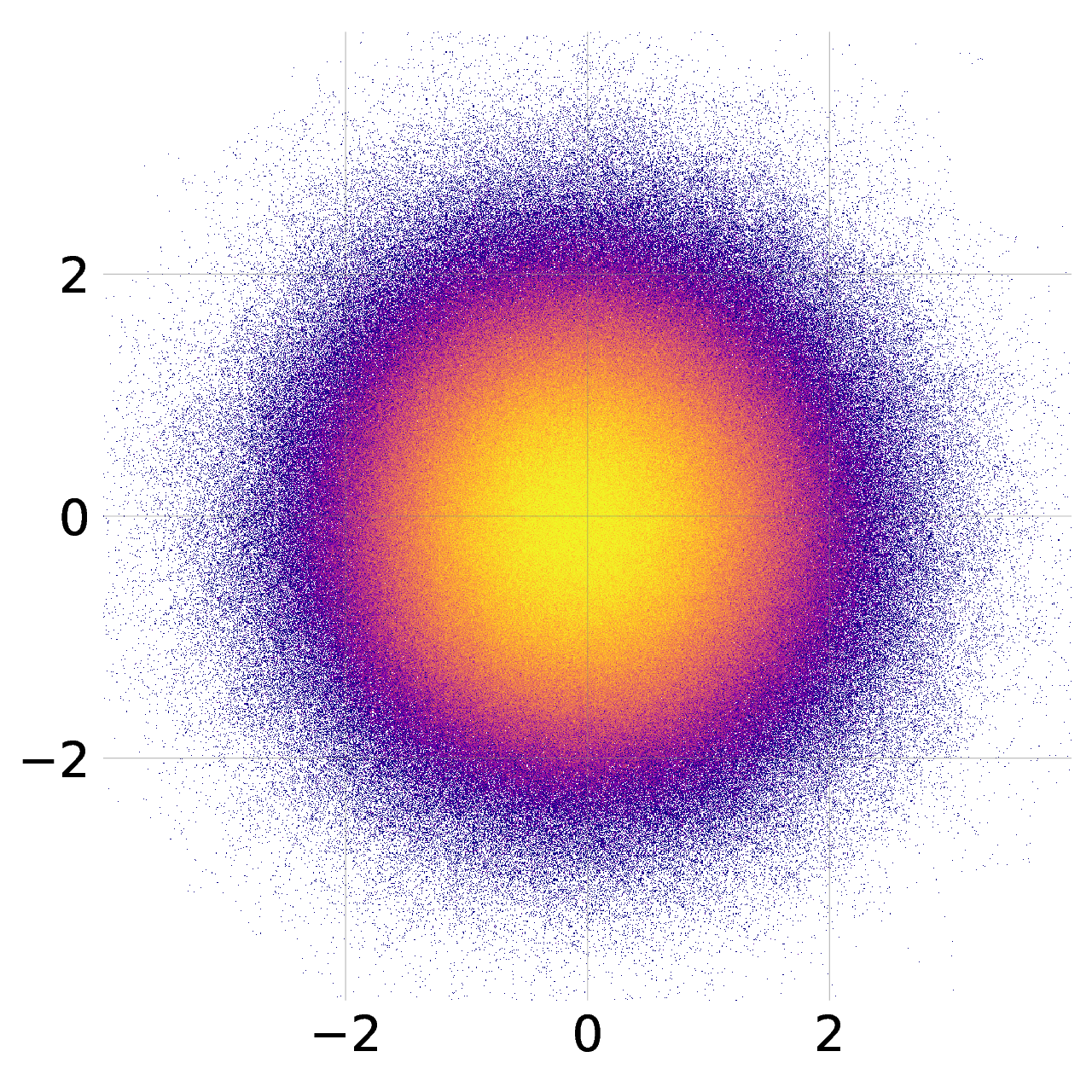} &
        \includegraphics[width=0.2\linewidth]{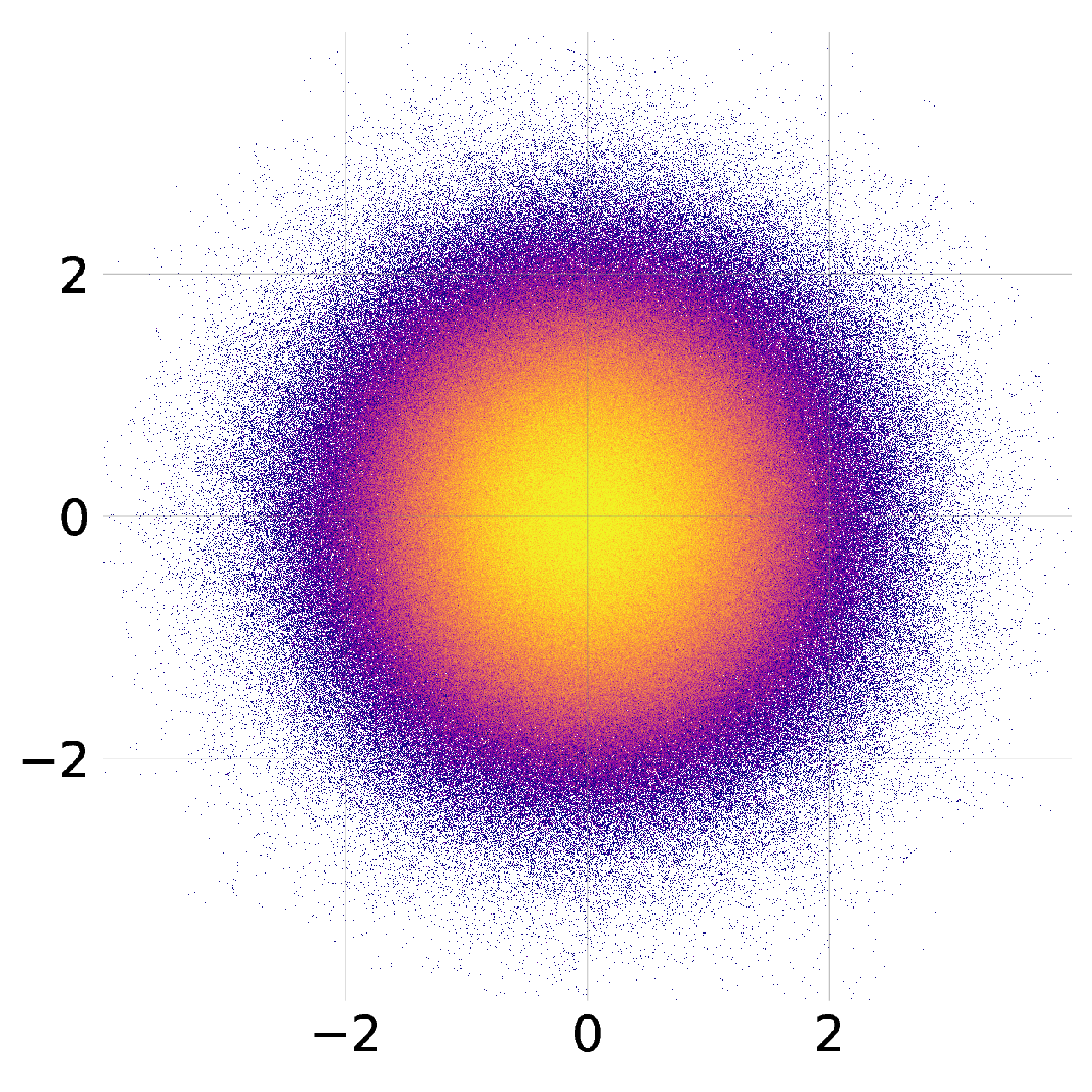} &
        \includegraphics[width=0.2\linewidth]{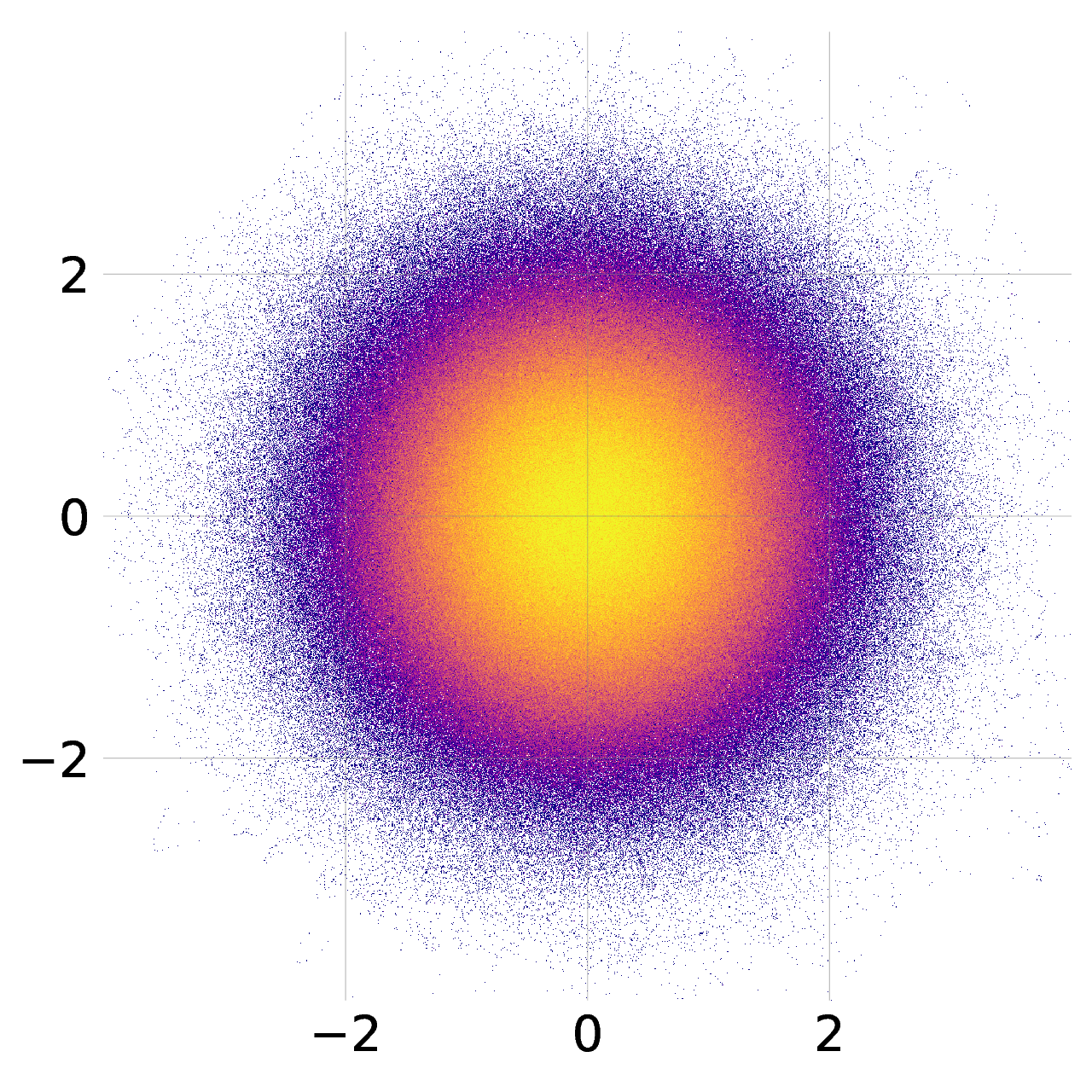} &
        \includegraphics[width=0.2\linewidth]{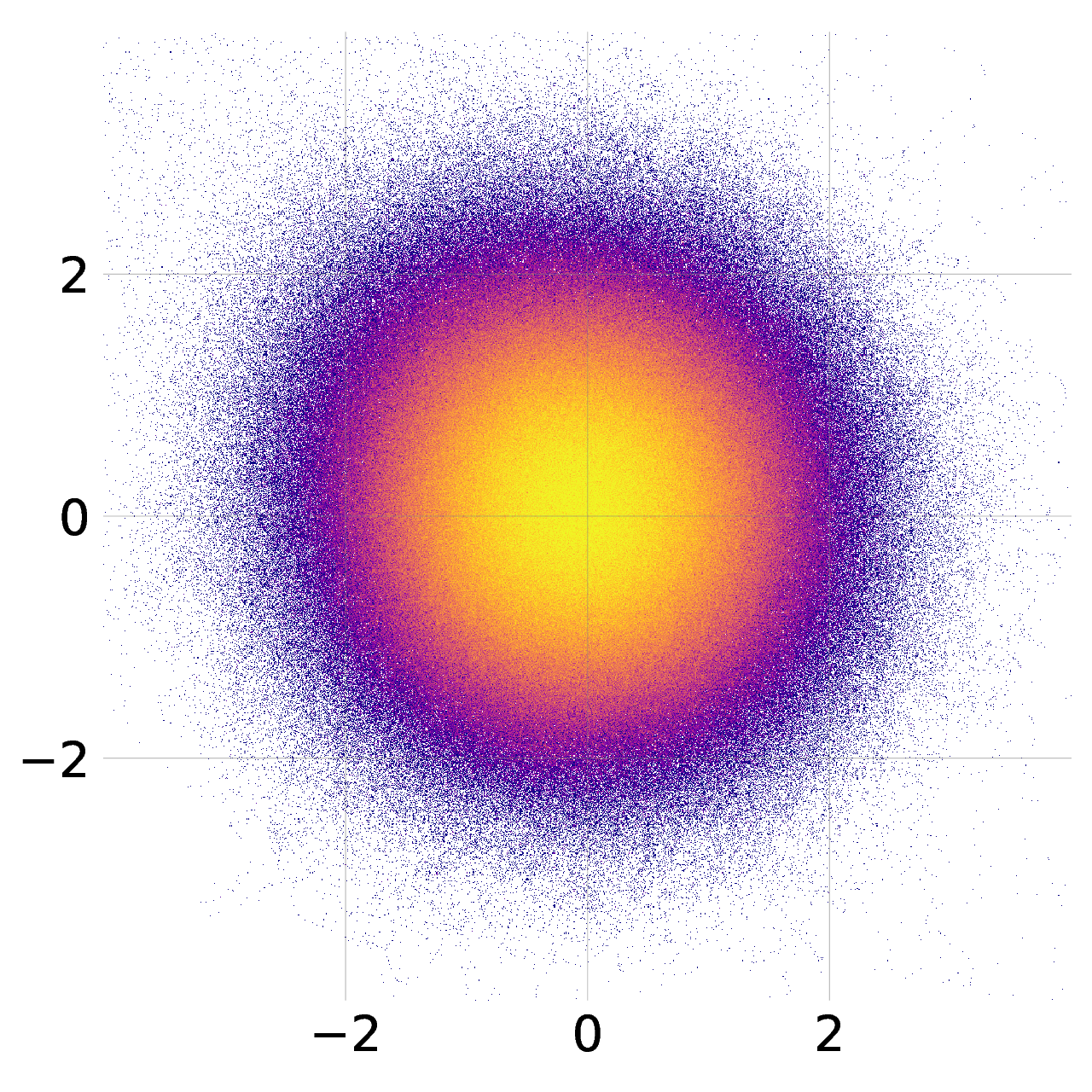} \\
        &   \footnotesize First & \footnotesize  Early Hidden &   \footnotesize Midpoint Hidden &   \footnotesize Later Hidden &   \footnotesize Last
    \end{tabular}
    \caption{(Bivariate) marginal posterior densities of a 4-hidden layer fully-connected BNN fitted on the \texttt{ionosphere} dataset (binary classification task, and same architecture as in \cref{fig:maingrid}). The grid visualizes the empirical densities of 8M posterior samples obtained from 8k independent chains. The rows and columns (input, three hidden, and output weights) display representative densities of randomly chosen weights of the network.}
    \label{fig:ionospherefireballgrid}
\end{figure}

\paragraph{Underparametrized Model} In order to contrast the bivariate marginal posterior densities of the above-considered overparametrized models with previously analyzed underparametrized models, we also considered the small $f_1$ architecture of \citet{wiese2023towards} and just as in their work, fitted it on the \texttt{airfoil} dataset, but now with 8M posterior samples and the MILE sampler. \cref{fig:airunderfireballgrid} displays the obtained marginal densities and shows clearly that in the underparametrized setting, very distinct symmetry patterns emerge in the margins of the sampled posterior. One of these marginal plots is also featured in \cref{fig:rohrschach}.

\begin{figure}[h!]
    \centering
    \setlength{\tabcolsep}{0pt}
    \renewcommand{\arraystretch}{0}
    \begin{tabular}{c@{}c@{}c@{}c@{}c@{}c}
        \rotatebox{90}{\;\;\;\quad\quad\quad\quad\quad\quad   \footnotesize Kernel} &
        \includegraphics[width=0.3\linewidth]{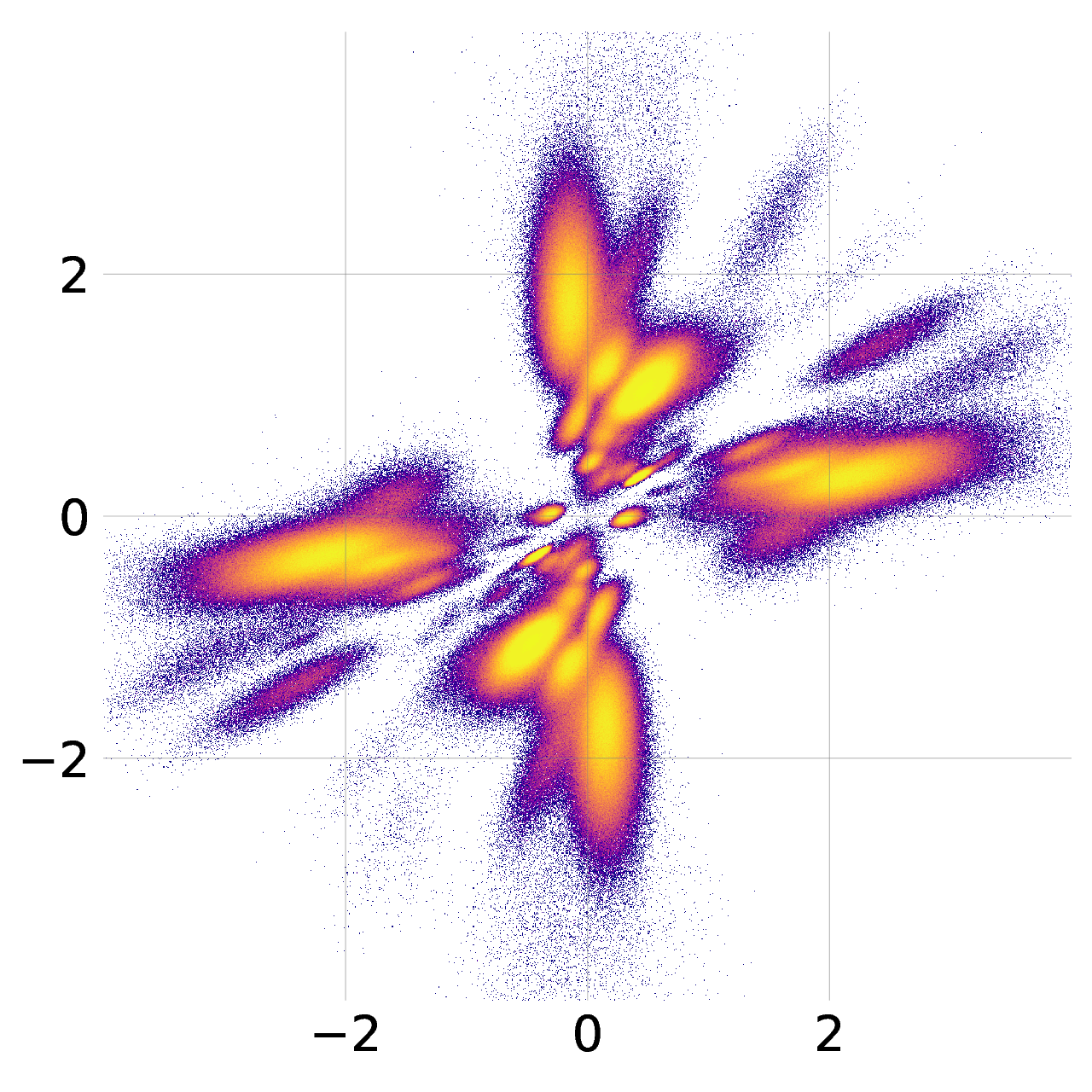} &
        \includegraphics[width=0.3\linewidth]{imgs/fireballs/air_under_full_mile_layer1_kernel_med_res.png} \\

        \rotatebox{90}{\;\;\;\quad\quad\quad\quad\quad\quad  \footnotesize  Bias} &
        \includegraphics[width=0.3\linewidth]{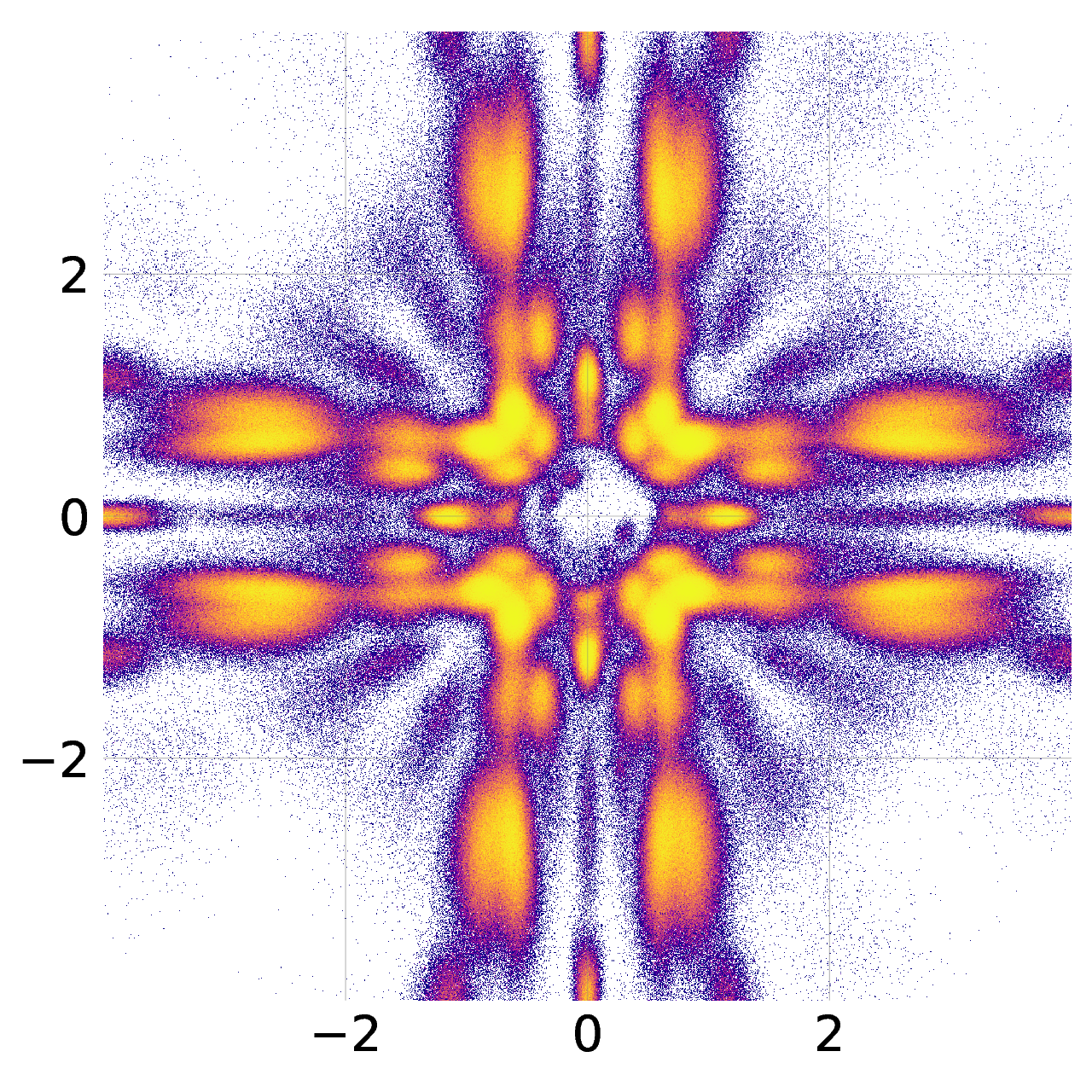} &
        \includegraphics[width=0.3\linewidth]{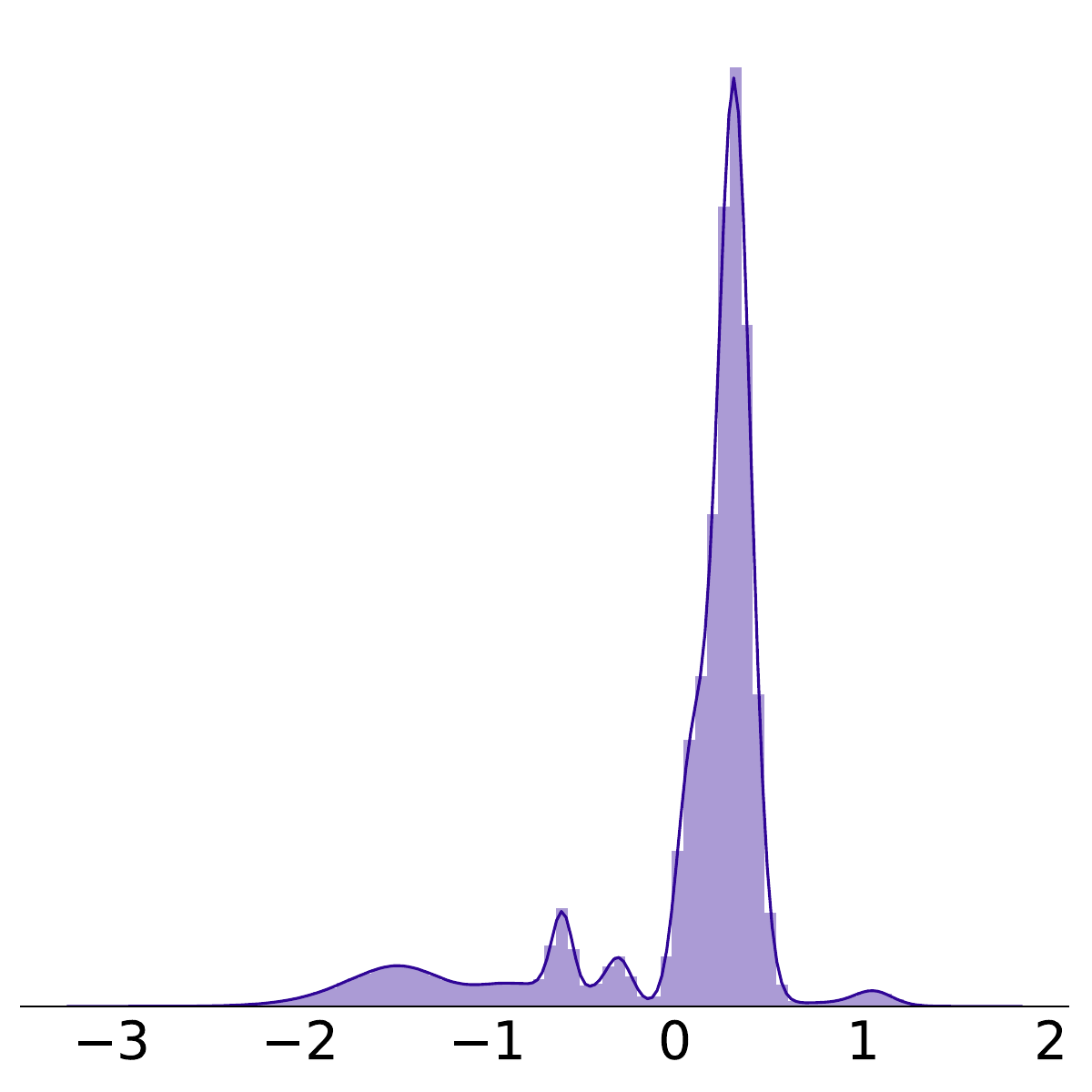} \\
        &  \footnotesize  Input to Hidden &   \footnotesize Hidden to Output
    \end{tabular}
    \caption{(Bivariate) marginal posterior densities of the underparametrized one-hidden layer fully-connected BNN fitted on the \texttt{airfoil} dataset (following the small $f_1$ architecture of \citet{wiese2023towards}). The grid visualizes the empirical densities of 8M posterior samples obtained from 8k independent chains. The rows and columns (input and output weights) display representative densities of randomly chosen weights of the network.}
    \label{fig:airunderfireballgrid}
\end{figure}

\clearpage

\section{BENCHMARKS}\label{app:uci_benchmark}

In this section, we provide benchmarks of sampling-based inference, both tabular regression with MLPs as well as image classification with CNN architectures. Overall, we find that sampling-based inference outperforms approximate Bayesian inference methods, delivering superior predictive performance and uncertainty quantification. We first describe the experimental setup in detail before showing the respective results. 

\paragraph{UCI Benchmark}
For the UCI benchmark presented in \cref{tab:uci_benchmark}, we fit classical mean regression to the different tasks corresponding to the datasets described in \cref{tab:dataoverview}. In the process, we always use a fully-connected feed-forward neural network with 3 hidden layers of size 16 each, resulting in about 700 total model parameters. If sampling from the posterior is done, we use 1000 samples per ensemble member (chain). We describe the configuration of the employed methods one by one:

\begin{itemize}
    \item For the \textbf{Laplace approximation} (LA), we utilize a JAX-based implementation to first train MAP solutions using the Adam optimizer with decoupled weight decay \citep{loshchilov2018decoupled} for 10000 epochs with a learning rate of $0.005$ to then carry out last-layer LA with a generalized Gauss Newton Hessian approximation and closed-form predictive approximation as detailed in \citet{daxberger_2021_LaplaceReduxa}. The variance of the predictive distribution is calculated according to \citet{daxberger_2021_LaplaceReduxa} with a small additional noise variance term.
    \item For \textbf{mean-field variational inference} (MFVI), we utilize a Gaussian posterior approximation with independence assumption. We optimize the evidence lower bound (ELBO) for 5000 epochs with the Adam optimizer and a learning rate of $0.005$. The variance of the predictive distribution is calculated as the variance over the predictions made with $100$ samples from the fitted approximate posterior with a small additional observation noise term.  
    \item As the recently proposed \textbf{Microcanonical Langevin Ensemble} (MILE) approach provides both an optimized \textbf{Deep Ensemble} (DE) and a \textbf{Bayesian Deep Ensemble} (BDE), we follow the suggested setup of \citet{MILE}, i.e., the DE is optimized with the Adam optimizer with decoupled weight decay with (memberwise) early stopping and the sampling then uses the proposed auto-tuning strategy of MILE comprising 50k steps before then providing 1k samples (after the thinning of 10k samples).
\end{itemize}
Each method is evaluated using three distinct train-test splits to assess the robustness of its performance.

\begin{table*}[!htb]
\caption{Mean RMSE ($\downarrow$) and LPPD ($\uparrow$) results ($\pm$ standard deviation across 3 train-test splits) for a 3 hidden-layer fully-connected neural network on regression tasks. Numbers in brackets indicate the number of ensemble members/chains.}
\label{tab:uci_benchmark}
\centering
\resizebox{0.7\textwidth}{!}{
{
\begin{tabular}{l|l|c|c|c|c}
 & \textbf{Dataset} & \textbf{Laplace} & \textbf{MFVI} & \textbf{\phantom{-}DE (10)} & \textbf{\phantom{-}MCMC}\\
\hline
\multirow{4}{*}{\rotatebox{90}{LPPD}} & Airfoil & $-1.056 \pm 0.003 $ & $-0.975 \pm 0.004 $ & $-0.293 \pm 0.096$ & $\mathbf{\phantom{-}0.016 \pm 0.293}$ \\
& Bikesharing & $ -1.046 \pm 0.001 $ & $ -0.990 \pm 0.005 $ & $-0.223 \pm 0.181$ & $\mathbf{-0.060 \pm 0.096}$ \\
& Concrete & $ -1.131 \pm 0.036 $ & $ -0.998 \pm 0.007 $ & $-0.510 \pm 0.189$ & $\mathbf{\phantom{-}0.042 \pm 0.056}$ \\
& Energy & $ -1.046 \pm 0.004 $ & $ -0.945 \pm 0.002 $ & $\phantom{-}1.561 \pm 0.101$ & $\mathbf{\phantom{-}1.947 \pm 0.047}$ \\
\hline
\multirow{4}{*}{\rotatebox{90}{RMSE}} & Airfoil & $\phantom{-}0.237 \pm 0.013 $ & $ \phantom{-}0.276 \pm 0.009 $ & $\phantom{-}0.269 \pm 0.016$& $ \mathbf{\phantom{-}0.184 \pm 0.016}$ \\
& Bikesharing & $\phantom{-}0.252 \pm 0.006 $ & $ \phantom{-}0.318 \pm 0.018 $  & $\mathbf{\phantom{-}0.253 \pm 0.015}$ & $ \phantom{-}0.262 \pm 0.018$ \\
& Concrete & $\phantom{-}0.482 \pm 0.100 $ & $ \phantom{-}0.350 \pm 0.025 $ & $\phantom{-}0.297 \pm 0.032$& $\mathbf{\phantom{-}0.270 \pm 0.034}$ \\
& Energy & $\phantom{-}0.065 \pm 0.008 $ & $ \phantom{-}0.126 \pm 0.007 $ & $\phantom{-}0.050 \pm 0.001$ & $\mathbf{\phantom{-}0.041 \pm 0.003}$ \\
\end{tabular}
}}
\end{table*}

\paragraph{Image Classification on CIFAR-10 Using SGHMC}
We extend the tabular UCI benchmarks of \cref{tab:uci_benchmark} to an image classification task on CIFAR-10 using a small ResNet with 76106 parameters. We employ scale-adapted SGHMC \citep{adasghmc} for sampling, use an ensemble of 10 with 5k warmup steps, 25k sampling steps, thinning of 250, step size 0.001, momentum decay 0.05, batch size 256, and standard normal isotropic priors. Notably, compared to DE optimization, \textbf{sampling required only 25\% additional compute}. \cref{tab:cifar10_results} displays both predictive and UQ performance as well as wallclock time in comparison with optimization-based approximate inference methods implemented via the novel \texttt{posteriors} package \citep{duffield2025scalable}. Again, the sampling approach provides the best predictive UQ and precision at a negligible cost in wall-clock time. Due to the non-trivial optimization of some VI methods (for example, caused by noisy gradients), sampling can be faster and more robust to fit even in these larger-scale settings.

\begin{table*}[h!]
\caption{Image classification task on CIFAR-10 using a small ResNet with 76106 parameters. We employ scale-adapted SGHMC for sampling, use an ensemble of 10 with 5k warmup steps, 25k sampling steps, thinning of 250, step size 0.001, momentum decay 0.05, batch size 256, and standard normal isotropic priors. Mean accuracy ($\uparrow$) and LPPD ($\uparrow$) results ($\pm$ standard deviation) for CIFAR10 classification are reported. Numbers in brackets indicate ensemble members/chains.}
\label{tab:cifar10_results}
\vspace{0.1cm}
\centering
\resizebox{0.9\textwidth}{!}{
\begin{tabular}{r|c|c|c|c} & \textbf{Laplace} & \textbf{MFVI} & \textbf{DE (10)} & \textbf{BDE (10)} \\
\hline
Accuracy ($\uparrow$) & $\phantom{-}0.7851 \pm 0.0067$ & $\phantom{-}0.7133 \pm 0.0102$ & $\phantom{-}0.8222 \pm 0.0013$ & $\mathbf{\phantom{-}0.8273 \pm 0.0020}$ \\
LPPD ($\uparrow$) & $-1.6731 \pm 0.0571$ & $-0.8410 \pm 0.0193$ & $-0.5341 \pm 0.0050$ & $\mathbf{-0.5149 \pm 0.0030}$ \\
Wallclock Time in Minutes ($\downarrow$) & $ \phantom{.}12.92 \pm 0.41$ & $ \phantom{.}15.78 \pm 0.27 $ & \phantom{.}$\mathbf{11.59 \pm 0.34}$ & \phantom{.}$14.41 \pm 0.02$ \\
\end{tabular}
}
\end{table*}

\paragraph{Computational Cost Comparison Between DE and BDE}

To assess the trade-off between computational cost and predictive performance, we compare Deep Ensembles and Bayesian Deep Ensembles on the \texttt{airfoil} and \texttt{ionosphere} datasets. While BDE for the same number of members is significantly more expensive---taking 12.22 and 8.25 times more wall-clock time per member on \texttt{airfoil} and \texttt{ionosphere}, respectively---it consistently outperforms DE in terms of LPPD. The results of the experiments are given in \cref{tab:compbudgetcomp}. For instance, BDE(1) on \texttt{ionosphere} achieves an LPPD of $-0.1544 \pm 0.0119$ compared to DE(8000)'s $-0.2923$, despite the latter using nearly 1000 times more compute. For the \texttt{airfoil} dataset to ensure a fair comparison in terms of runtime, we compare BDE(10) to DE(125) over 3 replications and again observe superior sampling performance for the same computational effort.

\begin{table*}[h!]
\caption{Computational cost and predictive performance comparison between DE and BDE on the \texttt{airfoil} and \texttt{ionosphere} datasets.}

\label{tab:compbudgetcomp}
\vspace{0.1cm}
\centering
\resizebox{0.7\textwidth}{!}{
\begin{tabular}{l|c|c|c|c} \textbf{Model} & \textbf{Ensemble Members} & \textbf{Dataset} & \textbf{LPPD} ($\uparrow$) & \textbf{RMSE/Accuracy} \\
\hline
DE & 125 & \texttt{airfoil} & $0.0285 \pm 0.0266$ & $0.2870 \pm 0.0240$  \\
BDE & 10 & \texttt{airfoil} & $\mathbf{0.7660 \pm 0.0075}$ & $\mathbf{0.1415 \pm 0.0079}$  \\
\hline
DE & 8000 & \texttt{ionosphere} & $-0.2923$ & $0.9154$  \\
BDE & 1 & \texttt{ionosphere} & $\mathbf{-0.1544}$ & $\mathbf{0.9436}$  
\end{tabular}
}
\end{table*}

\end{document}